\documentclass[journal]{./cls/IEEEtran}
\usepackage{enumitem}
\usepackage[table,usenames,dvipsnames]{xcolor}      
\usepackage[noadjust]{cite}
\usepackage{amsmath,amssymb,amsfonts,amsthm,dsfont,bbm} 
\usepackage{algorithm,algorithmicx,listings}        
\usepackage[noend]{algpseudocode}			        
\usepackage{graphicx,tabularx,adjustbox}
\usepackage[font={small}]{caption}   
\usepackage{subcaption}
\usepackage{thmtools,thm-restate}
\usepackage[breaklinks=true, colorlinks, bookmarks=true, citecolor=Black, urlcolor=Violet,linkcolor=Black]{hyperref}

\newtheorem{proposition}{Proposition}
\newtheorem*{proposition*}{Proposition}
\newtheorem{corollary}{Corollary}
\newtheorem*{corollary*}{Corollary}

\theoremstyle{definition}
\newtheorem{definition}{Definition}
\newtheorem{assumption}{Assumption}
\newtheorem*{assumption*}{Assumption}
\newtheorem*{problem*}{Problem}
\newtheorem{problem}{Problem}
\theoremstyle{remark}

\newtheorem*{solution*}{Solution}

\newcommand{\prl}[1]{\left(#1\right)}

\newcommand{\scaleMathLine}[2][1]{\resizebox{#1\linewidth}{!}{$\displaystyle{#2}$}}

\algnewcommand{\IfThenElse}[3]{
  \State \algorithmicif\ #1\ \algorithmicthen\ #2\ \algorithmicelse\ #3}

\DeclareMathOperator{\diag}{diag}

\newcommand{\NEW}[1]{{\color{black}#1}}

\newcommand{\bsym}[1]{{\boldsymbol{#1}}}

\ifCLASSINFOpdf
\else
\fi
\begin{document}
%
\title{Autonomous Navigation in Unknown Environments with Sparse Bayesian Kernel-based Occupancy Mapping}
%
%
%
\author{Thai~Duong,~\IEEEmembership{Student Member,~IEEE,}~Michael~Yip,~\IEEEmembership{Member,~IEEE,}~and~Nikolay Atanasov,~\IEEEmembership{Member,~IEEE}
\thanks{We gratefully acknowledge support from ARL DCIST CRA W911NF-17-2-0181 and ONR SAI N00014-18-1-2828.}%
\thanks{The authors are with the Department of Electrical and Computer Engineering, 
  University of California, San Diego, 
  La Jolla, CA 92093 USA, email \{tduong, yip, natanasov\}@ucsd.edu}
}

%
%

\markboth{IEEE Transactions on Robotics}%
{Shell \MakeLowercase{\textit{et al.}}: Bare Demo of IEEEtran.cls for IEEE Journals}
%

\newcommand{\calA}{{\cal A}}
\newcommand{\calB}{{\cal B}}
\newcommand{\calC}{{\cal C}}
\newcommand{\calD}{{\cal D}}
\newcommand{\calE}{{\cal E}}
\newcommand{\calF}{{\cal F}}
\newcommand{\calG}{{\cal G}}
\newcommand{\calH}{{\cal H}}
\newcommand{\calI}{{\cal I}}
\newcommand{\calJ}{{\cal J}}
\newcommand{\calK}{{\cal K}}
\newcommand{\calL}{{\cal L}}
\newcommand{\calM}{{\cal M}}
\newcommand{\calN}{{\cal N}}
\newcommand{\calO}{{\cal O}}
\newcommand{\calP}{{\cal P}}
\newcommand{\calQ}{{\cal Q}}
\newcommand{\calR}{{\cal R}}
\newcommand{\calS}{{\cal S}}
\newcommand{\calT}{{\cal T}}
\newcommand{\calU}{{\cal U}}
\newcommand{\calV}{{\cal V}}
\newcommand{\calW}{{\cal W}}
\newcommand{\calX}{{\cal X}}
\newcommand{\calY}{{\cal Y}}
\newcommand{\calZ}{{\cal Z}}

\newcommand{\setA}{\textsf{A}}
\newcommand{\setB}{\textsf{B}}
\newcommand{\setC}{\textsf{C}}
\newcommand{\setD}{\textsf{D}}
\newcommand{\setE}{\textsf{E}}
\newcommand{\setF}{\textsf{F}}
\newcommand{\setG}{\textsf{G}}
\newcommand{\setH}{\textsf{H}}
\newcommand{\setI}{\textsf{I}}
\newcommand{\setJ}{\textsf{J}}
\newcommand{\setK}{\textsf{K}}
\newcommand{\setL}{\textsf{L}}
\newcommand{\setM}{\textsf{M}}
\newcommand{\setN}{\textsf{N}}
\newcommand{\setO}{\textsf{O}}
\newcommand{\setP}{\textsf{P}}
\newcommand{\setQ}{\textsf{Q}}
\newcommand{\setR}{\textsf{R}}
\newcommand{\setS}{\textsf{S}}
\newcommand{\setT}{\textsf{T}}
\newcommand{\setU}{\textsf{U}}
\newcommand{\setV}{\textsf{V}}
\newcommand{\setW}{\textsf{W}}
\newcommand{\setX}{\textsf{X}}
\newcommand{\setY}{\textsf{Y}}
\newcommand{\setZ}{\textsf{Z}}

\newcommand{\bfa}{\mathbf{a}}
\newcommand{\bfb}{\mathbf{b}}
\newcommand{\bfc}{\mathbf{c}}
\newcommand{\bfd}{\mathbf{d}}
\newcommand{\bfe}{\mathbf{e}}
\newcommand{\bff}{\mathbf{f}}
\newcommand{\bfg}{\mathbf{g}}
\newcommand{\bfh}{\mathbf{h}}
\newcommand{\bfi}{\mathbf{i}}
\newcommand{\bfj}{\mathbf{j}}
\newcommand{\bfk}{\mathbf{k}}
\newcommand{\bfl}{\mathbf{l}}
\newcommand{\bfm}{\mathbf{m}}
\newcommand{\bfn}{\mathbf{n}}
\newcommand{\bfo}{\mathbf{o}}
\newcommand{\bfp}{\mathbf{p}}
\newcommand{\bfq}{\mathbf{q}}
\newcommand{\bfr}{\mathbf{r}}
\newcommand{\bfs}{\mathbf{s}}
\newcommand{\bft}{\mathbf{t}}
\newcommand{\bfu}{\mathbf{u}}
\newcommand{\bfv}{\mathbf{v}}
\newcommand{\bfw}{\mathbf{w}}
\newcommand{\bfx}{\mathbf{x}}
\newcommand{\bfy}{\mathbf{y}}
\newcommand{\bfz}{\mathbf{z}}

\newcommand{\bfalpha}{\boldsymbol{\alpha}}
\newcommand{\bfbeta}{\boldsymbol{\beta}}
\newcommand{\bfgamma}{\boldsymbol{\gamma}}
\newcommand{\bfdelta}{\boldsymbol{\delta}}
\newcommand{\bfepsilon}{\boldsymbol{\epsilon}}
\newcommand{\bfzeta}{\boldsymbol{\zeta}}
\newcommand{\bfeta}{\boldsymbol{\eta}}
\newcommand{\bftheta}{\boldsymbol{\theta}}
\newcommand{\bfiota}{\boldsymbol{\iota}}
\newcommand{\bfkappa}{\boldsymbol{\kappa}}
\newcommand{\bflambda}{\boldsymbol{\lambda}}
\newcommand{\bfmu}{\boldsymbol{\mu}}
\newcommand{\bfnu}{\boldsymbol{\nu}}
\newcommand{\bfomicron}{\boldsymbol{\omicron}}
\newcommand{\bfpi}{\boldsymbol{\pi}}
\newcommand{\bfrho}{\boldsymbol{\rho}}
\newcommand{\bfsigma}{\boldsymbol{\sigma}}
\newcommand{\bftau}{\boldsymbol{\tau}}
\newcommand{\bfupsilon}{\boldsymbol{\upsilon}}
\newcommand{\bfphi}{\boldsymbol{\phi}}
\newcommand{\bfchi}{\boldsymbol{\chi}}
\newcommand{\bfpsi}{\boldsymbol{\psi}}
\newcommand{\bfomega}{\boldsymbol{\omega}}
\newcommand{\bfxi}{\boldsymbol{\xi}}
\newcommand{\bfell}{\boldsymbol{\ell}}

\newcommand{\bfA}{\mathbf{A}}
\newcommand{\bfB}{\mathbf{B}}
\newcommand{\bfC}{\mathbf{C}}
\newcommand{\bfD}{\mathbf{D}}
\newcommand{\bfE}{\mathbf{E}}
\newcommand{\bfF}{\mathbf{F}}
\newcommand{\bfG}{\mathbf{G}}
\newcommand{\bfH}{\mathbf{H}}
\newcommand{\bfI}{\mathbf{I}}
\newcommand{\bfJ}{\mathbf{J}}
\newcommand{\bfK}{\mathbf{K}}
\newcommand{\bfL}{\mathbf{L}}
\newcommand{\bfM}{\mathbf{M}}
\newcommand{\bfN}{\mathbf{N}}
\newcommand{\bfO}{\mathbf{O}}
\newcommand{\bfP}{\mathbf{P}}
\newcommand{\bfQ}{\mathbf{Q}}
\newcommand{\bfR}{\mathbf{R}}
\newcommand{\bfS}{\mathbf{S}}
\newcommand{\bfT}{\mathbf{T}}
\newcommand{\bfU}{\mathbf{U}}
\newcommand{\bfV}{\mathbf{V}}
\newcommand{\bfW}{\mathbf{W}}
\newcommand{\bfX}{\mathbf{X}}
\newcommand{\bfY}{\mathbf{Y}}
\newcommand{\bfZ}{\mathbf{Z}}

\newcommand{\bfGamma}{\boldsymbol{\Gamma}}
\newcommand{\bfDelta}{\boldsymbol{\Delta}}
\newcommand{\bfTheta}{\boldsymbol{\Theta}}
\newcommand{\bfLambda}{\boldsymbol{\Lambda}}
\newcommand{\bfPi}{\boldsymbol{\Pi}}
\newcommand{\bfSigma}{\boldsymbol{\Sigma}}
\newcommand{\bfUpsilon}{\boldsymbol{\Upsilon}}
\newcommand{\bfPhi}{\boldsymbol{\Phi}}
\newcommand{\bfPsi}{\boldsymbol{\Psi}}
\newcommand{\bfOmega}{\boldsymbol{\Omega}}

\newcommand{\bbA}{\mathbb{A}}
\newcommand{\bbB}{\mathbb{B}}
\newcommand{\bbC}{\mathbb{C}}
\newcommand{\bbD}{\mathbb{D}}
\newcommand{\bbE}{\mathbb{E}}
\newcommand{\bbF}{\mathbb{F}}
\newcommand{\bbG}{\mathbb{G}}
\newcommand{\bbH}{\mathbb{H}}
\newcommand{\bbI}{\mathbb{I}}
\newcommand{\bbJ}{\mathbb{J}}
\newcommand{\bbK}{\mathbb{K}}
\newcommand{\bbL}{\mathbb{L}}
\newcommand{\bbM}{\mathbb{M}}
\newcommand{\bbN}{\mathbb{N}}
\newcommand{\bbO}{\mathbb{O}}
\newcommand{\bbP}{\mathbb{P}}
\newcommand{\bbQ}{\mathbb{Q}}
\newcommand{\bbR}{\mathbb{R}}
\newcommand{\bbS}{\mathbb{S}}
\newcommand{\bbT}{\mathbb{T}}
\newcommand{\bbU}{\mathbb{U}}
\newcommand{\bbV}{\mathbb{V}}
\newcommand{\bbW}{\mathbb{W}}
\newcommand{\bbX}{\mathbb{X}}
\newcommand{\bbY}{\mathbb{Y}}
\newcommand{\bbZ}{\mathbb{Z}}




\maketitle

\begin{abstract}
This paper focuses on online occupancy mapping and real-time collision checking onboard an autonomous robot navigating in a large unknown environment. Commonly used voxel and octree map representations can be easily maintained in a small environment but have increasing memory requirements as the environment grows. We propose a fundamentally different approach for occupancy mapping, in which the boundary between occupied and free space is viewed as the decision boundary of a machine learning classifier. This work generalizes a kernel perceptron model which maintains a very sparse set of support vectors to represent the environment boundaries efficiently. We develop a probabilistic formulation based on Relevance Vector Machines, allowing robustness to measurement noise and probabilistic occupancy classification, supporting autonomous navigation. We provide an online training algorithm, updating the sparse Bayesian map incrementally from streaming range data, and an efficient collision-checking method for general curves, representing potential robot trajectories. The effectiveness of our mapping and collision checking algorithms is evaluated in tasks requiring autonomous robot navigation \NEW{and active mapping} in unknown environments. 
\end{abstract}


%
\IEEEpeerreviewmaketitle

\section*{Supplementary Material}
Software and videos supplementing this paper:\\ 
\centerline{\url{https://thaipduong.github.io/sbkm}}
\section{Introduction}
\label{sec:intro}


Autonomous navigation in robotics involves online localization, mapping, motion planning, and control in partially known environments perceived through streaming data from onboard sensors~\cite{human_friendly_nav_guzzi_icra13, safe_auto_nav_pavone_rss18}. This paper focuses on the occupancy mapping problem and, specifically, on enabling large-scale, yet compact, representations and efficient collision checking to support autonomous navigation. Occupancy mapping is a well established and widely studied problem in robotics and a variety of explicit and implicit map representations have been proposed. Explicit maps model the obstacle surfaces directly, e.g., via surfels~\cite{rgbd-slam,chisel,elastic_fusion,DenseSurfelMapping,behley2018rss_surfel_slam}, geometric primitives~\cite{Kaess_infiniteplanes,Bowman_SemanticSLAM_ICRA17,quadric_slam,cubeslam,Shan_OrcVIO_IROS20}, or polygonal meshes~\cite{teixeira2016real,piazza2018real,kimera}. Implicit maps model the obstacle surfaces as the level set of an occupancy~\cite{occgrid,InfiniTAM,VoxelMapVisualSLAM,octomap,octree_fusion,supereight} or signed distance~\cite{curless1996volumetric,kazhdan2006poisson,kinfu,voxblox,fiesta} or spatial function encoded via voxels~\cite{niessner2013real,voxblox,fiesta} or octrees~\cite{octomap,octree_fusion,supereight}. The goal of this work is to generate \emph{sparse} \emph{probabilistic} maps online, enabling large-scale environment modeling, map uncertainty quantification, and efficient collision checking.



Our preliminary work \cite{duong2020autonomous} develops a kernel perceptron model for online occupancy mapping. The model uses support vectors and a kernel function to represent obstacle boundaries in configuration space. The number of support vectors scales with the complexity of the obstacle boundaries rather than the environment size. We develop an online training algorithm to update the support vectors incrementally as new range observations of the local surroundings are provided by the robot's sensors. To enable motion planning in the new occupancy representation, we develop efficient collision checking algorithms for piecewise-linear and piecewise-polynomial trajectories in configuration space. Our kernel perceptron model, however, provides occupancy labels without a probability distribution, making the classification accuracy susceptible to measurement noise. Since unknown regions are frequently assumed free for motion planning purposes, the lack of probabilistic information also does not allow us to distinguish between well-observed and unseen regions. This is especially important in active exploration problems, where the robot autonomously chooses the unknown regions to explore.

\begin{figure}[t]
\centering
        \includegraphics[width=0.48\textwidth]{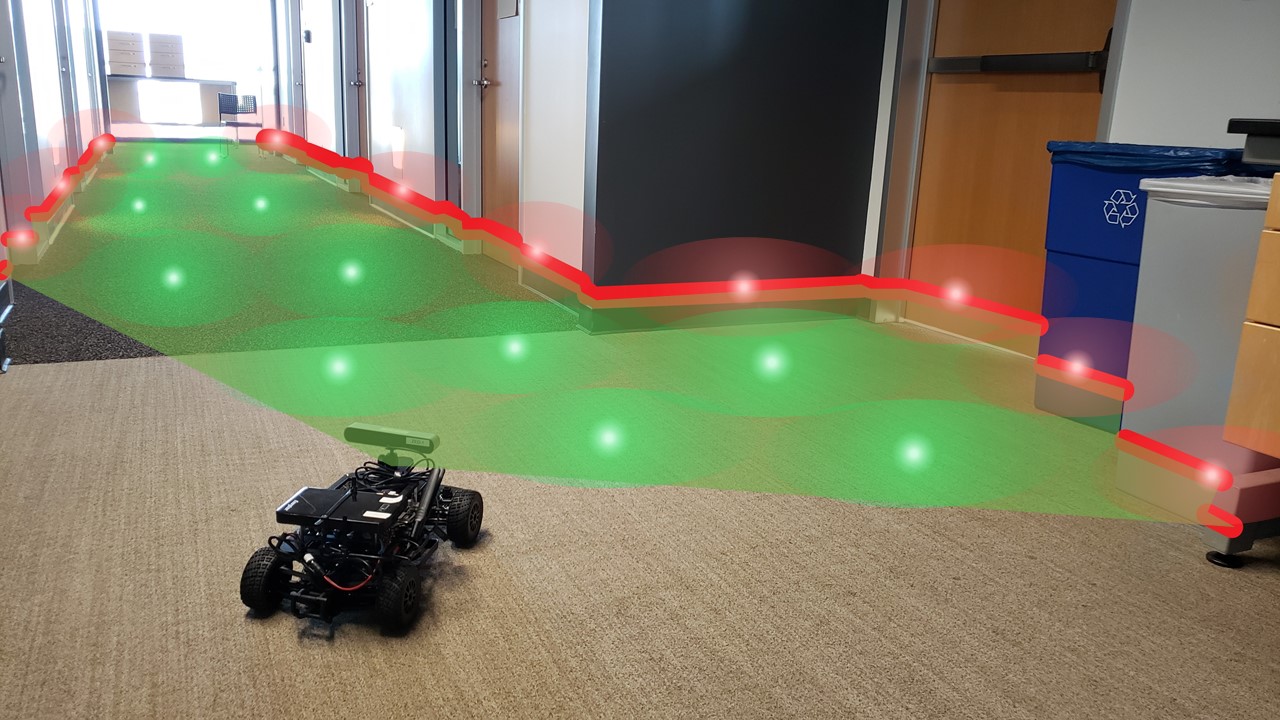}
\caption{A ground robot in an unknown environment relying on lidar scan data (red) for online occupancy mapping and collision-free trajectory planning. Our mapping algorithm learns a sparse set of occupied and free relevance vectors (light red and green dots, respectively) that represents the environment based on the lidar scans.}
\label{fig:robot_unkenv_scan}
\end{figure}


This paper develops a sparse Bayesian formulation of the occupancy mapping problem and introduces an incremental Relevance Vector Machine training algorithm to probabilistically model the environment. To make our sparse Bayesian kernel-based map compatible with motion planning algorithms, we derive collision checking algorithms for linear and general trajectories.


\textbf{Contributions}. This paper introduces a sparse Bayesian kernel-based mapping method that: 
\begin{itemize}[leftmargin=2em,nosep]
  \item represents continuous-space probabilistic occupancy using a sparse set of relevance vectors stored in an $R^*$-tree data structure (Sec. \ref{sec:online_probit_rvm} and \ref{subsec:online_mapping}),
  \item allows online map updates from streaming partial observations using an incremental Relevance Vector Machine training algorithm with the predictive distribution modeled by a probit function. (Sec. \ref{subsec:online_probit_rvm}), and

  \item provides efficient and complete (without sampling) collision checking for general robot trajectories (Sec. \ref{sec:poly_curve_check} and \ref{subsec:auto_nav}).
\end{itemize}

\section{Related Work}
\label{sec:related_work}

Occupancy grid mapping is a commonly used approach for modeling the free and occupied space of an environment. The space is discretized into a collection of cells, whose occupancy probabilities are estimated online using the robot's sensory data. While early work~\cite{thrun2005probabilistic, gmapping} assumes that the cells are independent, Gaussian process (GP) occupancy mapping~\cite{OCallaghan2012gaussian, wang2016fast, jadidi2017warped} uses a kernel function to capture the correlation among grid cells and predict the occupancy of unobserved cells. Online training of a Gaussian process model, however, does not scale well as its computational complexity grows cubically with the number of data points. Ramos et al.~\cite{ramos2016hilbertmap} improve on this by projecting the data points into Hilbert space and training a logistic regression model. Senanayake and Ramos~\cite{senanayake2017bayesian} propose a Bayesian treatment of Hilbert maps, called Sequential Bayesian Hilbert Map (SBHM), that updates the map from sequential observations of the environment. They achieve sparseness by calculating feature vectors based on a sparse set of hinged points, e.g., on a coarse grid. Instead of a fixed set of hinged points, Relevance Vector Machine (RVM)~\cite{tipping2000relevance, tipping2001sparse, tipping2003fast} learns a sparse set of relevance vectors from the training dataset. The original RVM work~\cite{tipping2000relevance} initially assumes that all data points are relevance vectors and prunes them down, incurring high computation cost. Tipping and Faul~\cite{tipping2003fast} derive a fast training algorithm that starts from an empty set of relevance vectors and adds points to the set gradually. Meanwhile, Lopez and How~\cite{lopez2017aggressive} propose an efficient determinstic alternative, which builds a k-d tree from point clouds and queries the nearest obstacles for collision checking. Using spatial partitioning similar to a k-d tree, octree-based maps~\cite{octomap,chen2017improving} offer efficient map storage by performing octree compression, while AtomMap~\cite{fridovich2017atommap} stores a collection of spheres in a k-d tree as a way to avoid grid cell discretization of the map. Instead of storing occupancy information, Voxblox \cite{voxblox} stores distance to obstacles in each cell and builds an Euclidean Signed Distance Field, as a map representation, online from streaming sensor data.

Navigation, in an unknown environment, requires the safety of potential robot trajectories to be evaluated through a huge amount of collision checks with respect to the map representation~\cite{bialkowski2016efficient,luo2014empirical,hauser2015lazy}. Many works rely on sampling-based collision checking, simplifying the safety verification of continuous-time trajectories by evaluating only a finite set of samples along the trajectory~\cite{luo2014empirical, Tsardoulias2016planninggridmap}. This may be undesirable in safety critical applications. Bialkowski et al.~\cite{bialkowski2016efficient} propose an efficient collision checking method using safety certificates with respect to the nearest obstacles. Using a different perspective, learning-based collision checking methods~\cite{das2017fastron, pan2015efficient, huh2016learningGMM} sample data from the environment and train machine learning models to approximate the obstacle boundaries. Pan et al.~\cite{pan2015efficient} propose an incremental support vector machine model for pairs of obstacles but train the models offline. Closely related to our work, Das et al.~\cite{das2017fastron, das2020learning} develop an online training algorithm, called Fastron, to train a kernel perceptron collision classifier. To handle dynamic environments, Fastron actively resamples the environment and updates the model globally. Geometry-based collision checking methods, such as the Flexible Collision Library (FCL)~\cite{pan2012fcl}, are also related but rely on mesh representations of the environment which may be inefficient to generate from local observations.


Our preliminary work \cite{duong2020autonomous}, summarized in Sec. \ref{sec:kernel_map_summary}, provides an approach to online occupancy mapping that supports efficient collision checking with guarantees. \NEW{However, to achieve robustness to noisy measurements and probabilistically model well-observed and unknown regions, we introduce a probabilistic formulation based on RVM inference that enables online sparse Bayesian kernel-based occupancy mapping}. Inspired by GP mapping techniques, we utilize a kernel function to capture occupancy correlations but focus on a compact representation of obstacle boundaries by building an RVM model, i.e. a sparse set of relevance vectors, incrementally from streaming local sensor data. Specifically, only a local subset of the relevance vectors is updated each time using our incremental RVM training algorithm. \NEW{Furthermore, motivated by the collision checking approach in~\cite{bialkowski2016efficient}, we derive our own efficient collision checking algorithms for our map representation}. We develop an ``inflated boundary" of the obstacle boundary that enables closed-form conditions for checking line segments and ellipsoids for collision. These key conditions allow us to check potential robot trajectories for motion planning purposes.

\section{Problem Formulation}
\label{sec:problem_formulation}

Consider a robot with state $\bfs \in \calS$, consisting of the robot's position $\bfp \in [0,1]^d$ and other variables such as orientation, velocity, etc., navigating in an unknown environment (Fig. \ref{fig:robot_unkenv_scan}). Let $\calO \subset [0,1]^d$ be a closed set representing occupied space and let $\calF$ be its complement, representing free space. Assume that the robot can be enclosed by a sphere of radius $r \in \mathbb{R}_{>0}$ centered at $\bfp$. In configuration space (C-space), the robot body becomes a point $\bfp$, while the obstacle space and free space are transformed as $\bar{\calO} = \cup_{\bfx\in \calO} \mathcal{B}(\bfx,r)$, where $\mathcal{B}(\bfx,r) = \{\bfx'\in [0,1]^d: \|\bfx-\bfx'\|_2 \leq r\}$, and $\bar{\calF} = [0,1]^d \setminus \bar{\calO}$. Let $\bar{\calS}$ be the subset of the robot state space that corresponds to the collision-free robot positions $\bar{\calF}$.

Let $\dot{\bfs}(t) = f(\bfs(t), \bfa(t))$ characterize the continuous-time robot dynamics with control input trajectory $\bfa(t) \in \calA$. We consider constant control inputs (zero-order hold) applied at discrete time steps $t_k$ for $k = 0, 1, \ldots, N$ so that $\bfa(t) \equiv \bfa_k$ for $[t_k, t_{k+1})$. We assume that the state $\bfs(t)$ is known or estimated by a localization algorithm and let $\bfs_k := \bfs(t_k)$.

The robot is equipped with a sensor, such as lidar or depth camera, that provides distance measurements $\bfz_k$ at time $t_k$ to the obstacle space $\calO$ within its field of view. Our objective is to construct an occupancy map $\hat{m}_{k}: [0,1]^d \rightarrow \{-1, 1\}$ of the C-space based on accumulated observations $\bfz_{0:k}$, where ``$-1$" and ``$1$" mean ``free" and ``occupied", respectively. As the robot is navigating, new sensor data are used to update the map as a function , $\hat{m}_{k+1} = g(\hat{m}_k, \bfz_k)$, of the previous estimate $\hat{m}_k$ and a newly received range observation~$\bfz_k$. Assuming unobserved regions are free, we rely on $\hat{m}_k$ to plan a robot trajectory to a goal region $\calG \subseteq \bar{\calS}$. Applying control action $\bfa$ at $\bfs$ incurs a motion cost $c(\bfs, \bfa)$, e.g., based on traveled distance or energy expenditure, and we aim to minimize the cumulative cost of navigating safely to the goal $\calG$.

\begin{problem}
\label{problem_formulation_unknown_env}
Given a start state $\bfs_0 \in \bar{\calS}$ and a goal region $\calG \subseteq \bar{\calS}$, find a sequence of control actions that leads the robot to $\calG$ safely, while minimizing the motion cost:
\begin{align}
\label{problem_formulation_unknown_env_equation}
\min_{N, \bfa_0, \ldots, \bfa_N} \;&\sum_{k=0}^{N-1}  c(\bfs_k, \bfa_{k})\\
\text{s.t.} \quad\; &\dot{\bfs} = f(\bfs, \bfa), \bfa(t) = \bfa_k \text{ for } t\in [t_k, t_{k+1}), \notag\\
&\;\bfs(t_0) = \bfs_0, \; \bfs_{N} \in \calG, \;\hat{m}_{k+1} = g(\hat{m}_{k}, \bfz_{k}), \notag\\
& \;\hat{m}_{k}(\bfs(t)) = -1 \text{ for } t\in [t_k, t_{k+1}), k= 0,\ldots,N.\notag
\end{align}
\end{problem}

In the remainder of the paper, we develop a sparse Bayesian kernel-based map representation, offering efficient collision checking for robot trajectories, and propose a complete solution to Problem \ref{problem_formulation_unknown_env}.

\section{A Sparse Kernel-based Classifier for Occupancy Mapping}
\label{sec:kernel_map_summary}

Our preliminary work~\cite{duong2020autonomous} on sparse Kernel-based map (SKM) develops a sparse kernel perceptron model for online classification of occupied and free space in the environment. The model uses a set of support vectors and a kernel function to represent the obstacle boundaries in configuration space. The number of support vectors necessary for accurate classification scales with the complexity of the obstacle boundaries rather than the environment size. Our approach extends the Fastron algorithm~\cite{das2017fastron, pan2015efficient}, which efficiently trains a kernel perceptron model using a training dataset collected globally from the environment. We develop an online training procedure (Alg.~\ref{alg:fastron_model}) that updates the support vectors incrementally as new range observations $\bfz_k$ of the local surroundings arrive. Given a training dataset $\mathcal{D} = \{(\bfx_l, y_l)\}$ generated from $\bfz_k$ (e.g. see Sec.~\ref{subsec:online_mapping} for details), Alg.~\ref{alg:fastron_model} prioritizes updating misclassified points' weight based on their margins (lines $6$ and $7$) and remove the redundant support vectors (line 8) without affecting the model. When the next local dataset arrives, it looks for new misclassified points and incrementally adds them to the set of support vectors. Alg.~\ref{alg:fastron_model} returns a set of $M^+$ positive support vectors and their weight $\Lambda^+ = \{(\bfx_i^+, \alpha_i^+)\}_i$ and a set of $M^-$ negative support vectors and their weight $\Lambda^- = \{(\bfx_j^-, \alpha_j^-)\}_j$. The classifier decision boundary is characterized by a score function: 
\begin{equation}
\label{eq:fastron_score}
F(\bfx) =\sum_{i = 1}^{M^+} \alpha_i^+ k(\bfx_i^+, \bfx) - \sum_{j = 1}^{M^-} \alpha_j^- k(\bfx_j^-, \bfx),
\end{equation}
where $k(\cdot, \cdot)$ is a kernel function and $\alpha_j^-, \alpha_i^+ > 0$. The occupancy of a query point $\bfx$ can be checked by evaluating the score function $F(\bfx)$ in Eq.~\eqref{eq:fastron_score}. Specifically, $\hat{m}_t(\bfx) = -1$ if $F(\bfx) < 0$ and $\hat{m}_t(\bfx) = 1$ if $F(\bfx) \geq 0$. The score calculation becomes slower when the number of support vectors increases. We improve on this by storing the support vectors in an $R^*$-tree data structure and efficiently query $K^+$ and $K^-$ nearest positive and negative support vectors (line 1 in Alg.~\ref{alg:fastron_model}) from the $R^*$-tree to approximate $F(\bfx)$.

\begin{algorithm}[t]
\caption{Incremental Kernel Perceptron Training \cite{duong2020autonomous}}
\label{alg:fastron_model}
\footnotesize
	\begin{algorithmic}[1]
		\Require Support vectors $\Lambda^+ = \{(\bfx^+_i, \alpha^+_i)\}_i$ and $\Lambda^- = \{(\bfx_j^-, \alpha_j^-)\}_j$ stored in an $R^*$-tree; Local dataset $\mathcal{D} = \{(\bfx_l, y_l)\}$; $\xi^+, \xi^- > 0$; $N_{max}$.
		\Ensure Updated $\Lambda^+, \Lambda^-$.
		\State Query $K^+, K^-$ nearest negative and positive support vectors from an $R^*$-tree data structure.
		\For {$(\bfx_l, y_l)$ in $\mathcal{D}$}
			\State Calculate $F_l = \sum_{i = 1}^{K^+} \alpha^+_i k(\bfx^+_i,\bfx_l) - \sum_{j = 1}^{K^-} \alpha^-_j k(\bfx^-_j,\bfx_l)$
		\EndFor
		\For {$t = 1$ to $N_{max}$}
			\If{$y_lF_l > 0 \quad \forall l$} \Return{$\Lambda^+, \Lambda^-$} 
			\EndIf
			\State $m = \text{argmin}_l y_l F_l$	
			\State \Call{Weight Correction}{$F_m,y_m,\Lambda^+, \Lambda^-,\xi^+, \xi^-$} 
			\State \Call{Redundancy Removal}{$\Lambda^+, \Lambda^-, \mathcal{D}$} 
		\EndFor
		\State \Return{$\Lambda^+, \Lambda^-$}
		\Function{Weight Correction}{$F_m, y_m, \Lambda^+, \Lambda^-, \xi^+, \xi^-$}
			\State $\xi = \xi^+$ if $y_m >0$; and $\xi = \xi^-$, otherwise.
			\State Calculate $\Delta_\alpha = \xi y_m - F_m$.
			\If{$\exists(\bfx_m, \alpha_m) \in \Lambda^+ \cup \Lambda^-$}
				\State Update weights: $\alpha_m \text{+=} y_m \Delta \alpha$, $F_l \text{+=} k(\bfx_l, \bfx_m) y_m \Delta\alpha, \forall l$
			\Else
				\State Calculate $\alpha_m = y_m \Delta \alpha$
				\State Add $(\bfx_m, \alpha_m)$ to $\Lambda^+$ if $y_m > 0$ and $\Lambda^-$, otherwise.
			\EndIf
		\EndFunction
		\Function{Redundancy Removal}{$\Lambda^+, \Lambda^-, \mathcal{D}$}
			\For{$(\bfx_l, y_l) \in \mathcal{D}$} 
				\If{$\exists (\bfx_l, \alpha_l) \in \Lambda^+ \cup \Lambda^-$ and $y_l(F_l - \alpha^+_l) > 0$}
				\State Remove $(\bfx_l, \alpha_l)$ from $\Lambda^+$ or $\Lambda^-$
				\State Update $F_n \text{-=}\; k(\bfx_l, \bfx_n)\alpha^+_l, \forall (\bfx_n, \cdot) \in \mathcal{D}$
				\EndIf
			\EndFor
		\EndFunction
	\end{algorithmic}
\end{algorithm}

Motivated by the use of piecewise-linear and piecewise-polynomial trajectories in many robot motion planning and control algorithms \cite{liu2017search, de2000stabilization, franch2009control}, we derive conditions to classify lines and curves, i.e., to check if every point on the curve is free using the trained model. Checking that a curve $\bfp(t)$ is classified as free is equivalent to verifying that $F(\bfp(t)) < 0$, $\forall t \geq 0$. It is not possible to express this condition for $t$ explicitly due to the nonlinearity of $F$. In Prop.~\ref{prop:score_bounds}, we show that an accurate upper bound $\bar{F}(\bfp(t))$ on the score $F(\bfp(t))$ exists and can be used to evaluate the condition $\bar{F}(\bfp(t)) < 0$ explicitly in $t$. The upper bound provides a conservative but fairly accurate ``inflated boundary'' and allows efficient classifications of curves $\bfp(t)$, assuming a radial basis function kernel $k(\bfx, \bfx') = \eta \exp{(-\gamma \Vert \bfx - \bfx' \Vert^2)}$ is used.

\begin{proposition}[\cite{duong2020autonomous}]
\label{prop:score_bounds}
For any $(\bfx_j^-, \alpha_j^-)\in \Lambda^-$, the score $F(\bfx)$ is bounded above by $\bar{F}(\bfx)=k(\bfx, \bfx_*^+) \sum_{i = 1} ^ {M^+} \alpha^+_i \!-\! k(\bfx, \bfx_j^-)\alpha^-_j$ where $\bfx_{*}^{+}$ is the closest positive support vector to $\bfx$.
\end{proposition}


To check if a line $\bfp(t)$ collides with the inflated boundary, we find the first time $t_u$ such that $\bar{F}(\bfp(t_u)) \geq 0$. This means that $\bfp(t)$ is classified as free for $t \in [0,t_u)$.
 
\begin{proposition}[\cite{duong2020autonomous}]
\label{prop:line_curve_defensive_checking} 
Consider a ray $\bfp(t) = \bfp_0 + t\bfv$, $t\geq 0$ such that $\bfp_0$ is classified as free, i.e., $\bar{F}(\bfp_0) < 0$, and $\bfv$ is constant. Let $\bfx_i^+$ and $\bfx_j^-$ be arbitrary positive and negative support vectors. Then, any point $\bfp(t)$ with $t\in [0,t_u) \subseteq [0,t_u^*)$ is free for 
\begin{eqnarray}
t_u &:=&  \min_{i \in \{1, \ldots, M^+\}} \rho (\bfp_0, \bfx_i^+, \bfx_j^-), \label{eq:line_curve_t_condition} \\
t_u^* &:=& \min_{i \in \{1, \ldots, M^+\}} \max_{j \in \{1, \ldots, M^-\}} \rho (\bfp_0, \bfx_i^+, \bfx_j^-), \label{eq:line_curve_t_condition_tigher_bound}
\end{eqnarray}
where $\beta = \frac{1}{\gamma}\left(\log (\alpha_j^-) - \log (\sum_{i = 1} ^ {M^+} \alpha_i^+)\right)$ and \\$\scaleMathLine[1.0]{\rho(\bfp_0, \bfx_i^+, \bfx_j^-) = 
\begin{cases} 
      +\infty, & \text{if}\ \bfv^T(\bfx_i^+ - \bfx_j^-) \leq 0  \\
      \frac{\beta - \Vert \bfp_0 - \bfx_j^-\Vert ^2 - \Vert \bfp_0 + \bfx_i^+ \Vert ^2}{2\bfv^T(\bfx_j^- - \bfx_i^+)}, & \text{if}\ \bfv^T(\bfx_i^+ - \bfx_j^-) > 0 
\end{cases}}$.
\end{proposition}

For a line segment $(\bfp_A, \bfp_B)$, all points on the segment can be expressed as $\bfp(t_A) = \bfp_A + t_A\bfv_A$, $\bfv_A = \bfp_B - \bfp_A$, $0 \leq t_A \leq 1$ or $\bfp(t_B) = \bfp_B + t_B\bfv_B$, $\bfv_B = \bfp_A - \bfp_B$, $0 \leq t_B \leq 1$. Using the upper bound provided by Eq.~\eqref{eq:line_curve_t_condition} or Eq.~\eqref{eq:line_curve_t_condition_tigher_bound}, we find the free regions $[0, t_{uA})$ and $[0, t_{uB})$ starting from $\bfp_A$ and $\bfp_B$, respectively. If the free regions overlap, the segment is classified as free and vice versa.

We extend line segment classification to general curves by finding a Euclidean ball $\mathcal{B}(\bfp_0, r)$ around $\bfp_0$ whose interior is free of obstacles. 

\begin{corollary}[\cite{duong2020autonomous}]
\label{corollary:free_ball}
Let $\bfp_0 \in \mathcal{C}$ be such that $\bar{F}(\bfp_0) <0$ and let $\bfx_i^+$ and $\bfx_j^-$ be arbitrary positive and negative support vectors. Then, every point inside the Euclidean balls $\mathcal{B}(\bfp_0, r_u) \subseteq \mathcal{B}(\bfp_0, r_u^*)$ is free for:
\begin{align}
r_u &:= \min_{i \in \{1, \ldots, M^+\}} \bar{\rho}(\bfp_0, \bfx_i^+, \bfx_j^-), \label{eq:line_curve_radius}\\
r_u^* &:= \min_{i \in \{1, \ldots, M^+\}} \max_{j \in \{1, \ldots, M^-\}} \bar{\rho}(\bfp_0, \bfx_i^+, \bfx_j^-) \label{eq:line_curve_radius_star}
\end{align}
where $\bar{\rho}(\bfp_0, \bfx_i^+, \bfx_j^-) = \frac{\beta - \|\bfp_0 - \bfx_j^-\| ^2 + \|\bfp_0 - \bfx_i^+ \|^2}{2 \| \bfx_j^- - \bfx_i^+ \|}$ and $\beta = \frac{1}{\gamma} \prl{\log (\alpha_j^-) - \log (\sum_{i = 1} ^ {M^+} \alpha_i^+)}$.
\end{corollary}

Consider a polynomial $\bfp(t) = \bfp_0 + \bfa_1 t + \bfa_2 t^2 + \ldots + {\bfa}_d t^d$, $t \in [0, t_f]$  from $\bfp_0$ to $\bfp_f := \bfp(t_f)$. Corollary~\ref{corollary:free_ball} shows that all points inside $\mathcal{B}(\bfp_0, r)$ are free for $r = r_u$ or $r^*_u$. If we can find the smallest positive $t_1$ such that $\|\bfp(t_1) - \bfp_0 \| = r$, then all points on the curve $\bfp(t)$ for $t \in [0, t_1)$ are free. We classify the polynomial curve by iteratively covering it by Euclidean balls. If any ball's radius is smaller than a threshold $\varepsilon$, the curve is considered colliding. Otherwise, it is considered free.


The sparse kernel-based model~\cite{duong2020autonomous} is accurate, updates efficiently from streaming range data, and evaluates curves $\bfp(t)$ for collisions without sampling. However, the model does not provide occupancy probability, which is desirable in autonomous navigation applications for distinguishing between unknown and well-observed free regions and for identifying map areas with large uncertainty. This observation motivates us to develop a sparse probabilistic model for online occupancy classification and efficient collision checking.

\section{Online Probit RVM Training}
\label{sec:online_probit_rvm}

In this section, we develop an online probit relevance vector machine (RVM) training algorithm that builds a sparse probabilistic model for online occupancy mapping from streaming range observations. 

\subsection{Relevance Vector Machine Preliminaries}
\label{subsec:sequential_rvm}


A relevance vector machine~\cite{tipping2003fast} is a sparse Bayesian approach for classification. Given a training dataset of $N$ binary-labeled samples $\mathcal{D} = (\bfX, \bfy) = \{(\bfx_l, y_l)\}_l$, where $y_l \in \{-1, 1\}$, an RVM model maintains a sparse set of relevance vectors $\bfx_m$ for $m = 1, \ldots, M$. The relevance vectors map a point $\bfx$ to a feature vector $\bfPhi_{\bfx} = [k_1(\bfx), k_2(\bfx), \ldots, k_M(\bfx)]^\top \in \mathbb{R}^M$ via a kernel function $k_m(\bfx) := k(\bfx, \bfx_m)$. The likelihood of label $y$ at point $\bfx$ is modeled by squashing a linear feature function:
\begin{equation}
\label{eq:rvm_score_fcn}
F(\bfx) := \bfPhi_{\bfx}^\top\bsym{w} + b, 
\end{equation}
with weights $\bsym{w} \in \mathbb{R}^M$ and bias $b \in \mathbb{R}$ through a  function $\sigma : \mathbb{R} \mapsto [0,1]$:
\begin{equation*}
\mathbb{P}(y = 1|\bfx, \bsym{w}) = \sigma(F(\bfx)), \mathbb{P}(y = -1|\bfx, \bsym{w}) = 1 - \sigma(F(\bfx)).
\end{equation*}
Note that Eq.~\eqref{eq:fastron_score} is a special case of~\eqref{eq:rvm_score_fcn} with $b = 0$. Examples of $\sigma$ are the logistic function $\sigma(f) := \frac{1}{1+\exp(-f)}$ and the probit function $\sigma(f) := \int_{-\infty}^f \varphi(z) dz$, where $\varphi(z) := \frac{1}{\sqrt{2\pi}}\exp(-z^2/2)$ is the standard normal probability density. The data likelihood of the whole training set is:
\begin{equation}
\label{eq:rvm_likelihood}
p(\bsym{y}| \bsym{X}, \bsym{w}) = \prod_{l = 1}^N \sigma(F(\bfx_l))^{\frac{1+y_l}{2}}\left(1-\sigma(F(\bfx_l))\right)^{\frac{1-y_l}{2}}.
\end{equation}

An RVM model imposes a Gaussian prior on each weight $w_m$ with zero mean and precision $\xi_m$ (i.e., variance $1/\xi_m$):
\begin{equation}
\label{eq:rvm_w_prior}
p(\bsym{w}|\bsym{\xi}) = (2\pi)^{M/2} \prod_{m = 1}^M \xi_m ^{1/2} \exp\left(-\frac{\xi_m w_m ^2}{2}\right).
\end{equation}
The weight posterior is obtained via Bayes' rule:
\begin{equation}
\label{eq:rvm_w_posterior}
p(\bsym{w}| \bsym{y}, \bsym{X}, \bsym{\xi}) = \frac{p(\bsym{y}| \bsym{X}, \bsym{w}) p(\bsym{w}|\bsym{\xi})}{p(\bsym{y}|\bsym{X}, \bsym{\xi})}.
\end{equation}
The precision $\bfxi$ is determined via type-II maximum likelihood estimation, i.e., by maximizing the marginal likelihood:
\begin{equation}
\label{eq:rvm_marginal_likelihood}
\calL(\bfxi) = \log p(\bsym{y}|\bsym{X}, \bsym{\xi}) = \log \int {p(\bsym{y}| \bsym{X}, \bsym{w}) p(\bsym{w}|\bsym{\xi}) d\bsym{w}}.
\end{equation}
Given a maximizer $\bfxi$, the posterior $p(\bsym{w}| \bsym{y}, \bsym{X}, \bsym{\xi})$ is generally intractable and approximated by a Gaussian distribution $p(\bsym{w}|\bfy, \bfX, \bfmu, \bfSigma)$ with mean $\bfmu$ and covariance $\bfSigma$ using Laplace approximation~\cite{mackay1992evidence}. Training consists in determining $\bfxi$, $\bfmu$, $\bfSigma$.

At test time, due to the Laplace approximation, the predictive distribution of a query point $\bfx$ becomes:
\begin{equation}
\label{eq:predictive_integral}
p(y|\bfx, \bfxi) \approx \int{p(y|\bfx, \bsym{w})p(\bsym{w}|\bfy, \bfX, \bfmu, \bfSigma)}d\bsym{w}.
\end{equation}
The usual formulation of RVM~\cite{tipping2003fast} uses a logistic function for $\sigma$, requiring additional approximations to the integral in~\eqref{eq:predictive_integral}. We emphasize that using a probit function, instead, enables a closed-form for the predictive distribution:
\begin{align}
p(y|\bfx, \bfxi) &\approx \int \sigma(y(\bfPhi_{\bfx}^\top\bsym{w} + b))p(\bsym{w}|\bfy, \bfX, \bfmu, \bfSigma)d\bsym{w} \notag \\
& = \sigma\left(\frac{y(\Phi_{\bfx}^\top\bfmu + b)}{\sqrt{1 + \Phi_{\bfx}^\top \bfSigma \Phi_{\bfx}}}\right).
\label{eq:rvm_probit_score}
\end{align}
This expression enables our results on closed-form classification of curves in Sec.~\ref{sec:poly_curve_check}. 

We review the details of RVM training and then propose an online training algorithm that handles streaming training data.

\subsubsection{Laplace approximation}
Approximation of the weight posterior $p(\bsym{w}| \bsym{y}, \bsym{X}, \bsym{\xi})$ is performed by fitting a Gaussian density function around its mode $\bfmu$, the maximizer of
\begin{equation}
\label{eq:laplace_approximation}
L(\bsym{w}) := \log(p(\bsym{y}| \bsym{X}, \bsym{w}) p(\bsym{w}|\bsym{\xi})).
\end{equation}
Substituting~\eqref{eq:rvm_likelihood} and~\eqref{eq:rvm_w_prior} in~\eqref{eq:laplace_approximation}, we can obtain the gradient and Hessian of $L(\bsym{w})$ for the probit function $\sigma$:
\begin{equation}
\nabla L(\bsym{w}) = \bfPhi^\top \bfdelta - \bfA\bsym{w}, \quad \nabla^2 L(\bsym{w}) = -\bfPhi^\top\bfB\bfPhi -\bfA,
\end{equation}
where $\bfPhi \in \mathbb{R}^{N \times M}$ is the feature matrix with entries $\bfPhi_{i,j} := k_j(\bfx_i)$, $\bfdelta \in \mathbb{R}^N$ is a vector with entries $\delta_l := \frac{\varphi(y_lF(\bfx_l))}{\sigma(y_lF(\bfx_l))} y_l$, $\bfA := \diag(\bfxi) \in \mathbb{R}^{M \times M}$, $\bfB := \diag(\bfD\bfPhi^\top\bsym{w} + b\bfdelta+ \bfD\bfdelta) \in \mathbb{R}^{N \times N}$, and $\bfD := \diag(\bfdelta) \in \mathbb{R}^{N \times N}$. The Hessian is negative semi-definite and, hence, $L(\bsym{w})$ is concave. Setting $L(\bsym{w}) = 0$, we obtain a Gaussian approximation $p(\bsym{w}|\bfy, \bfX, \bfmu, \bfSigma)$ with:
\begin{align}
\bfSigma &= (\bfPhi^\top\bfB\bfPhi + \bfA)^{-1}, \label{eq:approx_sigma}\\
\bfmu &= \bfSigma\bfPhi^\top\bfB \prl{ \bfPhi\bsym{\mu} + \bfB^{-1}\bfdelta }, \label{eq:approx_mu}
\end{align}
where $\bfmu$ is defined implicitly and is obtained via first- or second-order ascent in practice~\cite{nabney2004efficient}.

\subsubsection{Sequential RVM training}

To the determine the precision $\bsym{\xi}$ of the weight prior in~\eqref{eq:rvm_w_prior}, Tipping and Faul\cite{tipping2003fast} proposed a sequential training algorithm that starts from an empty set of relevance vectors, i.e., $\xi_l = \infty$, and incrementally introduces new vectors to maximize the marginal likelihood in~\eqref{eq:rvm_marginal_likelihood}:
\begin{equation}
\calL(\bfxi) \approx -\frac{1}{2}\prl{ N\log 2\pi + \log \det \bfC + \hat{\bsym{t}}^\top\bsym{C}^{-1}\hat{\bsym{t}} }
\end{equation}
where $\hat{\bft} := \bfPhi\bsym{\mu} + \bfB^{-1}\bfdelta$ and $\bfC := \bfB + \bfPhi\bsym{A}^{-1}\bfPhi^\top$. For each $(\bfx_l, y_l)$ in the training set $\mathcal{D}$, define $\theta_l = q_l^2 - s_l$ as follows:
\begin{equation}
s_l := \begin{cases}
\frac{\xi_l S_l}{\xi_l - S_l},& \text{if $\xi_l < \infty$} \\
S_l, &\text{else} \\
\end{cases} \; q_l := \begin{cases}
\frac{\xi_l Q_l}{\xi_l - S_l},& \text{if $\xi_l < \infty$} \\
Q_l, &\text{else} \\
\end{cases}
\end{equation}
where $S_l = \bfPhi_l^\top \bsym{C}^{-1} \bfPhi_l$, $Q_l = \bfPhi_l^\top \bsym{C}^{-1} \hat{\bsym{t}}$, and $\bfPhi_l$ is the $l$-th row of $\bfPhi$. If $\theta_l > 0$, the point $\bfx_l$ is updated (if $\xi_l < \infty$) or added (if $\xi_l = \infty$) as a relevance vector with $\xi_l = \frac{s_l^2}{q_l^2 - s_l}$. If $\theta_l \leq 0$ and $\xi_l < \infty$, the point $\bsym{x_l}$ is removed from the RVM model. These steps are shown in lines 8-12 of Alg.~\ref{alg:rvm_training_online}.

\begin{algorithm}
\caption{Online Probit RVM Training.}
\label{alg:rvm_training_online}
  \footnotesize
	\begin{algorithmic}[1]
		\Require Relevance vectors $\Lambda_{k} = \{(\bsym{x}_i^{(k)}, y_i^{(k)}, \xi_i^{(k)})\}$; training set $\mathcal{D}_{k+1} = \{(\bfx_l, y_l)\}_l$; number of nearest relevance vectors to use $K$ (optional)
		\Ensure Relevance vectors $\Lambda_{k+1} = \{(\bsym{x}_i^{(k+1)}, y_i^{(k+1)}, \xi_i^{(k+1)})\}$; weight posterior mean $\bfmu$ and covariance $\bfSigma$
		\State Initialize $\Lambda_{k+1} = \Lambda_{k}$.
		\If{$K$ is defined} $\Lambda_{local} = K$ nearest relevance vectors from $\Lambda_k$
		\Else{} $\Lambda_{local} = \Lambda_k$.
		\EndIf
		\State $\bfPhi$ = \Call{FeatureMatrix}{$\Lambda_{local}, \mathcal{D}_{k+1}$}
		\State $\xi_l = \infty$ for each $(\bfx_l, y_l)$ in $\mathcal{D}_{k+1}$
		\State $\bsym{\Sigma}, \bsym{\mu}$ = \Call{LaplaceApproximation}{$\Lambda_{local}, \mathcal{D}_{k+1}$}.
		\While{not converged and max number iterations not reached}
			\State Pick a candidate $(\bfx_m, y_m)$ from $\mathcal{D}_{k+1}$.
			\State Calculate $S_m, Q_m, s_m, q_m, \theta_m$.
			\State If $\theta_m > 0$ and $\xi_m = \infty$, add $(\bfx_m, y_m, \xi_m)$ to $\Lambda_{local}$. 
			\State If $\theta_m \leq 0$ and $\xi_m < \infty$, remove $(\bfx_m, y_m, \xi_m)$ from $\Lambda_{local}$. 
			\State If $\theta_m > 0$ and $\xi_m < \infty$, re-estimate $\xi_m = \frac{s_m^2}{q_m^2 - s_m}$ in $\Lambda_{local}$.
			\State $\bsym{\Sigma}, \bsym{\mu}$ = \Call{LaplaceApproximation}{$\Lambda_{local}, \mathcal{D}_{k+1}$}.
		\EndWhile
		\State $\Lambda_{k+1} = \Lambda_{k+1} \cup \Lambda_{local}$.
		\State $\bsym{\Sigma}, \bsym{\mu}$ = \Call{GlobalPosteriorApproximation}{$\Lambda_{k+1}$}
		\State \Return $\Lambda_{k+1}$, $\bsym{\Sigma}$, $\bsym{\mu}$
	\State
	\Function{FeatureMatrix }{$\Lambda, \mathcal{D}$}
		\State Calculate $\bfPhi_{i,j} = k(\bfx_i, \bfx_j)$ for all $\bfx_j \in \Lambda$ and all $\bfx_i \in \mathcal{D}$
		\State \Return $\bfPhi$
	\EndFunction
	\Function{LaplaceApproximation}{$\Lambda, \mathcal{D}$}
		\State Calculate $\bsym{\Sigma}$, $\bsym{\mu}$ for relevance vectors $\Lambda$ using $\mathcal{D}$ (Eq.~\eqref{eq:approx_sigma} and~\eqref{eq:approx_mu}).
		\State \Return $\bsym{\Sigma}, \bsym{\mu}$.
	\EndFunction
	\Function{GlobalPosteriorApproximation}{$\Lambda$}
		\State \Return \Call{LaplaceApproximation}{$\Lambda, \Lambda$}.
	\EndFunction
	\end{algorithmic}
\end{algorithm}

\begin{figure*}[t]
\centering
\begin{subfigure}[b]{0.33\textwidth}
        \centering
        \includegraphics[width=\textwidth]{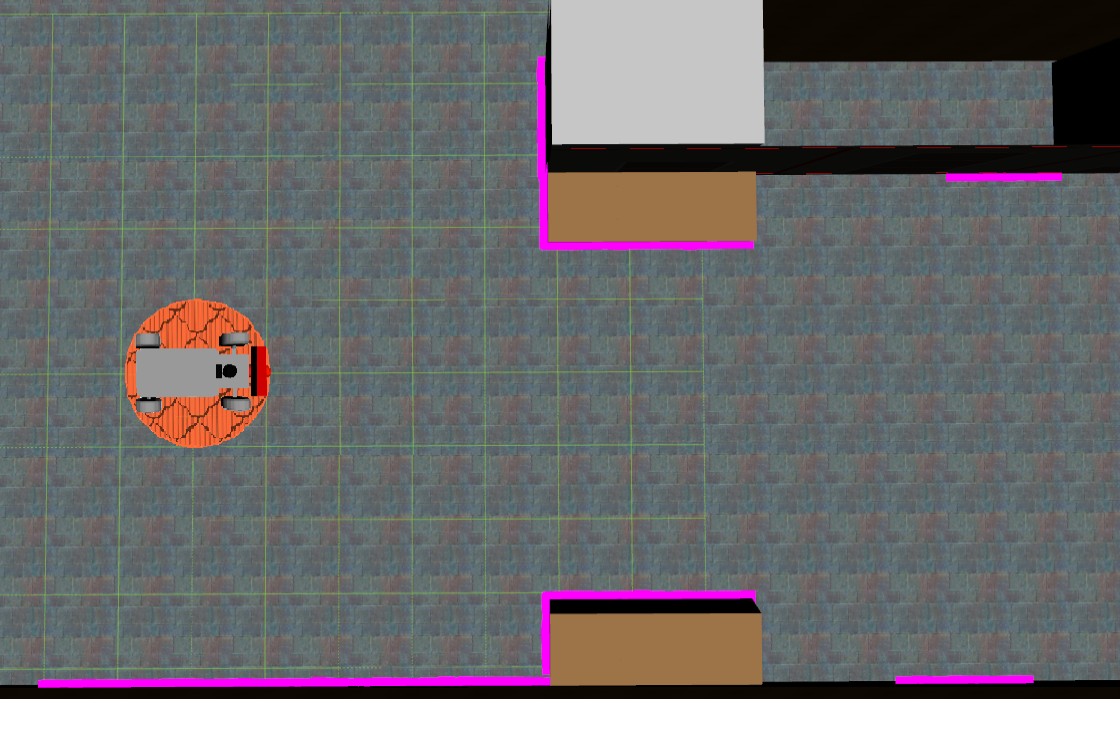}
        \caption{}
        \label{fig:sim_car}
\end{subfigure}%
\hfill%
\begin{subfigure}[b]{0.33\textwidth}
        \centering
        \includegraphics[width=\textwidth]{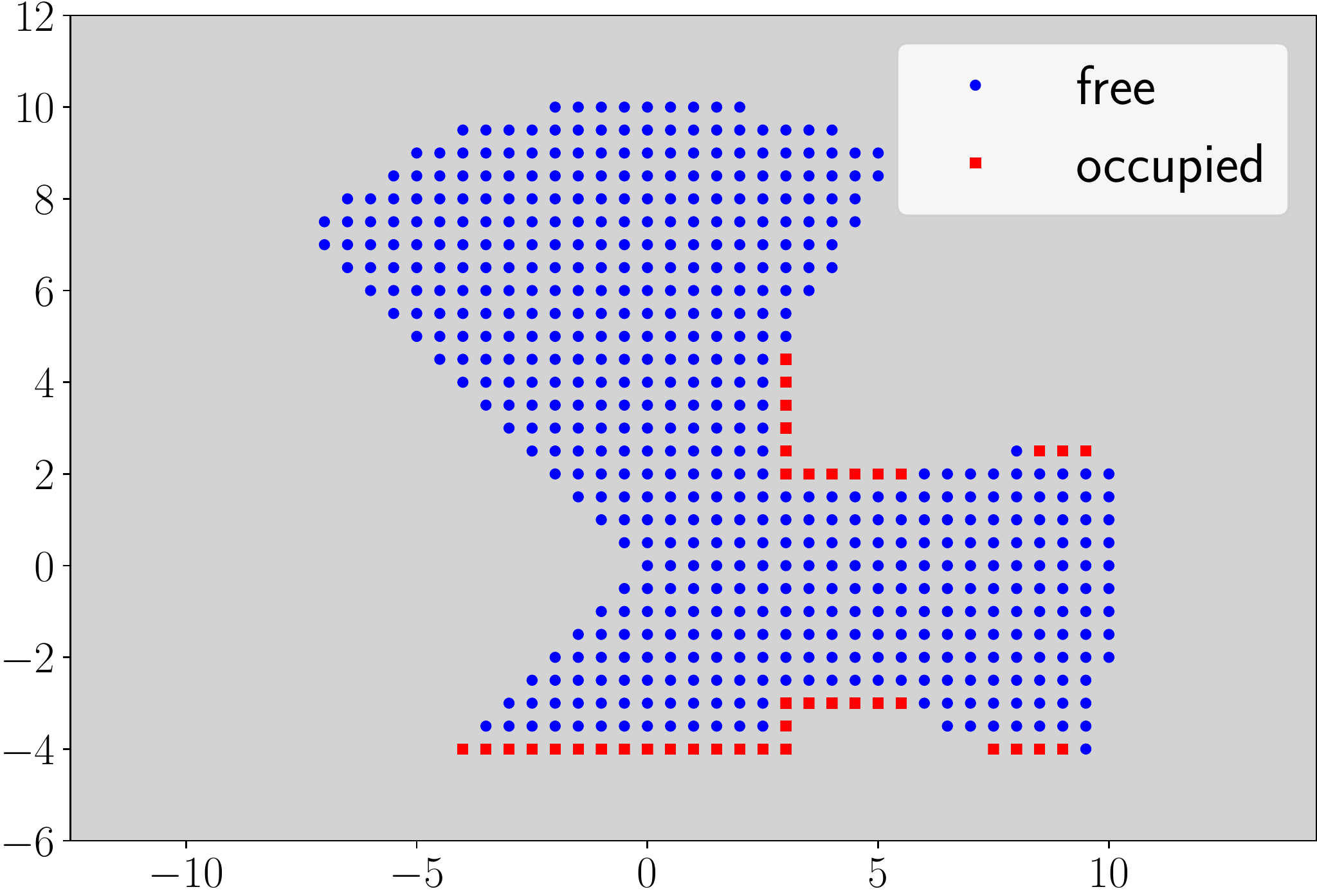}
        \caption{}
        \label{fig:laser_uninflated_ws}
\end{subfigure}%
\hfill%
\begin{subfigure}[b]{0.33\textwidth}
        \centering
        \includegraphics[width=\textwidth]{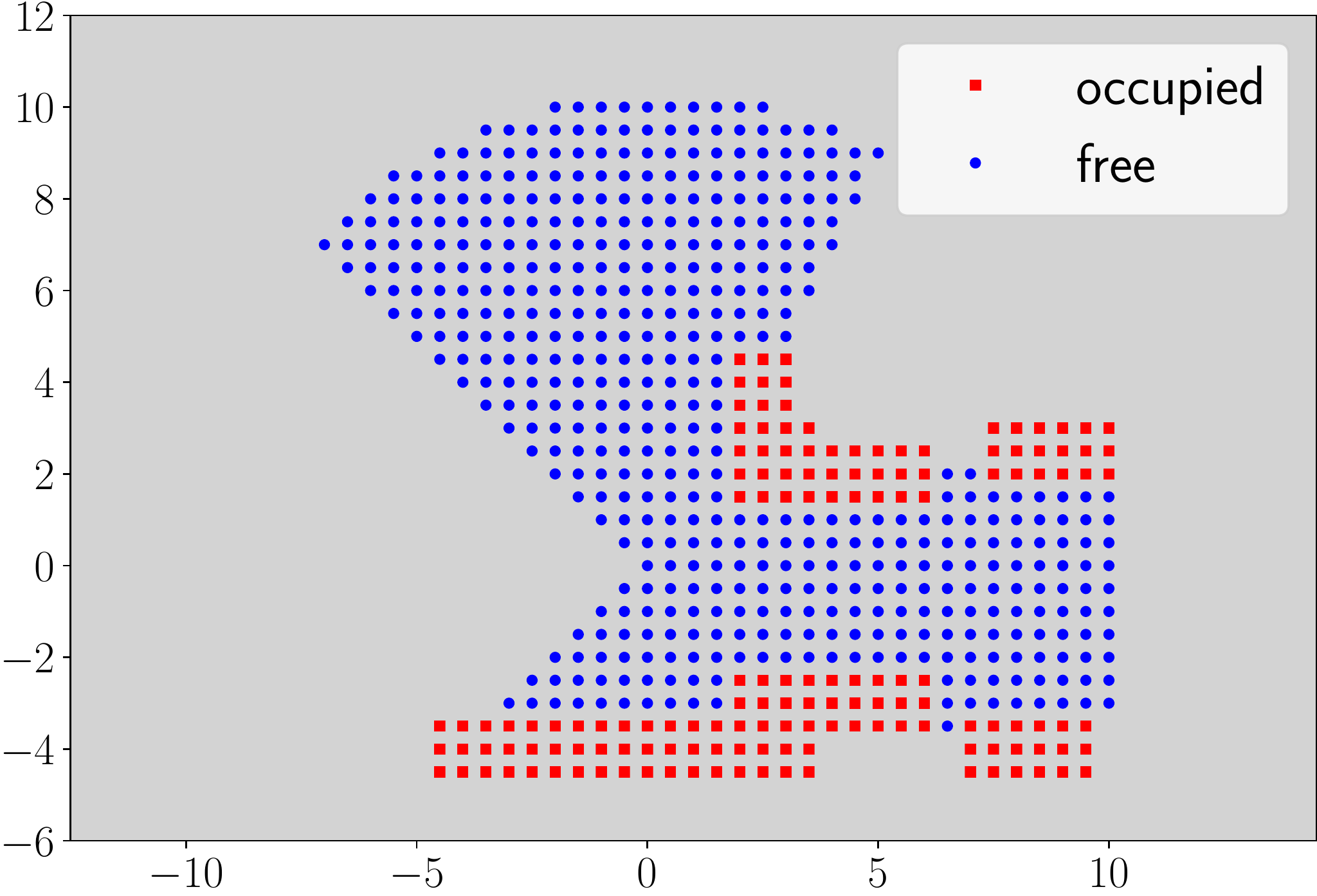}
        \caption{}
        \label{fig:laser_inflated_cs}
\end{subfigure}
\begin{subfigure}[b]{0.33\textwidth}
        \centering
        \includegraphics[width=\textwidth]{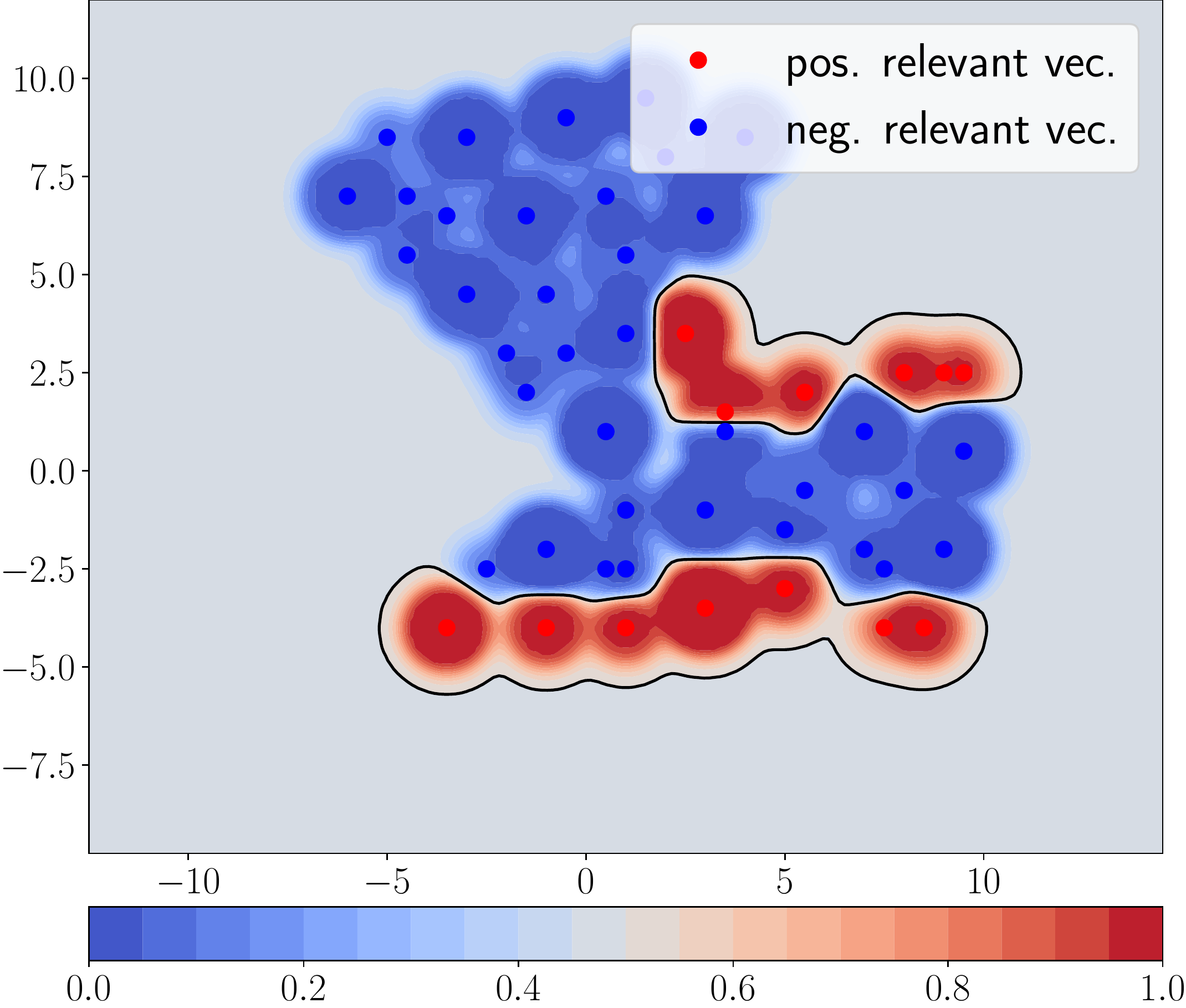}
        \caption{}
        \label{fig:trained_model}
\end{subfigure}%
\hfill%
\begin{subfigure}[b]{0.33\textwidth}
        \centering
        \includegraphics[width=\textwidth]{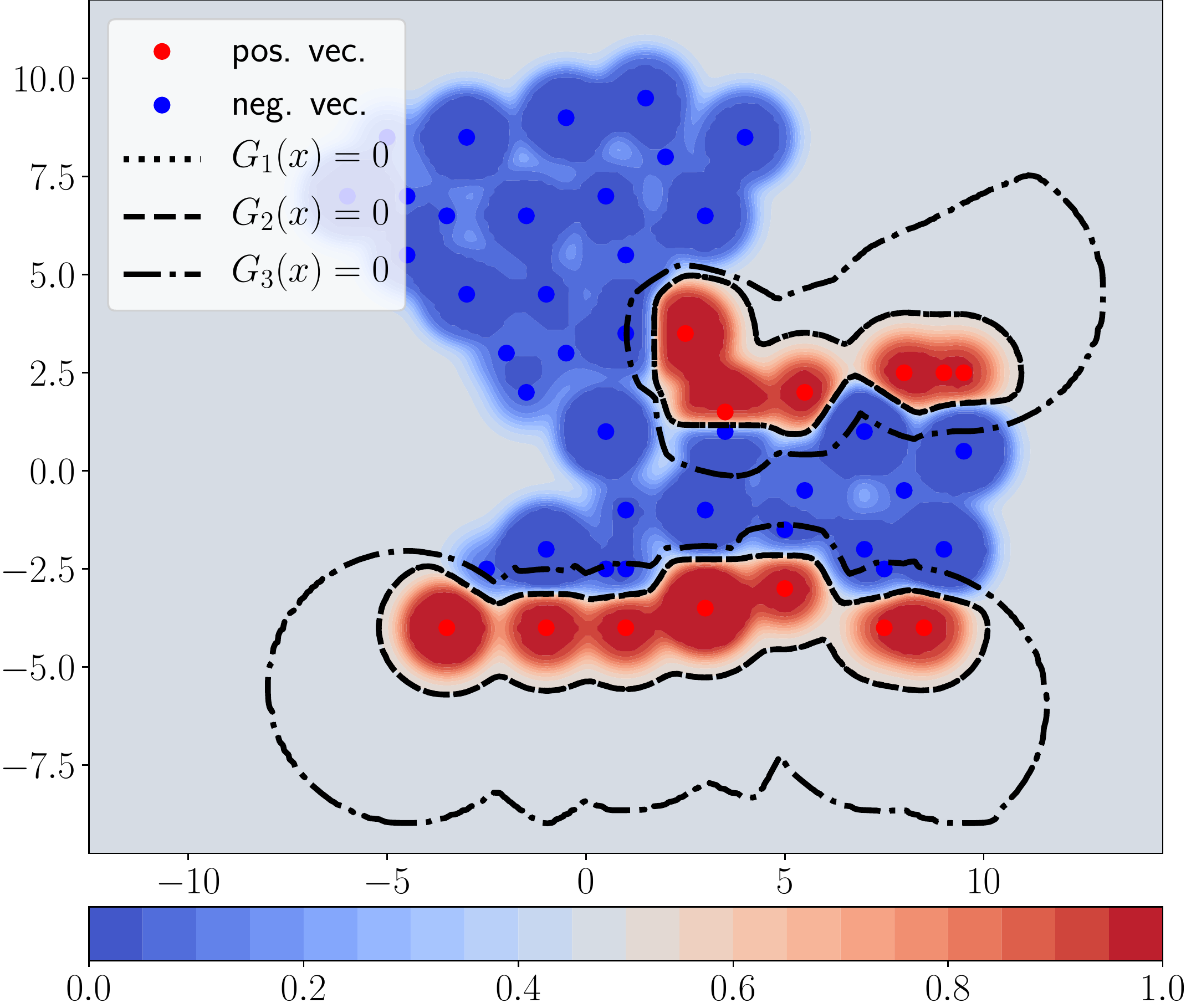}
        \caption{}
        \label{fig:rvm_model_lines}
\end{subfigure}%
\hfill%
\begin{subfigure}[b]{0.33\textwidth}
        \centering
        \includegraphics[width=\textwidth]{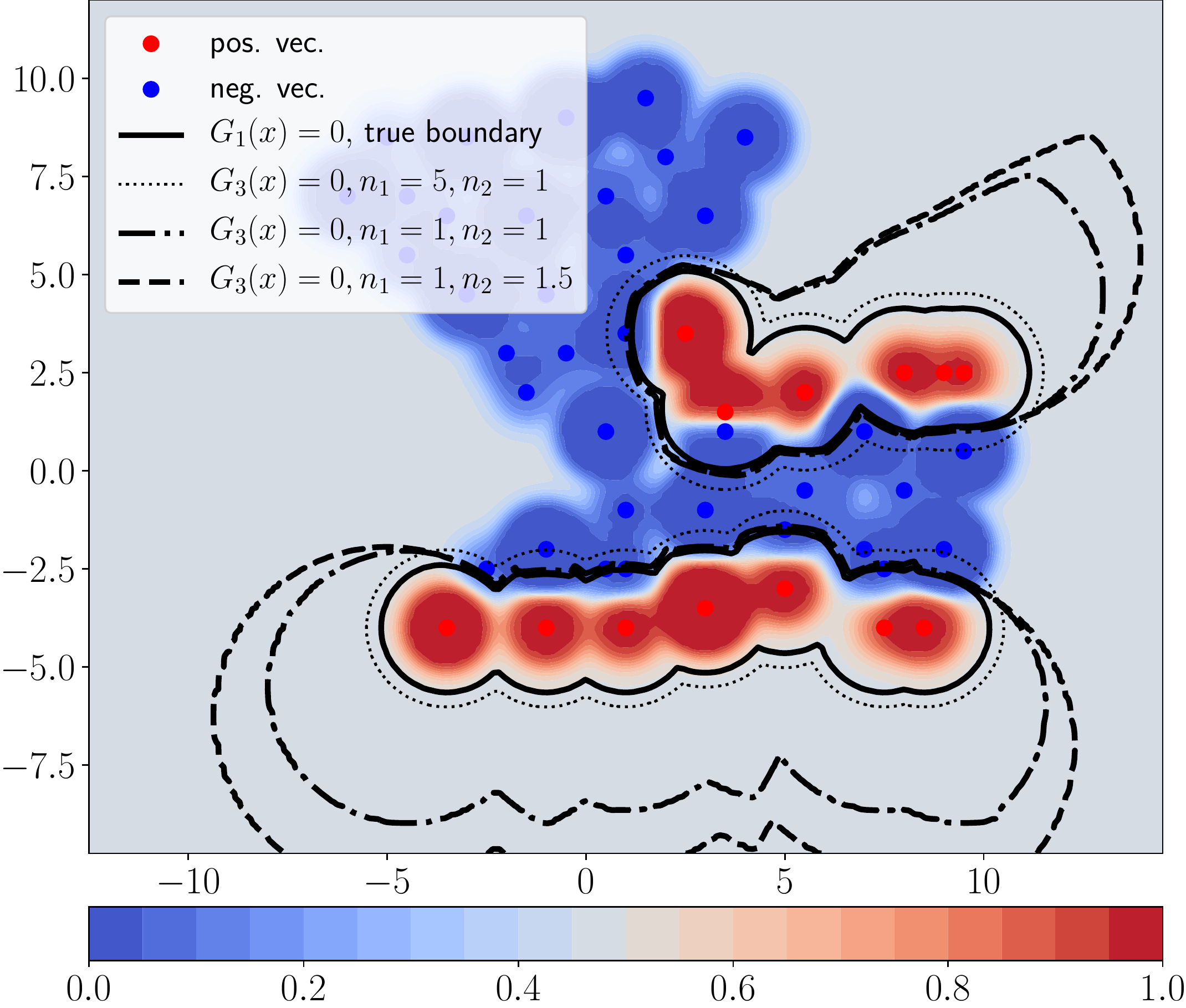}
        \caption{}
        \label{fig:rvm_upperbound_n1n2}
\end{subfigure}
\caption{Example of our mapping method: (a) a robot observing the environment via a laser scan (magenta); (b) work-space samples generated from the laser scan $\bfz_k$; (c) configuration-space samples used as a training set at time $k$; (d) exact decision boundary with bias $b = -0.05$ and classification threshold $e = -0.01$, $\bar{e} = 0.494$; (e) inflated boundaries (c.f. Sec.~\ref{sec:poly_curve_check}) generated by $G_1(\bfx), G_2(\bfx), G_3(\bfx)$ with $n_1 = n_2 = 1$; (f) inflated boundary $G_3(\bfx) = 0$ with various $n_1,n_2$.}
\label{fig:rvm_training_probit}
\end{figure*}

\subsection{Online RVM Training using Streaming Data}
\label{subsec:online_probit_rvm}

Existing techniques for RVM training assume that all data is available a priori. In this section, we develop an online RVM training algorithm that updates the set of relevance vectors $\Lambda_k = \{\bfx_i^{(k)}, y_i^{(k)}, \xi_i^{(k)})\}_i$ incrementally using streaming data. Suppose that $\Lambda_k$ has been obtained based on prior data $\calD_{0}, \ldots, \calD_{k}$. At time $k+1$, a new training set $\calD_{k+1}$ is received. The training set generation depends on the application. We construct $\calD_{k+1}$ using a lidar scan $\bfz_{k+1}$ of an unknown environment as detailed in Sec.~\ref{subsec:online_mapping}. New relevance vectors are added to $\Lambda_k$ to correctly classify the latest training set $\calD_{k+1}$ without affecting the accuracy of the classification on the prior data and maintaining the sparsity of the model.

Alg.~\ref{alg:rvm_training_online} presents our online probit RVM training approach. The algorithm starts with the existing set of relevance vectors $\Lambda_k$ and adds new relevance vectors based on the samples in $\mathcal{D}_{k+1}$ using the sequential training approach in Sec.~\ref{subsec:sequential_rvm}. Instead of using the feature matrix $\bfPhi$ (line 4) associated with all prior relevance vectors, we use a feature matrix approximation based on a local set $\Lambda_{local}$ of $K$ nearest relevance vectors (line 2). Sec.~\ref{sec:complexities} provides a discussion on the computational improvements and assumptions of the score function approximation resulting from using $\Lambda_{local}$ instead of $\Lambda_k$. For test time classification, we compute the mean $\bfmu$ and covariance $\bfSigma$ of the Laplace approximation to the weight posterior according to Eq. \eqref{eq:approx_mu} and \eqref{eq:approx_sigma}. Laplace approximation requires all data $\mathcal{D}=\cup_{i=1}^{k+1} \mathcal{D}_i$, used for training up to time $k+1$ but only the local dataset $\mathcal{D}_{k+1}$ is available. Interestingly, the set $\Lambda_{k+1}$ of relevance vectors itself globally and sparsely represents all the data used for training and, therefore, can be used for Laplace approximation (line 15). If additional computation for Laplace approximation is not feasible, one might directly store the weight mean $\bfmu$ and covariance $\bfSigma$ (line 15) over time. The memory requirements for either case are discussed in Sec. \ref{sec:complexities}.

Fig.~\ref{fig:sim_car} depicts a ground robot equipped with a lidar scanner whose goal is to build an occupancy map of the environment. Fig.~\ref{fig:laser_inflated_cs} plots the training set $\mathcal{D}_{k+1}$ generated from the lidar scan $\bfz_{k+1}$, assuming the current set of relevance vectors $\Lambda_k$ is empty. Fig.~\ref{fig:trained_model} shows the trained RVM model as a sparse set of relevance vectors, serving as a sparse probabilistic occupancy map of the environment, incrementally updated via the streaming lidar scans. A map representation is useful for autonomous navigation (Problem~\ref{problem_formulation_unknown_env}) only if it allows checking potential robot trajectories $\bfs(t)$ for collisions. We propose classification methods for points, line segments, and general curves next.

\section{RVM Classification of Points, Lines, and Curves}
\label{sec:poly_curve_check}

This section discusses classification using the predictive distribution in Eq. \eqref{eq:rvm_probit_score} and makes a connection with our preliminaries results in \cite{duong2020autonomous}. Commonly, machine learning models are only able to classify point queries but applications, such as robot trajectory planning, may requires classification of general curves. This can be done by successively checking a dense set of points, sampled along the curve. However, we show that under certain assumptions on the kernel function and the decision threshold, line and general curve classification based on the RVM decision boundary can be performed directly and efficiently, based on the closed-form of the predictive distribution in Eq.~\eqref{eq:rvm_probit_score}, without the need to sample.

\subsection{RVM Classification of Points}
\label{subsec:point_classification}

Consider a set $\Lambda$ of $M$ relevance vectors with prior weight precision $\bfxi$ and mean $\bfmu$ and covariance $\bfSigma$ of the approximate weight posterior $p(\bsym{w}|\bfy, \bfX, \bfmu, \bfSigma)$. To classify a query point $\bfx$ using the RVM model, we place a threshold $\bar{e}=\sigma(e)$ on the probability $\mathbb{P}(y = 1| \bfx, \bfxi)$ (Def.~\ref{def:threshold_classification}). Fig.~\ref{fig:trained_model} illustrates the decision boundary defined by Def.~\ref{def:threshold_classification} with $\bar{e} = \sigma(e) = 0.494$, i.e., $e = -0.01$.
\begin{definition}
\label{def:threshold_classification}
Let $\bar{e} \in [0,1]$ and $e := \sigma^{-1}(\bar{e})$. A point $\bfx$ is classified as ``-1'' (\NEW{free}) if
\begin{equation}
\label{eq:threshold_classification_cond1}
\mathbb{P}(y = 1| \bfx, \bfxi) =\sigma\left(\frac{\Phi_{\bfx}^\top\bfmu + b}{\sqrt{1 + \Phi_{\bfx}^\top \bfSigma \Phi_{\bfx}}}\right) \leq \bar{e},
\end{equation}
or, equivalently, if
\begin{equation}
\label{eq:rvm_classification_cond3}
G_1(\bfx) := \bfPhi_{\bfx}^\top\bfmu + b - e\sqrt{1 + \bfPhi_{\bfx}^\top \bfSigma \bfPhi_{\bfx}} \leq 0.
\end{equation}
\end{definition}

%


The condition in Eq. \eqref{eq:rvm_classification_cond3} can be verified for a given point but it is challenging to obtain an explicit expression in terms of $\bfx$. If, instead of a point $\bfx$, we consider a time-parameterized curve $\bfp(t)$, then Eq. \eqref{eq:rvm_classification_cond3} becomes a nonlinear programming feasibility problem in $t$. To avoid nonlinear programming, we develop a series of upper bounds for $G_1(\bfx)$ that make the condition for classifying a point as free (i.e., $y = -1$) more conservative but with a simpler dependence on $\bfx$.

\begin{restatable}{proposition}{propRvmProbBound}
\label{prop:rvm_prob_bound}
For a non-negative kernel function $k_m(\bfx) := k(\bfx, \bfx_m)$, a point $\bfx$ is classified as ``-1'' if
\begin{equation}
\label{eq:rvm_prob_bound}
G_2(\bfx) :=\! \sum_{m = 1} ^{M}(\mu_{m} - e\mathbbm{1}_{\{\textcolor{blue}{e\leq 0}\}}\sqrt{\lambda_{\max}})k_m(\bfx) + b - e \leq 0,
\end{equation}
where $\lambda_{\max} \geq 0$ is the largest eigenvalue of the covariance $\bfSigma$, $\mu_{m}$ is the $m$th element of the mean $\bfmu$, and $\mathbbm{1}_{\{\textcolor{blue}{e\leq 0}\}}$ is an indicator function which equals $1$ if $\textcolor{blue}{e\leq 0}$ and $0$, otherwise.
\end{restatable}




\begin{proof}
Please refer to Appendix \ref{appendix:rvm_prob_bound}.
\end{proof}

The relaxed condition in Eq.~\eqref{eq:rvm_prob_bound} adjusts the weights of the relevance vectors by an amount of $\delta\mu = -e\mathbbm{1}_{\{\textcolor{blue}{e\leq 0}\}}\lambda_{max} \geq 0$. Intuitively, this increases the effect of the positive relevance vectors, leading to a more conservative condition than Def.~\ref{def:threshold_classification}. Prop.~\ref{prop:rvm_prob_bound} also allows us to use only the largest eigenvalue $\lambda_{max}$ of $\bfSigma$ for point classification, which is easier to obtain and store than the whole covariance matrix $\bfSigma$. Methods for computing $\lambda_{max}$ are discussed in Sec.~\ref{subsec:comp_improvements}.

To simplify the notation, let $\nu_m := \mu_{m} - e\mathbbm{1}_{\{\textcolor{blue}{e\leq 0}\}}\lambda_{max}$ be the corrected relevance vector weights and split $\Lambda$ into $M^+$ positive relevance vectors $\Lambda^+ = \{(\bfx_m^+, \nu_m^+)\}$ and $M^-$ negative relevance vectors $\Lambda^- = \{(\bfx_m^-, \nu_m^-)\}$, where $\nu_m^+ = \nu_m$ if $\nu_m > 0$ and $\nu_m^- = -\nu_m$ if $\nu_m < 0$. Now, Eq. \eqref{eq:rvm_prob_bound} can be re-written as:
\begin{equation}
\label{eq:rvm_prob_bound_rewritten}
	G_2(\bfx) = \sum_{i = 1} ^{M^+}\nu_i^+k(\bfx, \bfx_i^+) - \sum_{j = 1} ^{M^-}\nu_j^-k(\bfx, \bfx_j^-) + b - e \leq 0.
\end{equation}
%
Hence, Prop.~\ref{prop:rvm_prob_bound} allows us to make an important connection between sparse kernel classification with a `hard' decision threshold (Sec.~\ref{sec:kernel_map_summary}) and its Bayesian counterpart (Sec.~\ref{subsec:online_probit_rvm}). Specifically, after the relevance vector weight correction, Eq.~\eqref{eq:rvm_prob_bound} is equivalent to the kernel perceptron score in Eq.~\eqref{eq:fastron_score} except for the bias term $b-e$.

\subsubsection{The role of the bias term}

One of the motivations for developing a Bayesian map representation is to distinguish between observed and unobserved regions in the environment. Intuitively, as a query point $\bfx$ is chosen further away from ``observed'' regions, where training data has been obtained, its correlation with existing relevance vectors, measured by $k(\bfx,\bfx_m)$, decreases. To capture and exploit this property, we assume that the kernel has a common radial basis function structure that depends only on a quadratic norm $\|\bfGamma(\bfx - \bfx_m)\|$. 

\begin{assumption}
\label{assumption:rbf}
Let $k(\bfx, \bfx_m) := \eta\exp\prl{-\|\bfGamma(\bfx - \bfx_m)\|^2}$ with parameters $\eta > 0$ and $\bfGamma \in \mathbb{R}^{d \times d}$.
\end{assumption}

In our application, the kernel parameters $\eta$ and $\bfGamma$ may be optimized offline via automatic relevance determination~\cite{neal2012bayesian} using training data from known occupancy maps. Under this assumption, the feature vector $\bfPhi_{\bfx}$ tends to $0$ as $\bfx$ goes towards unobserved regions and the occupancy probability $\mathbb{P}(y = 1|\bfx, \bfxi)$ tends to $\sigma(b)$ in Eq. \eqref{eq:threshold_classification_cond1}. Therefore, the value of $\sigma(b)$ represents the occupancy probability of points in the unknown regions. In other words, $\sigma(b)$ specifies how much we trust that unknown regions are occupied and should be a constant. For this reason, the bias $b$ is fixed in our online RVM training algorithm. \NEW{If we are optimistic about the unknown regions, the parameter $b$ can be set to a large negative number, i.e. $\sigma(b) \approx 0$, and the decision boundary shrinks towards the occupied regions. If we want the robot to be cautious about the unknown regions, the parameter $b$ can be set to a large positive number, i.e. $\sigma(b) \approx 1$, and the decision boundary expands towards the unknown regions. }A common assumption in motion planning~\cite{freespaceassumption} is to treat unknown regions as free in order to allow trajectory planning to goals in the unknown space. In the context of this paper, this means that the occupancy probability of points in unknown regions, $\sigma(b)$, should be lower than or equal to the decision threshold $\bar{e} = \sigma(e)$ in Def.~\ref{def:threshold_classification}.

\begin{assumption}
\label{assumption:threshold_b}
Assume that $e \geq b$ and, hence, $\bar{e} \geq \sigma(b)$.
\end{assumption}



\subsubsection{RVM Classification with $e = b$}
A natural choice for the occupancy probability of unknown regions, $\sigma(b)$, is to set it exactly equal to the decision threshold between free and occupied space, i.e., $e = b$. \NEW{While the sparse kernel-based map (SKM) in our preliminary work \cite{duong2020autonomous} does not have either the bias parameter $b$ or the threshold $e$ in its model, Eq.~\eqref{eq:rvm_prob_bound_rewritten} \NEW{have the same form} as Eq.~\eqref{eq:fastron_score} when $b=e$, and are exactly equivalent (i.e. $\nu_m = \mu_m, \forall m$) when $b=e=0$. Therefore, all results in Sec.~\ref{sec:kernel_map_summary} for classification of points, lines, and curves can be reused.}


\begin{corollary}
\label{corollary:rvm_fastron_connection}
If the bias $b$ in Eq.~\eqref{eq:rvm_score_fcn} is used as the decision threshold for classification in Def.~\ref{def:threshold_classification}, i.e., $e = b$, then, according to Prop.~\ref{prop:rvm_prob_bound} and Eq.~\eqref{eq:rvm_prob_bound_rewritten}, $\bfx$ is classified as ``-1'' if:
\begin{equation}
\label{eq:rvm_prob_bound_be_equal}
  \sum_{i = 1} ^{M^+}\nu_i^+k(\bfx, \bfx_i^+) - \sum_{j = 1} ^{M^-}\nu_j^-k(\bfx, \bfx_j^-) \leq 0.
\end{equation}
Hence, Prop.~\ref{prop:line_curve_defensive_checking} and Corollary~\ref{corollary:free_ball} hold for line and curve classification using a Relevance Vector Machine model.
\end{corollary}

\subsubsection{RVM Classification with $e \geq b$}
For a general decision threshold, $e \geq b$, and a kernel function $k_m(\bfx)$ satisfying Assumption~\ref{assumption:rbf}, we develop an explicit condition for classifying a point $\bfx$ as free.


\begin{restatable}{proposition}{propRvmProbBoundAMGM}
\label{prop:rvm_prob_bound_amgm}
For integers $n_1, n_2 \geq 1$, define $\rho(a,b) := (n_1+n_2){ \left(\frac{a}{n_1} \right)^{\frac{n_1}{n_1+n_2}} \left(\frac{b}{n_2}\right)^{\frac{n_2}{n_1+n_2}}}$. A point $\bfx$ is classified as ``-1'' if
\begin{equation}
\scaleMathLine[0.89]{G_3(\bfx) \!:=\! \biggl(\sum_{i = 1} ^{M^+}\nu_i^+\!\biggr) k(\bfx,\bfx_*^+) - \rho(e-b,\nu_j^- k(\bfx,\bfx_j^-)) \leq 0,}
\end{equation}
where $\bfx^+_*$ is the closest positive relevance vector to $\bfx$ and $\bfx_j^-$ is any negative relevance vector.
\end{restatable}


\begin{proof} 
Please refer to Appendix \ref{appendix:rvm_prob_bound_amgm}. 
%
\end{proof}


Fig.~\ref{fig:rvm_model_lines} illustrates the exact RVM decision boundary from Eq.~\eqref{eq:rvm_classification_cond3}, $G_1(\bfx) = 0$, and the boundaries $G_2(\bfx) = 0$ and $G_3(\bfx) = 0$ resulting from the upper bounds in Prop.~\ref{prop:rvm_prob_bound} and Prop.~\ref{prop:rvm_prob_bound_amgm}. Note that the boundary generated by $G_2(\bfx)$ is very close to the true boundary from $G_1(\bfx)$, \NEW{empirically showing that the bound $G_2(\bfx)$ is tight}. The upper bound $G_3(\bfx)$ provides a conservative ``inflated boundary'', whose accuracy can be controlled via the integers $n_1, n_2$ in Prop.~\ref{prop:rvm_prob_bound_amgm}. Note that $G_3(\bfx)$ is inaccurate mainly in the unknown regions because the Arithmetic Mean-Geometric Mean inequality used in Prop.~\ref{prop:rvm_prob_bound_amgm}'s proof (Appendix \ref{appendix:rvm_prob_bound_amgm}) effectively replaces the kernel function $k(\bfx,\bfx^-_j)$ by a slower decaying one $k(\bfx,\bfx^-_j)^{\frac{n_2}{n_1+n_2}}$. This suits the intuition that unknown regions should be categorized as free more cautiously. Fig.~\ref{fig:rvm_upperbound_n1n2} shows that increasing the ratio $n_2/n_1$ makes the ``inflated boundary" closer to the true decision boundary in the unknown regions but slightly looser in the well-observed regions and vice versa. Next, based on Prop.~\ref{prop:rvm_prob_bound_amgm}, we develop conditions for classification of lines and curves when $e \geq b$ without the need for sampling.

\begin{figure*}[t]
\centering
\begin{subfigure}[b]{0.33\textwidth}
        \centering
        \includegraphics[width=\textwidth]{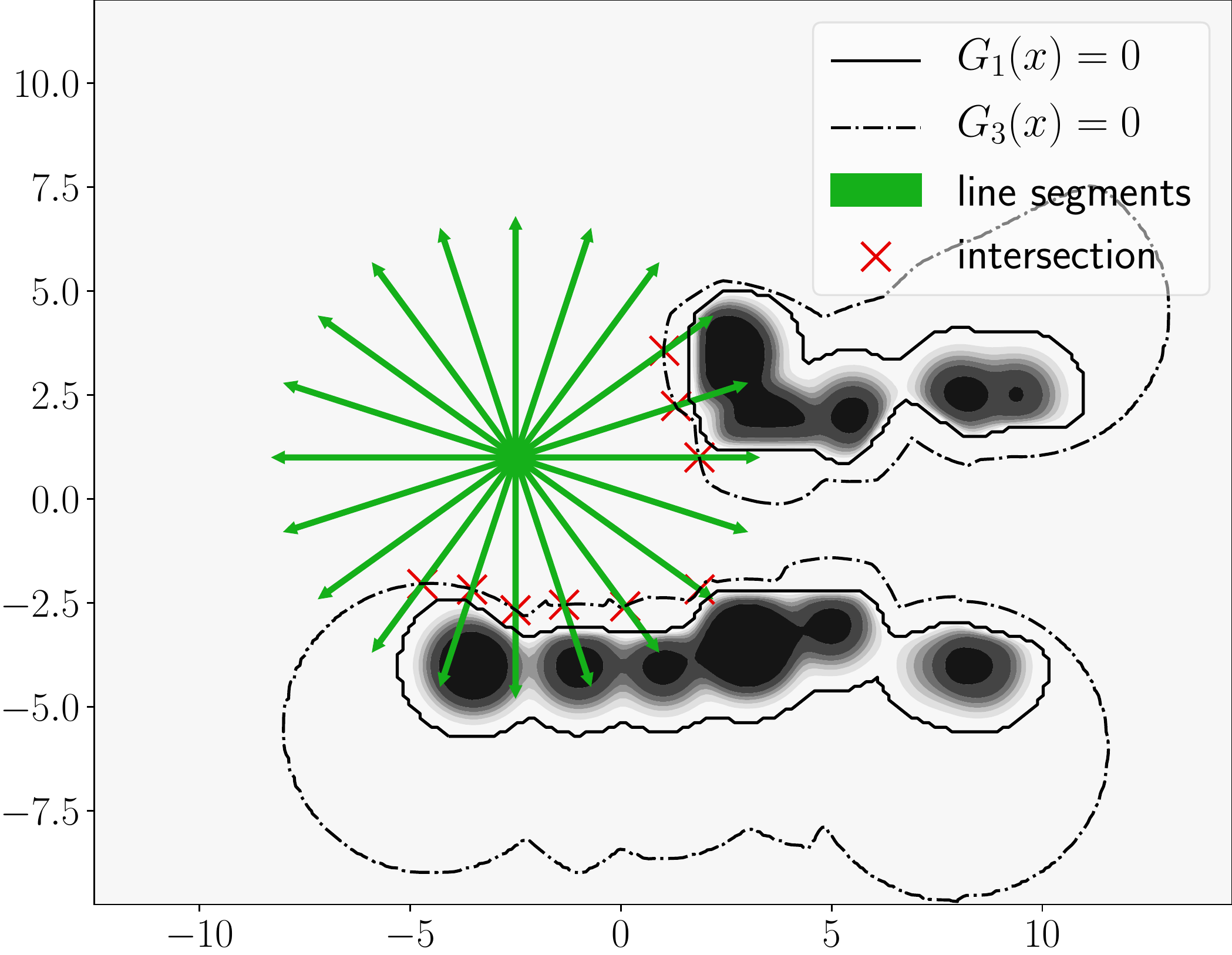}
        \caption{Checking line segments.}
        \label{fig:checking_line_example}
\end{subfigure}%
\hfill
\begin{subfigure}[b]{0.33\textwidth}
        \centering
        \includegraphics[width=\textwidth]{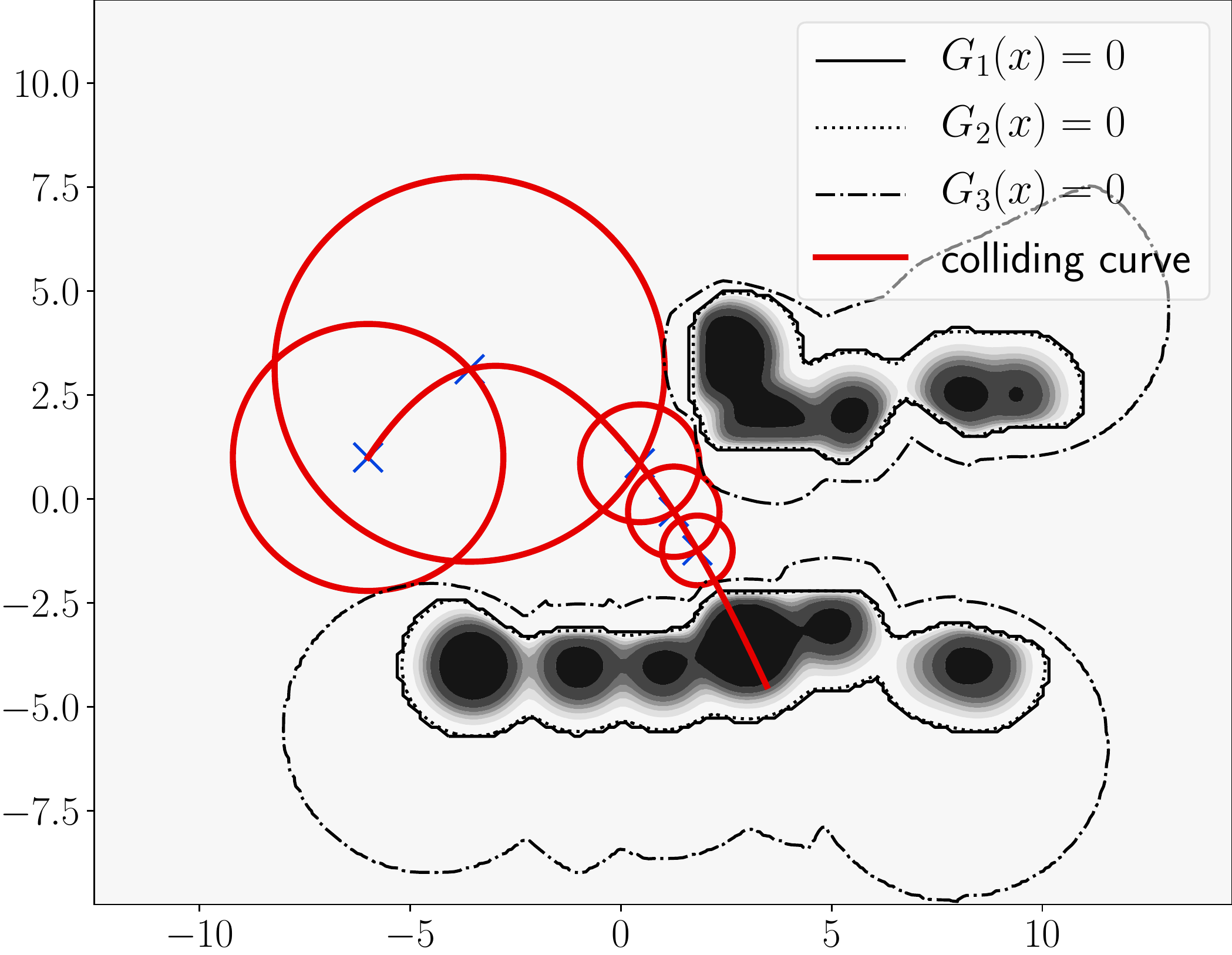}
        \caption{Checking a colliding curve.}
        \label{fig:checking_curve_example_colliding}
\end{subfigure}%
\hfill
\begin{subfigure}[b]{0.33\textwidth}
        \centering
        \includegraphics[width=\textwidth]{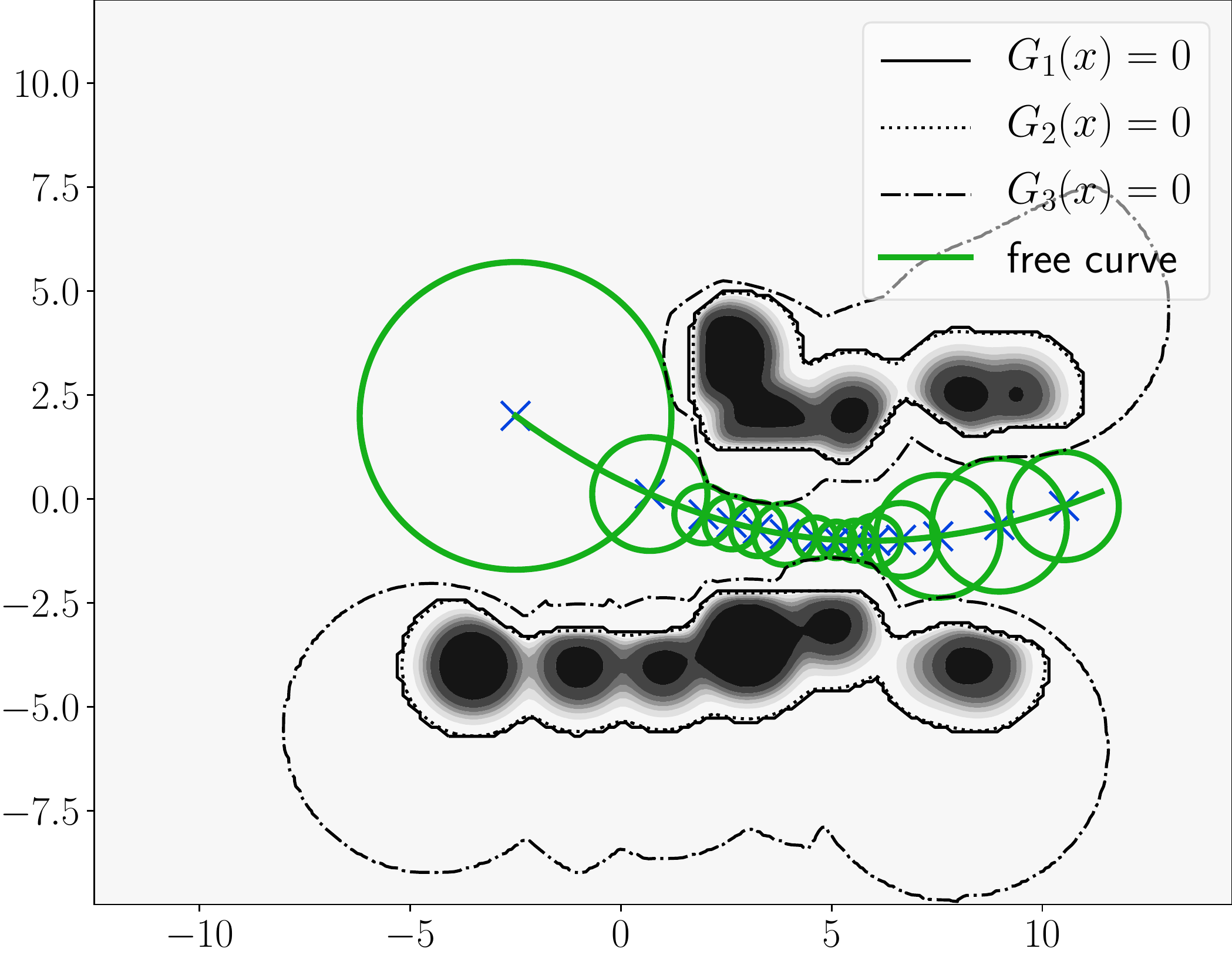}
        \caption{Checking a free curve.}
        \label{fig:checking_curve_example_free}
\end{subfigure}%
\caption{Illustration of our classification algorithms for the trained RVM model in Fig. \ref{fig:rvm_training_probit} with $b = -0.05, e = -0.01$, and $n_1 = n_2 = 1$.}
\label{fig:collision_checking_illustration}
\end{figure*}

\subsection{RVM Classification of Lines}
Consider a linear trajectory described by a ray $\bfp(t) = \bfp_0 + t\bfv$, $t\geq 0$ such that $\bfp_0$ is obstacle-free according to Prop.~\ref{prop:rvm_prob_bound_amgm}, i.e., $G_3(\bfp_0) \leq 0$, and $\bfv$ is a constant. To check if $\bfp(t)$ collides with the inflated boundary $G_3(\bfx) = 0$, we find a time $t_u$ such that any point $\bfp(t)$ is classified free for $t \in [0,t_u)$.


\begin{restatable}{proposition}{RvmProbBoundLineSeg}
\label{prop:rvm_prob_bound_lineseg}
Consider a ray $\bfp(t) = \bfp_0 + t\bfv$, $t\geq 0$. Let $\bfx_i^+$ and $\bfx_j^-$ be arbitrary positive and negative relevance vectors. Then, any point $\bfp(t)$ with $t\in [0,t_u) \subseteq [0,t_u^*)$ is free for:
\begin{eqnarray}
t_u &:=& \min_{i = 1, \ldots, M^+} {\tau(\bfp_0, \bfx^+_i, \bfx_j^-)} \label{eq:rvm_tu}\\
t_u^* &:=& \min_{i =1, \ldots, M^+}\max_{j = 1, \ldots, M^-} {\tau(\bfp_0, \bfx^+_i, \bfx_j^-)}, \label{eq:rvm_tu_star}
\end{eqnarray}
where  $\tau(\bfp_0, \bfx^+_i, \bfx_j^-)$
 \begin{eqnarray}
&=&\begin{cases} 
      +\infty, & \text{if } V(t, \bfx^+_i,\bfx^-_j)  \text{ has less than 2 roots} \\
      +\infty, & \text{if } V(t, \bfx^+_i,\bfx^-_j) \text{ has 2 roots $t_1 < t_2 \leq 0$} \\
      t_1 & \text{if } V(t, \bfx^+_i,\bfx^-_j)  \text{ has 2 roots $0 \leq t_1 < t_2$} \\
      0 & \text{if } V(t, \bfx^+_i,\bfx^-_j)  \text{ has 2 roots $t_1 \leq 0 \leq t_2$} 
\end{cases}. \nonumber
\end{eqnarray}
and $V(t, \bfx^+_i,\bfx^-_j) = at^2 + b(\bfx_i^+, \bfx_j^-)t+ c(\bfx_i^+, \bfx_j^-)$ with
\small
\begin{eqnarray}
a &:=& -n_1\Vert{\bf\Gamma}\bfv\Vert^2, \nonumber \\
b(\bfx_i^+, \bfx_j^-)\!&:=&\!-2\bfv^\top {\bf\Gamma}^\top {\bf\Gamma}(n_1\bfp_0 - (n_1 + n_2)\bfx^+_i + n_2\bfx^-_j), \nonumber \\
c(\bfx_i^+, \bfx_j^-)\!&:=&\!-(n_1\! +\! n_2)\Vert{\bf\Gamma}(\bfp_0- \bfx^+_i)\Vert^2\! +\! n_2\Vert{\bf\Gamma}(\bfp_0- \bfx^-_j)\Vert^2 \nonumber \\
&&\!-\!(n_1 + n_2)\log\frac{\rho(e-b,\nu_j^-)}{\eta^{\frac{n_1}{n_1+n_2}}\sum_{i = 1} ^{M^+}\nu_i^+}.\nonumber
\end{eqnarray}
\normalsize
\end{restatable}

\begin{proof}
Please refer to Appendix \ref{appendix:rvm_prob_bound_lineseg}.
\end{proof}

For a line segment $(\bfp_A, \bfp_B)$, all points on the segment can be expressed as $\bfp(t_A) = \bfp_A + t_A\bfv_A$, $\bfv_A = \bfp_B - \bfp_A$, $0 \leq t_A \leq 1$ or $\bfp(t_B) = \bfp_B + t_B\bfv_B$, $\bfv_B = \bfp_A - \bfp_B$, $0 \leq t_B \leq 1$. Using the upper bound $t_{uA}$ on $t_A$ provided by Eq.~\eqref{eq:rvm_tu} or Eq.~\eqref{eq:rvm_tu_star}, we find the free region on $(\bfp_A, \bfp_B)$ starting from $\bfp_A$. Likewise, we calculate $t_{uB}$ which specifies the free region from $\bfp_B$. If $t_{uA} + t_{uB} > 1$, the entire line segment is free, otherwise the segment is considered colliding. The proposed approach is summarized in Alg.~\ref{alg:rvm_collision_checking_line} and illustrated in Fig.~\ref{fig:checking_line_example} for the trained RVM model in Fig.~\ref{fig:rvm_training_probit}.

\begin{algorithm}[t]
\caption{RVM Line Classification}
\label{alg:rvm_collision_checking_line}
\footnotesize
	\begin{algorithmic}[1]	  
		\Require Line segment $(\bfp_A, \bfp_B)$; relevance vectors $\Lambda = \{(\bfx_i, y_i, \bfxi_i)\}$; weight posterior mean $\bfmu$ and max covariance eigenvalue $\lambda_{max}$
		\State $\bfv_A = \bfp_B - \bfp_A$, $\bfv_B = \bfp_A - \bfp_B$
		\State Calculate $t_{uA}$ and $t_{uB}$ using Eq.~\eqref{eq:rvm_tu} or Eq.~\eqref{eq:rvm_tu_star}.
    \If{$t_{uA} + t_{uB} > 1$} \Return True (Free)
		\Else \;\Return False (Colliding)
		\EndIf
	\end{algorithmic}
\end{algorithm}

\subsection{RVM Classification of Curves}
Instead of a constant velocity $\bfv$ representing the direction of motion, we can define a general curve $\bfp(t)$ by considering a time-varying term $\bfv(t)$. We extend the collision checking conditions in Prop.~\ref{prop:rvm_prob_bound_lineseg} by finding an ellipsoid $\mathcal{E}(\bfp_0, r) := \{\bfx: \Vert\bf\Gamma (\bfx - \bfp_0)\Vert \leq r\}$ around $\bfp_0$ whose interior is free of obstacles, \NEW{where $\Gamma$ is the kernel parameter defined in Assumption \ref{assumption:rbf}. This specific form of the ellipsoid leads a closed-conditions as shown in the Prop.~\ref{prop:rvm_prob_bound_ball}.} 


\begin{restatable}{proposition}{RvmProbBoundEucBall}
\label{prop:rvm_prob_bound_ball}
Let $\bfp_0$ be such that $G_3(\bfp_0) < 0$ and let $\bfx_i^+$ and $\bfx_j^-$ be arbitrary positive and negative support vectors. Then, every point inside the ellipsoids $\mathcal{E}(\bfp_0, r_u) \subseteq \mathcal{E}(\bfp_0, r_u^*)$ is free for:
\begin{eqnarray}
r_u &=& \min_{i = 1, \ldots, M^+} {r(\bfp_0, \bfx^+_i, \bfx_j^-)} \label{eq:rvm_ru} \\
r_u^* &=& \min_{i =1, \ldots, M^+}\max_{j = 1, \ldots, M^-} {r(\bfp_0, \bfx^+_i, \bfx_j^-)}. \label{eq:rvm_ru_star}
\end{eqnarray}
 where $r(\bfp_0, \bfx^+_i, \bfx_j^-)$
 \begin{eqnarray}
&=&\begin{cases} 
      +\infty, & \text{if } \bar{V}(t, \bfx^+_i,\bfx^-_j)  \text{ has less than 2 roots} \\
      +\infty, & \text{if } \bar{V}(t, \bfx^+_i,\bfx^-_j) \text{ has 2 roots $t_1 < t_2 \leq 0$} \\
      t_1 & \text{if } \bar{V}(t, \bfx^+_i,\bfx^-_j)  \text{ has 2 roots $0 \leq t_1 < t_2$} \\
      0 & \text{if }\bar{V}(t, \bfx^+_i,\bfx^-_j)  \text{ has 2 roots $t_1 \leq 0 \leq t_2$} 
\end{cases}, \nonumber
\end{eqnarray}
and $\bar{V}(t, \bfx^+_i,\bfx^-_j) =  \bar{a}t^2 + \bar{b}(\bfx_i^+, \bfx_j^-)t + \bar{c}(\bfx_i^+, \bfx_j^-)$ with
\small
\begin{eqnarray}
\bar{a} &:=& -n_1, \nonumber \\
\bar{b}(\bfx_i^+, \bfx_j^-)\!&:=&\!2\Vert {\bf\Gamma}(n_1\bfp_0 - (n_1 + n_2)\bfx^+_i + n_2\bfx^-_j)\Vert, \nonumber \\
\bar{c}(\bfx_i^+, \bfx_j^-)\!&:=&\!c(\bfx_i^+, \bfx_j^-).\nonumber
\end{eqnarray}
\normalsize
\end{restatable}

\begin{proof}
Please refer to Appendix \ref{appendix:rvm_prob_bound_ball}.
\end{proof}

Consider a general time-parameterized curve $\bfp(t)$, $t \in [0, t_f]$ from $\bfp_0 := \bfp(0)$ to $\bfp_f := \bfp(t_f)$. Prop.~\ref{prop:rvm_prob_bound_ball} shows that all points inside the ellipsoid $\mathcal{E}(\bfp_0, r)$ are free for $r = r_u \leq r^*_u$. If we can find the smallest positive $t_1$ such that
\begin{equation}
\label{eq:curve_safety}
\|\bfGamma(\bfp(t_1) - \bfp_0) \| = r,
\end{equation}
then all points on the curve $\bfp(t)$ for $t \in [0, t_1)$ are free. This is equivalent to finding the smallest positive solution of Eq. \eqref{eq:curve_safety}. We perform curve classification by iteratively covering the curve by free ellipsoids. If the value of $r$ is smaller than a threshold $\varepsilon$, the curve is considered colliding. Otherwise, it is considered free. The classification process for curves is shown in Alg.~\ref{alg:rvm_collision_checking_curve} and illustrated in Fig. \ref{fig:checking_curve_example_colliding} and \ref{fig:checking_curve_example_free} for the trained RVM model in Fig. \ref{fig:rvm_training_probit} for a colliding curve and a free curve, respectively. 


\begin{algorithm}[t]
\caption{RVM Curve Classification}
\label{alg:rvm_collision_checking_curve}
  \footnotesize
	\begin{algorithmic}
		\Require Curve $\bfp(t)$, $t \in [0,t_f]$; threshold $\varepsilon$; relevance vectors $\Lambda = \{(\bfx_i, y_i, \bfxi_i)\}$; weight posterior mean $\bfmu$ and max covariance eigenvalue $\lambda_{max}$
		\While{True}
			\State Calculate $r_{k}$ using Eq. \eqref{eq:rvm_ru} or Eq. \eqref{eq:rvm_ru_star}.
			\If{$r_k < \varepsilon$} \Return False (Colliding)
			\EndIf
			\State Solve $\Vert \bf\Gamma(\bfp(t) - \bfp(t_k)) \Vert = r_k$ for $t_{k+1} \geq t_{k}$
			\If{$t_{k+1} \geq t_f$}
				\Return True (Free)
			\EndIf
		\EndWhile
	\end{algorithmic}
\end{algorithm}

In Prop.~\ref{prop:rvm_prob_bound_lineseg} and \ref{prop:rvm_prob_bound_ball}, calculating $t_u$ and $r_u$ takes $O(M)$ time, while the computational complexity of calculating $t_u^*$ and $r_u^*$ are $O(M^2)$, where $M= M^+ + M^-$. If the line segments or curves are limited to the neighborhood of the starting point $\bfp_0$, the bound $t_u$ and $r_u$ can reasonably approximate $t^*_u$ and $r^*_u$, respectively, if $\bfx_j^-$ is chosen as the negative support vector, closest to $\bfp_0$. Calculation of $t_u$ and $r_u$ in Prop.~\ref{prop:rvm_prob_bound_lineseg} and \ref{prop:rvm_prob_bound_ball} is efficient in the sense that it has the same complexity as classifying a point, yet it can classify an entire line segment for $t \in [0, t_u)$ and an entire ellipsoid $\mathcal{E}(\bfp_0, r_u)$, respectively.


\section{Computational and Storage Improvements}
\label{sec:complexities}
\subsection{Computational Improvements}
\label{subsec:comp_improvements}

In the context of autonomous navigation, as a robot explores new regions of its environment, the number of relevance vectors required to represent the obstacle boundaries increases. Since the score function in Eq.~\eqref{eq:rvm_score_fcn} depends on all relevance vectors, the training time (Alg.~\ref{alg:rvm_training_online}) and the classification time (Def.~\ref{def:threshold_classification} for points, Alg.~\ref{alg:rvm_collision_checking_line} for lines, and Alg.~\ref{alg:rvm_collision_checking_curve} for curves) increase as well. We propose an approximation to the score function $F(\bfx)$ for the radial basis kernel in Assumption \ref{assumption:rbf}. Since $k(\bfx,\bfx_m)$ approaches zero rapidly as the distance between $\bfx$ and $\bfx_m$ increases, the value of $F(\bfx)$ is not affected significantly by relevance vectors far from $\bfx$. We use $R^*$-tree data structures constructed from the relevance vectors $\Lambda^+$, $\Lambda^-$ to allow efficient lookup of the nearest $K^+$ and $K^-$ positive and negative relevance vectors. Approximating the score function $F(\bfx)$ using the nearest $K^+$ and $K^-$ relevance vectors improves its computational complexity from $O(M)$ to $O(\log M)$. Similarly, to classify a point $\bfx$, the $M$-dimensional feature vector $\bfPhi_x$, may be approximated by a $K$-dimensional one using the $K$ relevance vectors closest to $\bfx$. Classification of a line segment or a curve in Prop.~\ref{prop:rvm_prob_bound_lineseg} and \ref{prop:rvm_prob_bound_ball} can be approximated by using the $K^+$ and $K^-$ nearest positive and negative relevance vectors. The computational complexities of Eq.~\eqref{eq:rvm_tu}, \eqref{eq:rvm_tu_star}, \eqref{eq:rvm_ru}, and \eqref{eq:rvm_ru_star} improve from $O(M)$ and $O(M^2)$ to $O(\log M)$.



The line and curve classification algorithms depend on Prop.~\ref{prop:rvm_prob_bound} which requires the largest eigenvalue $\lambda_{max}$ of the weight posterior covariance matrix $\bfSigma$. Obtaining $\lambda_{max}$ from $\bfSigma$ can be expensive as the number of relevance vectors grows. Under Assumption \ref{assumption:rbf}, the entries in the feature matrix $\bfPhi$ for relevance vectors that are far from each other go to $0$ quickly and can be set to zero, e.g., using a cut-off threshold for the kernel values or only keeping the kernel values for the $K$ nearest relevance vectors. This leads to a sparse matrix $\bfPhi$ and, in turn, the inverse covariance matrix $\bfSigma^{-1}$ in Eq. \eqref{eq:approx_sigma} is sparse and its smallest eigenvalue $1/\lambda_{max}$ can be approximated efficiently (e.g.~\cite{lehoucq1998arpack}).

\subsection{Storage Improvements}
\label{subsec:storage_req}

Alg. \ref{alg:rvm_training_online} returns a set of $M$ relevance vectors $\bsym{x_m}$ with labels $y_m$ and weight prior precision $\xi_m$. This set represents the RVM model parameters and its memory requirements are linear in $M$. However, the predictive distribution in Eq.~\eqref{eq:rvm_probit_score} needs to be obtained via Laplace approximation (Eq.~\eqref{eq:approx_mu} and \eqref{eq:approx_sigma}) when the RVM model is used for classification. If additional computation for Laplace approximation is not feasible during test time, the weight posterior mean $\bfmu$ and covariance $\bfSigma$ may be stored also but $\bfSigma$ requires $O(M^2)$ storage. Fortunately, the approximate decision boundary $G_2(\bfx) = 0$ in Prop.~\ref{prop:rvm_prob_bound} used for point, line, and curve classification only requires the largest eigenvalue $\lambda_{max}$ of $\bfSigma$. Hence, only the value of $\lambda_{max}$ needs to be stored in addition to the relevance vectors $\bsym{x}_m$, labels $y_m$, and weight mean $\mu_m$. In this case, line 15 in Alg.~\ref{alg:rvm_training_online} should return the weight posterior mean $\bfmu$ and $\lambda_{max}$ instead of $\bfmu$ and $\bfSigma$.


\section{Application to Occupancy Mapping and Autonomous Navigation}

\subsection{Online Mapping}
\label{subsec:online_mapping}

We consider a robot placed in an unknown environment at time $t_k$ as illustrated in Fig.~\ref{fig:robot_unkenv_scan}. It is equipped with a lidar scanner measuring distances to nearby obstacles. Samples generated from the lidar range scan $\bfz_k$ are shown in Fig.~\ref{fig:laser_uninflated_ws}. Since the robot body is bounded by a sphere of radius $r$, each laser ray end point in configuration space becomes a ball-shaped obstacle, while the robot body becomes a point. To generate local training data, the occupied and free C-space areas observed by the lidar are sampled (e.g., on a regular grid). As shown in Fig.~\ref{fig:laser_inflated_cs}, this generates a set $\bar{\mathcal{D}}_k$ of points with label ``1'' (occupied) in the ball-shaped occupied areas and with label ``-1'' (free) between the robot position and each laser end point. To accelerate training, only the difference between two consecutive local datasets $\mathcal{D}_k=\bar{\mathcal{D}}_k\setminus\bar{\mathcal{D}}_{k-1}$ is used in our online RVM training algorithm (Alg.~\ref{alg:rvm_training_online}). Storing the sets of relevance vectors $\Lambda_k$ over time requires significantly less memory than storing the training data $\cup_k \mathcal{D}_k$. The occupancy of a query point $\bfx$ can be estimated from the relevance vectors by evaluating the function $G_1(\bfx)$ in Eq.~\eqref{eq:rvm_classification_cond3}. Specifically, $\hat{m}_k(\bfx) = -1$ if $G_1(\bfx) \leq 0$ and $\hat{m}_k(\bfx) = 1$ if $G_1(\bfx) > 0$. Fig.~\ref{fig:trained_model} illustrates the boundaries generated by Alg.~\ref{alg:rvm_training_online}.

\subsection{Autonomous Navigation}
\label{subsec:auto_nav}


Finally, we present a complete online mapping and navigation approach that solves Problem~\ref{problem_formulation_unknown_env}. Given the sparse Bayesian kernel-based map $\hat{m}_{k}$ proposed in Sec.~\ref{subsec:online_mapping}, a motion planning algorithm such as $A^*$~\cite{Russell_AI_Modern_Approach} or $RRT^*$ \cite{karaman2010incremental} may be used with our collision-checking algorithms to generate a path that solves the autonomous navigation problem (Alg.~\ref{alg:get_successors}). The robot follows the path for some time and updates the map estimate $\hat{m}_{k+1}$ with new observations. Using the updated map, the robot re-plans the path and follows the new path instead. This process is repeated until the goal is reached or a time limit is exceeded (Alg.~\ref{alg:auto_nav_rvm}).

The use of the ``inflated boundary" $G_3(\bfx) = 0$ from Prop.~\ref{prop:rvm_prob_bound_amgm} for collision checking might block the motion planning task if it is not tight enough in certain regions of the environment (e.g., unobserved regions as discussed in Sec.~\ref{subsec:point_classification}). For such regions, a different ratio of $n_2/n_1$ can be used in Prop.~\ref{prop:rvm_prob_bound_amgm} to achieve a tighter bound $G_3(\bfx)$. Increasing the decision threshold $\bar{e}$ (Def.~\ref{def:threshold_classification}) can also improve the accuracy of $G_3(\bfx)$ if a trade-off with robot safety is allowed. Another resort is to use sampling-based collision checking, selecting points along the curve $\bfp(t)$ and using Def.~\ref{def:threshold_classification}.



\begin{algorithm}[t]
\caption{\textsc{GetSuccessors} and \textsc{ObstacleFree} subroutines in $A^*$\cite{Russell_AI_Modern_Approach} and $RRT^*$ \cite{karaman2010incremental}, respectively}
\label{alg:get_successors}
  \footnotesize
	\begin{algorithmic}
		\Require Current position $\bfp_{k}$; set of relevance vectors $\Lambda = \{(\bfx_i, y_i, \bfxi_i)\}$ with posterior weight mean $\bfmu$ and covariance $\bfSigma$; set $\mathcal{N}(\bfp_{k})$ of potential reference trajectories $\bfp(t-t_k)$ with $\bfp(t_k) = \bfp_k$.
		\Ensure Set of collision-free trajectories $S$.
		\State $S \gets \emptyset$;.
		\For{$\bfp'$ in $\mathcal{N}(\bfp_{k})$}
			\If{$\bfp'$ is a line \textbf{and} \Call{CheckLine}{$\bfp_{k}, \bfp',\Lambda$}} \Comment{Alg.~\ref{alg:rvm_collision_checking_line}}
			  \State $S \leftarrow S \cup \{\bfp'\}$
			\EndIf
			\If{$\bfp'$ is a curve \textbf{and} \Call{CheckCurve}{$\bfp_{k}, \bfp',\Lambda$}} \Comment{Alg.~\ref{alg:rvm_collision_checking_curve}}
				\State $S \leftarrow S \cup \{\bfp'\}$
			\EndIf
		\EndFor
		\State \Return $S$
	\end{algorithmic}
\end{algorithm}

We consider robots with two different motion models. In simulation, we use a first-order fully actuated robot, $\dot{\bfp} = \bfv$, where the state $\bfs$ is the robot position $\bfp \in [0,1]^3$, with piecewise-constant velocity $\bfv(t) \equiv \bfv_k \in \calV$ for $t \in [t_k,\;t_{k+1})$, leading to piecewise-linear trajectories:
\begin{equation}
\label{eq:fas_trajectory}
\bfp(t) = \bfp_k + (t-t_k)\bfv_k, \qquad t \in [t_k,\;t_{k+1}),
\end{equation}
where $\bfp_k := \bfp(t_k)$. In this case, the classification algorithm for line segments (Alg.~\ref{alg:rvm_collision_checking_line}) is used during motion planning.

In the real experiments, we consider a ground wheeled Ackermann-drive robot with dynamics model:
\begin{equation}
\label{eq:car_dynamics}
\dot{\bfp} = v \begin{bmatrix} \cos(\theta)\\\sin(\theta) \end{bmatrix},\qquad \dot{\theta} = \frac{v}{\ell}\tan \phi,
\end{equation}
where the state $\bfs$ consists of the position $\bfp \in \mathbb{R}^2$ and orientation $\theta \in \mathbb{R}$, the control input \NEW{$\bfu$} consists of the linear velocity $v \in \mathbb{R}$ and the steering angle $\phi \in \mathbb{R}$, and $\ell$ is the distance between the front and back wheels. The nonlinear car dynamics can be transformed into a 2nd-order fully actuated system $\ddot{\bfp} = \bfa$ via feedback linearization~\cite{de2000stabilization, franch2009control}. 
Using piecewise-constant acceleration $\bfa(t) \equiv \bfa_k \in \calA$ for $t \in [t_k,\;t_{k+1})$ leads to piecewise-polynomial trajectories:
\begin{equation}
\label{eq:car_trajectory}
\bfp(t) = \bfp_k + (t-t_k) v_k\begin{bmatrix} \cos(\theta_k)\\\sin(\theta_k) \end{bmatrix} + \frac{(t-t_k)^2}{2}\bfa_k,
\end{equation}
where $\bfp_k := \bfp(t_k)$, $\theta_k := \theta(t_k)$, $v_k := v(t_k)$. In our experiments, the input set $\calA$ is finite and the classification algorithm for curves (Alg.~\ref{alg:rvm_collision_checking_curve}) is used to get successor nodes in an $A^*$ motion planning algorithm.


\begin{algorithm}[t]
\caption{Autonomous Mapping and Navigation with a Sparse Bayesian Kernel-based Map}
\label{alg:auto_nav_rvm}
  \footnotesize
	\begin{algorithmic}[1]
		\Require Initial state $\bfs_0 \in \bar{\mathcal{S}}$; goal region $\mathcal{G}$; prior relevance vectors $\Lambda_0$.
		\For{$k = 0, 1, \ldots$}
		  \If{$\bfs_k \in \calG$} \textbf{break} \EndIf
			\State $\bfz_k\gets$ new range sensor observation
			\State $\bf\mathcal{D}_k \leftarrow $ Training Data Generation$(\bfz_k, \bfs_k)$ \Comment{Sec. \ref{subsec:online_mapping}}
			\State $\Lambda_{k+1} \leftarrow $ Online RVM Training$(\Lambda_k, \mathcal{D}_k)$ \Comment{Alg.~\ref{alg:rvm_training_online}}
			\State Path Planning$(\Lambda_{k+1}, \bfs_k, \mathcal{G})$ \Comment{Alg.~\ref{alg:get_successors}}
			\State Move to the first state $\bfs_{k+1}$ along the path 
		\EndFor
	\end{algorithmic}
\end{algorithm}


\section{Experimental Results}
\label{sec:experimental_results}

\begin{figure*}[t]
\centering
\begin{subfigure}[b]{0.327\textwidth}
        \centering
        \includegraphics[width=\textwidth]{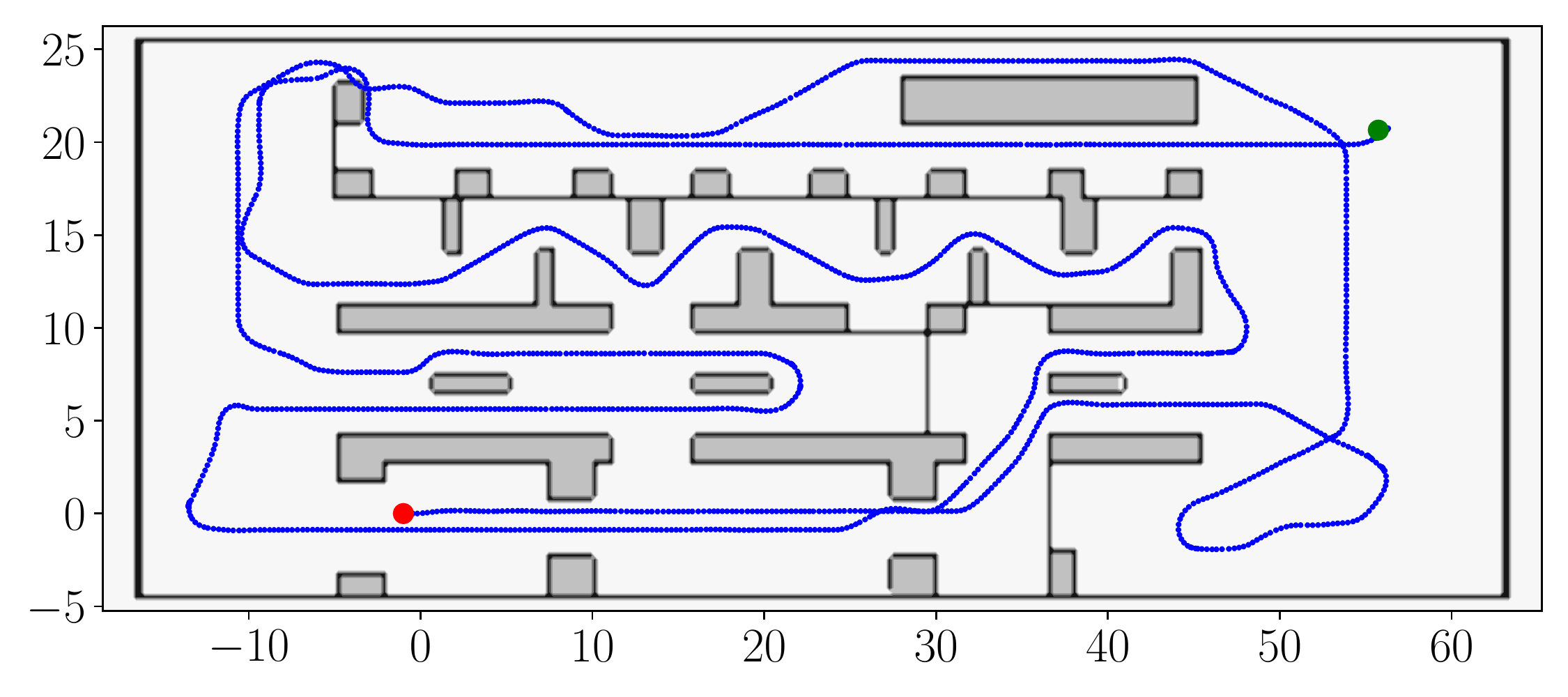}
        \caption{Ground truth map and robot trajectory.}
        \label{fig:gt_sim_map}
\end{subfigure}%
\hfill
\begin{subfigure}[b]{0.327\textwidth}
        \centering
        \includegraphics[width=\textwidth]{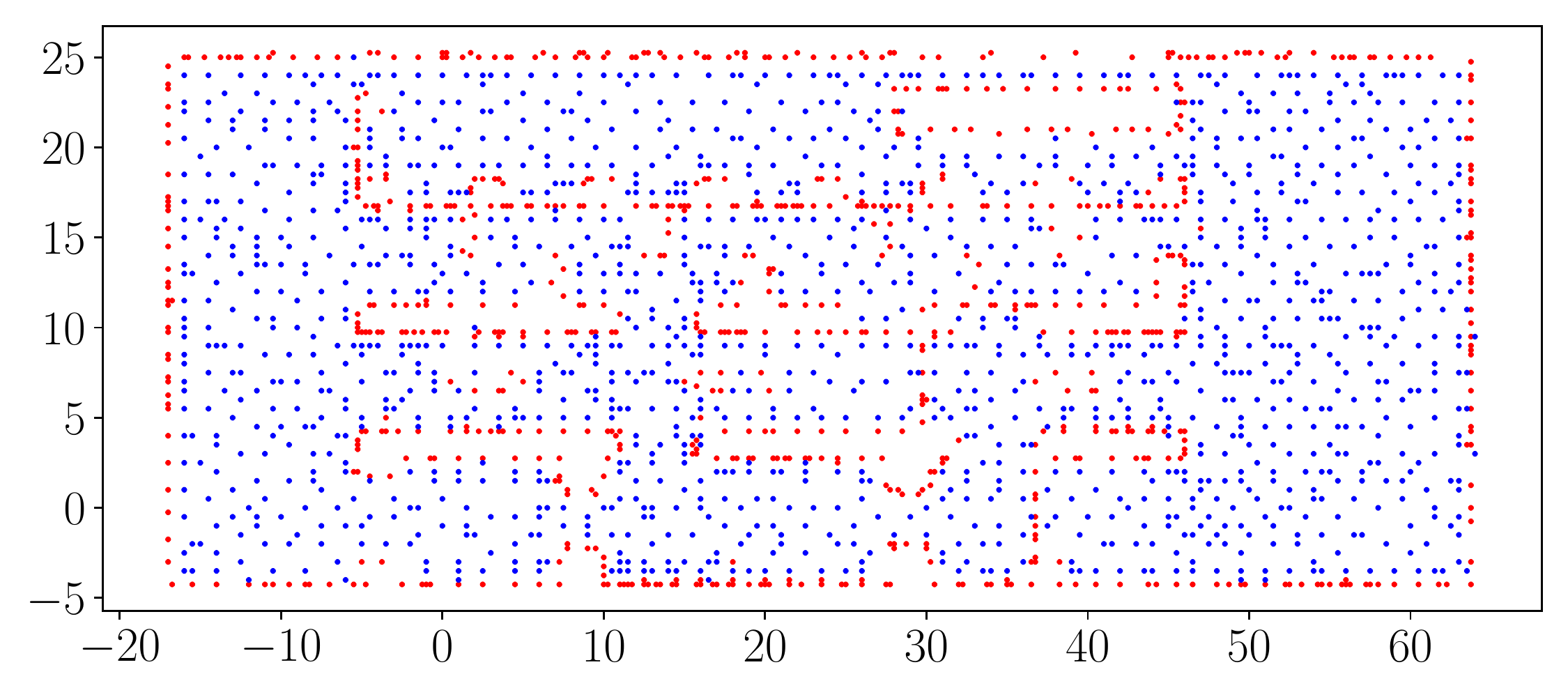}
        \caption{Set of $2141$ relevance vectors.}
        \label{fig:rvs_sim_map}
\end{subfigure}%
\hfill
\begin{subfigure}[b]{0.346\textwidth}
        \centering
        \includegraphics[trim=0 0 100 0,clip,width=\textwidth]{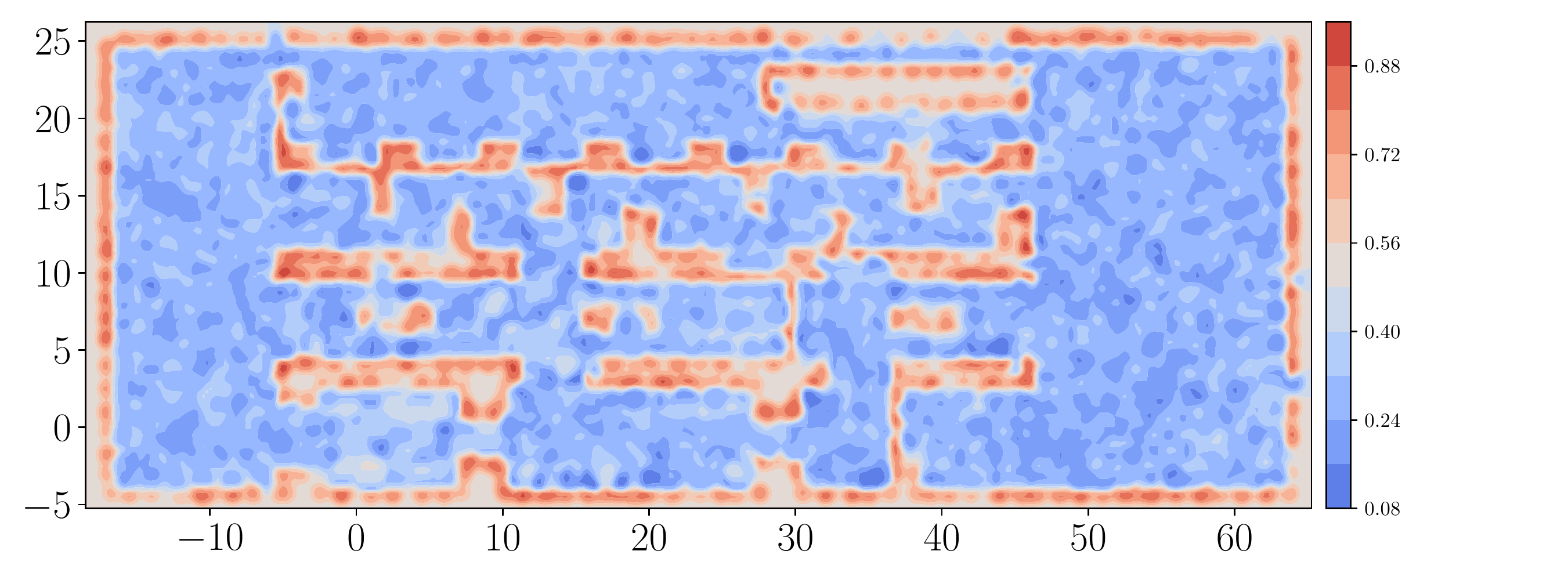}
        \caption{The final SBKM map.}
        \label{fig:sim_rvmmap}
\end{subfigure}%
%


\caption{Sparse map representation (with $\eta = 1, \Gamma = \sqrt{\gamma}I, \gamma = 3.0$) built from local streaming laser scans along the robot trajectory.}
\label{fig:sim_map}
\end{figure*}

\begin{table*}[t]
\caption{Comparison among our sparse Bayesian kernel-based map (SBKM), our sparse kernel-based map (SKM) \cite{duong2020autonomous}, OctoMap (OM)~\cite{octomap}, \NEW{and sequential Bayesian Hilbert map (SBHM)~\cite{senanayake2017bayesian}}. An RBF kernel with $\eta = 1, \Gamma = \sqrt{\gamma}I$ was used for SBHM map and our SBKM and SKM maps. $^\dagger$The storage requirements for SBKM are calculated for two storing approaches mentioned in Sec. \ref{sec:complexities}: 1) with Laplace approximation at test time, i.e., storing the relevance vectors' location with their label and precision $\bfxi$; 2) without Laplace approximation at test time, i.e., storing the relevance vectors' location with their label and weight's mean $\bf\mu$, and the largest eigenvalue of the covariance matrix $\lambda_{max}$. As we only need an extra float to store $\lambda_{max}$, both approaches offer similar storage requirements. \NEW{$^\ddag$ The storage requirements for SBHM are calculated as if the hinge points are stored using our storing approach.}}
\label{table:accuracy_column}
\centering

\begin{tabular}{ ccccccc } 
Methods  & Kernel param. $\gamma$ & Threshold $\bar{e}$ & Accuracy & Recall & Vectors/Nodes & Storage\\
\hline
\hline
SBKM & 1.0 & 0.5 & 97.8\% & 97.9\% & 1115 & 9kB$^\dagger$ \\
SBKM & 2.0 & 0.5 & 99.0\% & 99.3\% & 1642 & 13kB$^\dagger$ \\
SBKM & 3.0 & 0.45 & 99.2\% & 99.7\% & 2141 & 17kB$^\dagger$ \\
SBKM & 3.0 & 0.5 & 99.3\% & 99.4\% & 2141 & 17kB$^\dagger$ \\
SBKM & 3.0 & 0.55 & 99.5\% & 98.7\% & 2141 & 17kB$^\dagger$ \\
\hline
SKM & 3.0 & - & 99.9\% & 99.0\% & 2463 & 20kB \\

\NEW{SKM} & \NEW{2.0} & \NEW{-} & \NEW{99.8\%} & \NEW{98.3\%} & \NEW{2613} & \NEW{21kB} \\
\NEW{SKM} & \NEW{1.0} & \NEW{-} & \NEW{99.8\%} & \NEW{98.5\%} & \NEW{3064} & \NEW{25kB} \\
\hline
OM & - & 0.5 & 99.9\% & 99.7\% & 12432 non-leafs \& 34756 leafs & 25kB(binary)/236kB (full) \\
\hline
\NEW{SBHM} & \NEW{1.0} & \NEW{0.5} & \NEW{97.0\%} & \NEW{98.0\%} & \NEW{1156} & \NEW{9kB$^\ddag$} \\
\NEW{SBHM} & \NEW{2.0} & \NEW{0.5} & \NEW{99.0\%} & \NEW{99.4\%} & \NEW{1676} & \NEW{13kB$^\ddag$} \\
\NEW{SBHM} & \NEW{3.0} & \NEW{0.45} & \NEW{98.6\%} & \NEW{99.0\%} & \NEW{2205} & \NEW{17kB$^\ddag$} \\
\NEW{SBHM} & \NEW{3.0} & \NEW{0.5} & \NEW{99.0\%} & \NEW{98.6\%} & \NEW{2205} & \NEW{17kB$^\ddag$} \\
\NEW{SBHM} & \NEW{3.0} & \NEW{0.55} & \NEW{99.5\%} & \NEW{98.2\%} & \NEW{2205} & \NEW{17kB$^\ddag$}
\end{tabular}
\end{table*}


This section presents an evaluation of our autonomous mapping and navigation method using a fully actuated robot~\eqref{eq:fas_trajectory} in a simulated environment (Sec.~\ref{subsec:compare_simulated_warehouse}), the Intel Research Lab dataset \cite{Radish} (Sec. \ref{subsec:compare_intel}), and a car-like robot (Fig.~\ref{fig:robot_unkenv_scan}) with Ackermann-drive dynamics~\eqref{eq:car_dynamics} in real experiments (Sec. \ref{subsec:real_experiments}). \NEW{We examined the obstacle boundary with respect to the bias parameter $b$ and the threshold $e$ in Sec. \ref{subsec:boundary_conservativeness} and carried out an active mapping experiment using our map uncertainty in Sec. \ref{subsec:active_mapping}}. We used a radial basis function (RBF) kernel with parameters $\eta = 1$ and $\Gamma = \sqrt{\gamma} \bfI$. \NEW{The bias parameter $b$ is set to $-0.05$ in Sec. \ref{subsec:compare_simulated_warehouse}, Sec. \ref{subsec:real_experiments} and Sec. \ref{subsec:active_mapping}, and $0.0$ in Sec. \ref{subsec:compare_intel}}. Timing results are reported from an Intel i9 3.1 GHz CPU with 32GB RAM.


\begin{figure}[t]
\centering
\begin{subfigure}[b]{0.225\textwidth}
        \centering
        \includegraphics[width=\textwidth]{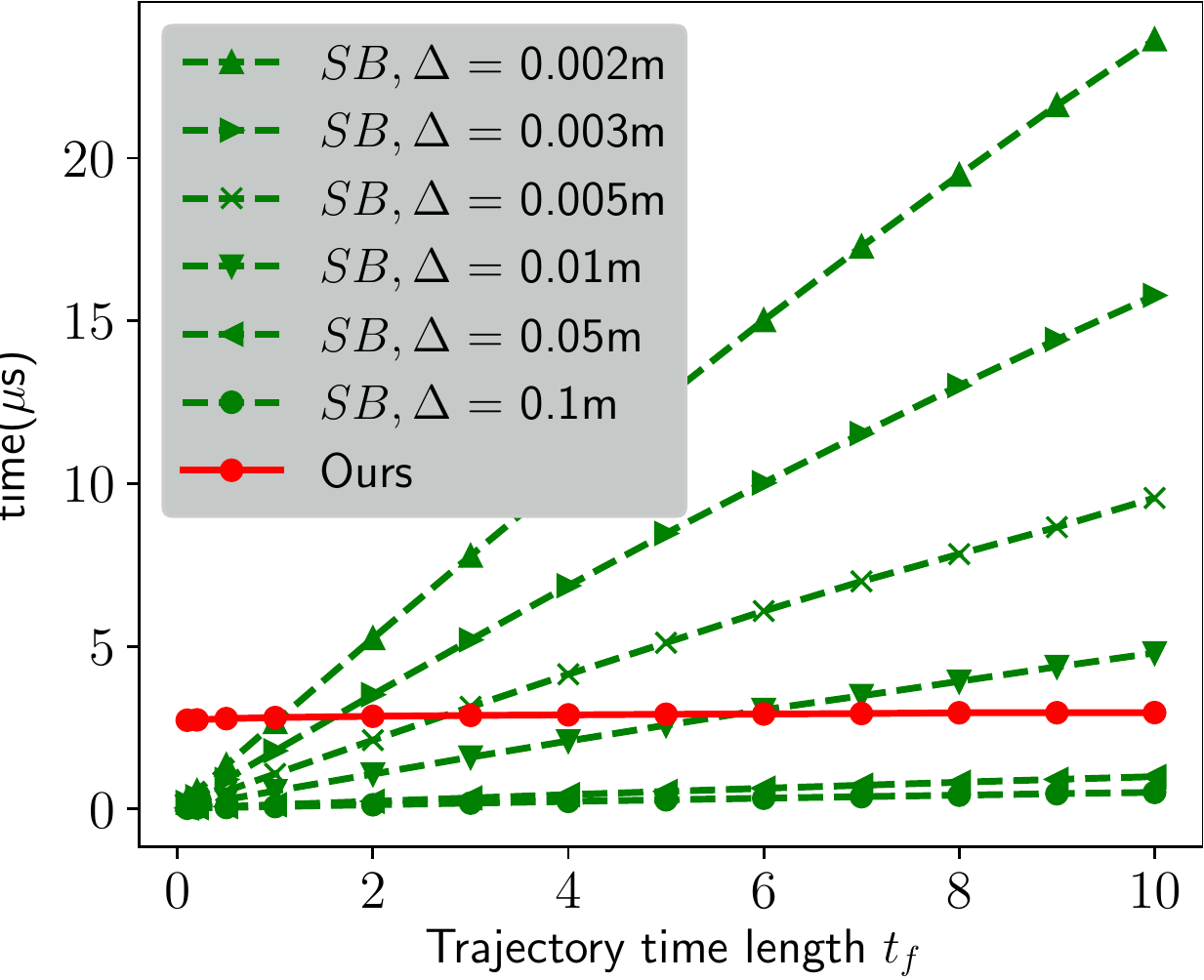}
        \caption{}
        \label{fig:checking_line_comparison}
\end{subfigure}%
\hfill
\begin{subfigure}[b]{0.225\textwidth}
        \centering
        \includegraphics[width=\textwidth]{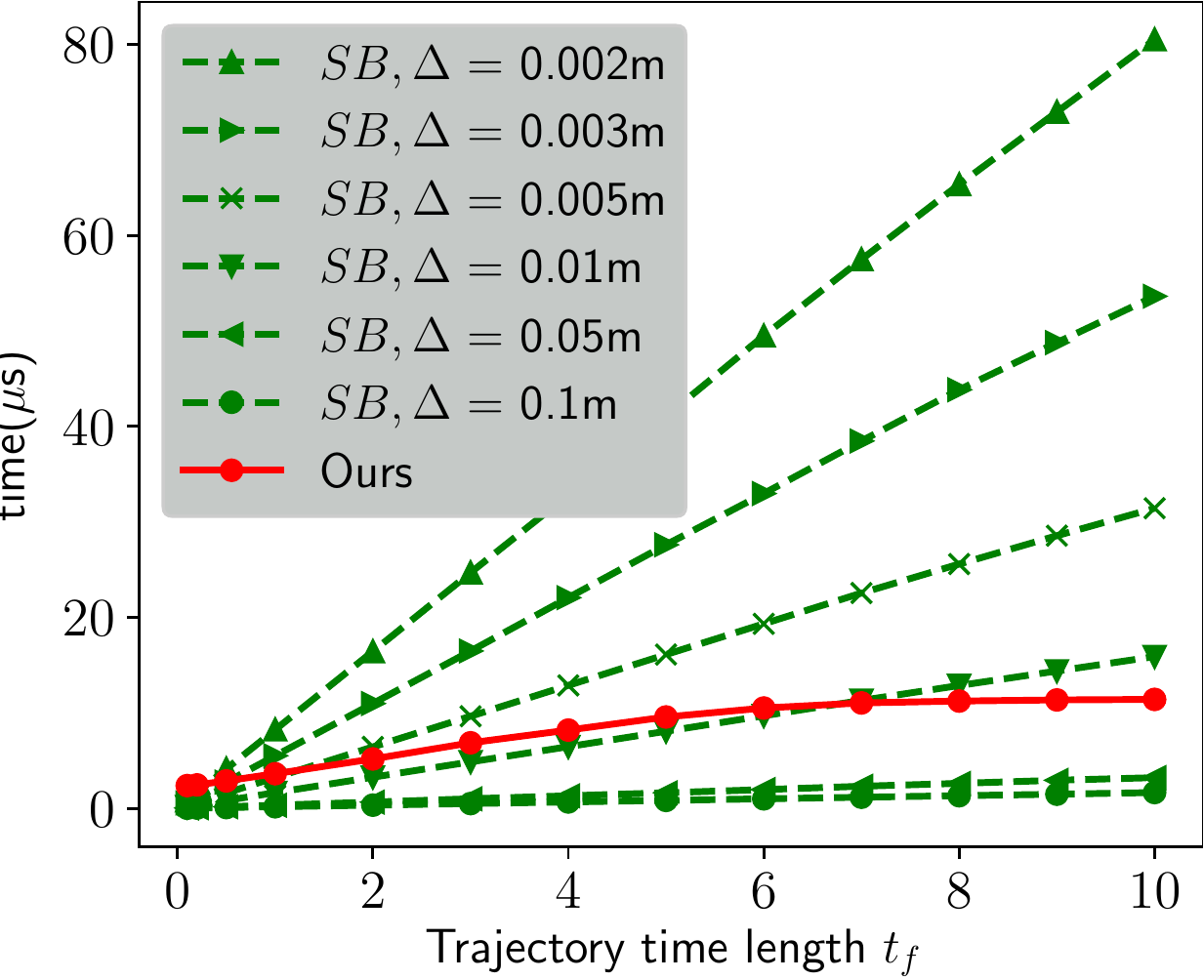}
        \caption{}
        \label{fig:checking_curve_comparison}
\end{subfigure}%
%


\caption{Collision checking time comparison between our methods and sampling-based (SB) ones with different sampling interval $\Delta$ for (a) line segments $\bfp(t) = \bfp_0 + \bfv t$ and (b) $2^{nd}$-order polynomial curves $\bfp(t) = \bfp_0 + \bfv t + \bfa t^2$ for $t \in [0,t_f]$  with various values of $t_f$.}

\label{fig:collision_checking_comparison}
\end{figure}

\begin{figure*}[t]
\centering
\begin{subfigure}[b]{0.245\textwidth}
        \centering
        \includegraphics[width=\textwidth]{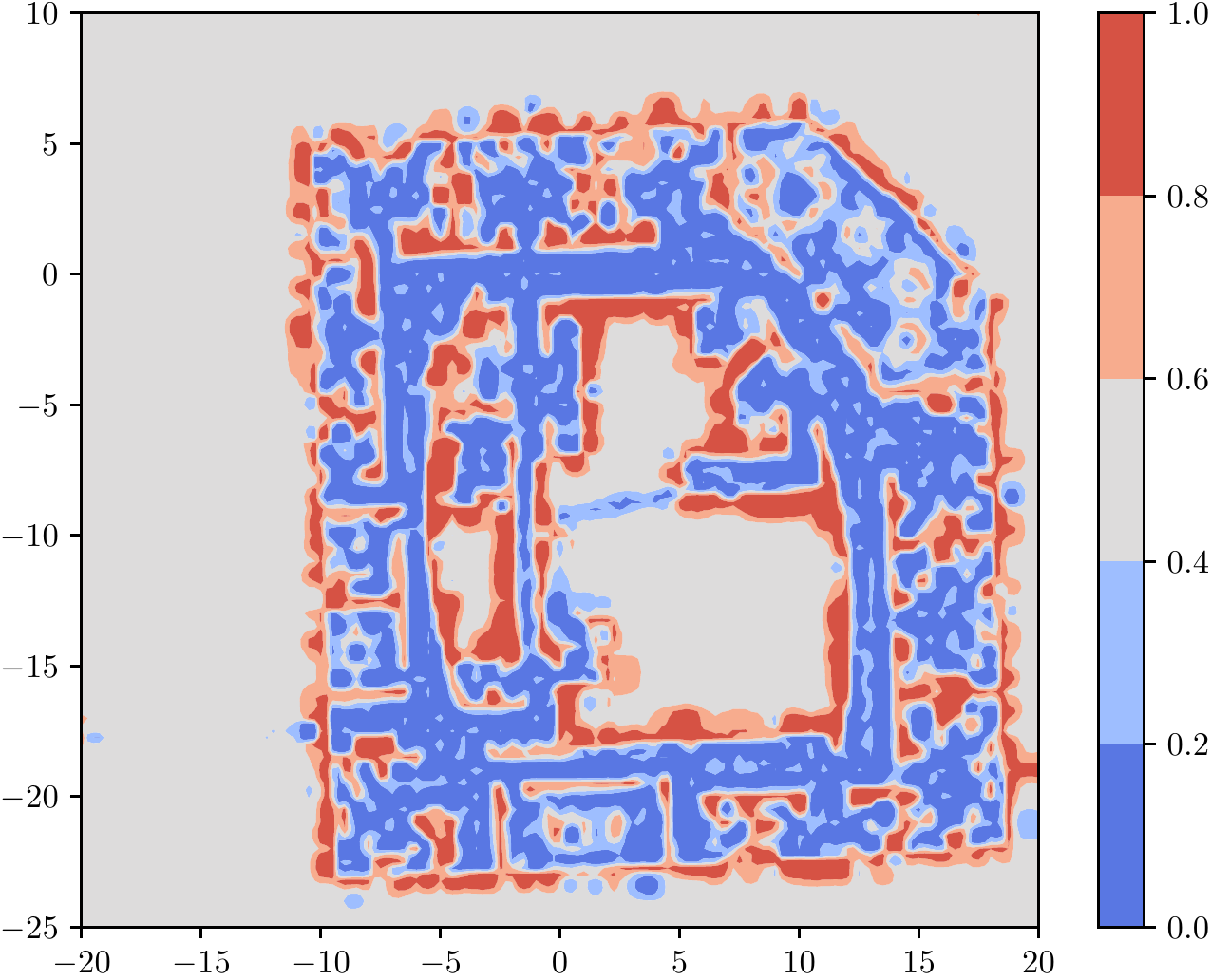}
        \caption{SBHM map.}
        \label{fig:sbhm_intel}
\end{subfigure}
\hfill
\begin{subfigure}[b]{0.245\textwidth}
        \centering
        \includegraphics[width=\textwidth]{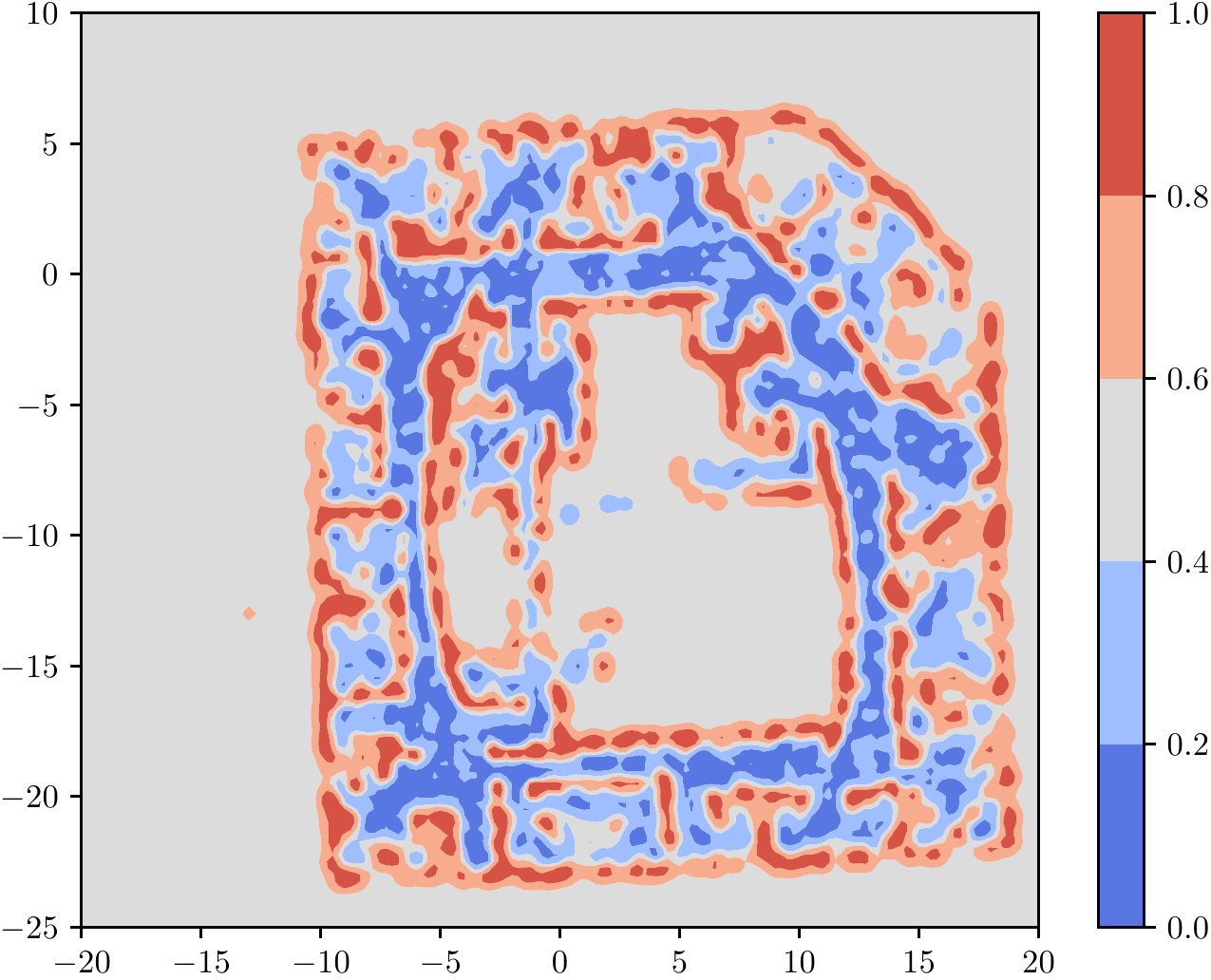}
        \caption{Our SBKM map.}
        \label{fig:rvmmap_intel}
\end{subfigure}%
\hfill
\begin{subfigure}[b]{0.245\textwidth}
        \centering
        \includegraphics[width=\textwidth]{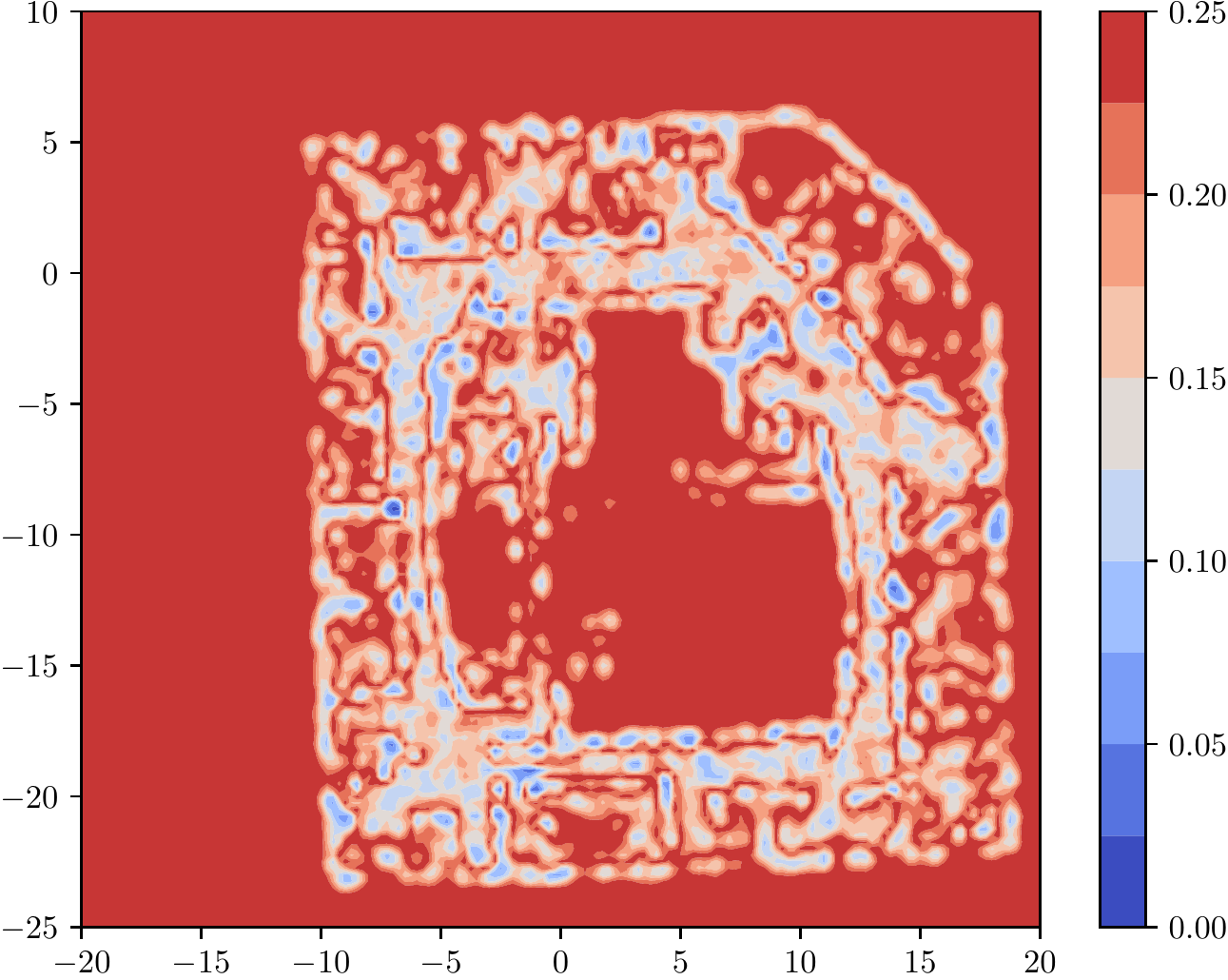}
        \caption{\NEW{Our SBKM map's variance.}}
        \label{fig:rvmmap_intel_variance}
\end{subfigure}%
\hfill
\begin{subfigure}[b]{0.25\textwidth}
        \centering
        \includegraphics[width=\textwidth]{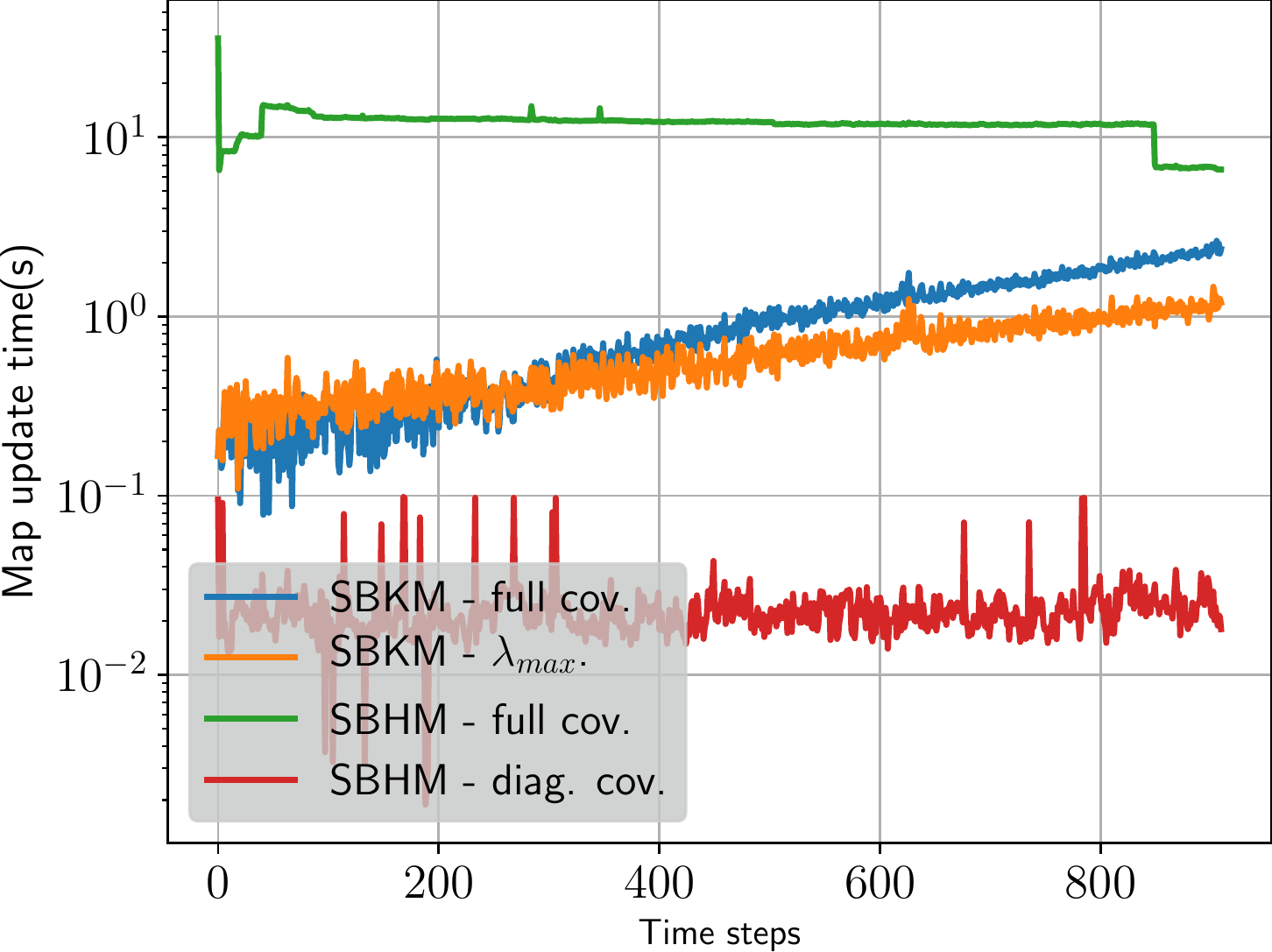}
        \caption{\NEW{Map update time per step.}}
        \label{fig:update_time_intel}
\end{subfigure}%
\caption{Comparison between sparse Bayesian kernel-based map (SBKM) and Sequential Bayesian Hilbert Map (SBHM) \cite{senanayake2017bayesian} with $\eta = 1$, $\Gamma = \sqrt{6.71}I$ built online using lidar scans from the Intel Research Lab dataset~\cite{Radish}.}
\label{fig:intel_map}
\end{figure*}

\subsection{\NEW{Comparison with binary map representation}}
\label{subsec:compare_simulated_warehouse}

In this section, we compared the accuracy, the recall and storage requirements of our sparse Bayesian kernel-based map (SBKM) with those of the non-Bayesian sparse kernel-based map (SKM) from our preliminary work~\cite{duong2020autonomous}, the popular occupancy mapping algorithm OctoMap~\cite{octomap}, \NEW{and the sequential Bayesian Hilbert map (SBHM)~\cite{senanayake2017bayesian}. Since SKM only provides binary maps, binary maps are used to calculate accuracy and recall.} As the ground-truth map (Fig.~\ref{fig:gt_sim_map}) represents the work space instead of C-space, a point robot ($r=0$) was used for an accurate comparison. Lidar scans were simulated along the robot trajectory shown in Fig.~\ref{fig:gt_sim_map} and used to build our sparse Bayesian kernel-based map (SBKM), the non-Bayesian sparse kernel-based map (SKM) \cite{duong2020autonomous}, OctoMap, \NEW{and sequential Bayesian Hilbert map (SBHM)}. An $R^*$-tree approximation of the score $F(\bfx)$ was used with $K^++K^- = 200$ nearest support vectors around the robot location $\bfp_k$ for map updating and with $K^+ + K^- = 10$ nearest support vectors for collision checking. OctoMap's resolution was set to $0.25m$ to match that of the grid used to sample our training data from. \NEW{As SBHM depends on a grid of hinge points to generate feature vectors, we chose the grid's resolution so that the number of hinge points is similar to our SBKM's number of relevance vectors for a fair comparison.}


Table~\ref{table:accuracy_column} compares the accuracy, the recall and the storage requirements of our SBKM maps, our SKM maps, \NEW{SBHM maps \cite{senanayake2017bayesian}} and OctoMap's binary and probabilistic maps. The SBKM map and its sparse set of relevance vectors are shown in Fig.~\ref{fig:sim_map}. To calculate map accuracy, we used different thresholds $\bar{e}$ to generate binary versions of our map and compare with the ground truth. The ground truth map was sampled on a grid with the same resolution $0.25m$ and the accuracy was calculated as the number of correct predictions divided by the total number of samples. Note that the interior (gray regions in the ground-truth map) of the obstacles were considered occupied for our map - since it was surrounded by positive relevance vectors - but were considered free in SKM and OctoMap maps.

Table~\ref{table:accuracy_column} shows that SBKM (with threshold $\bar{e} = 0.5$), SKM, OctoMap's binary map, \NEW{and SBHM (with threshold $\bar{e} = 0.5$) led to a similar accuracy of $\sim99\%$ ($\gamma = 2.0 \& 3.0$) and $\sim 97\%$ ($\gamma = 1.0$) and a similar recall of $\sim 99\%$ ($\gamma = 2.0 \& 3.0$) and $\sim 98\%$ ($\gamma = 1.0$). SKM has $\sim 0.5\%$ higher accuracy than SBKM since the support vectors lie around the obstacle boundary, leading to sharper decision boundaries.} As we decreased the decision threshold $\bar{e}$, the accuracy decreased, as more free cells were classified as ``occupied", while the recall increased, as more occupied cells were classified as ``occupied", and vice versa. When the parameter $\gamma$ decreased, the support of the kernel expanded, leading to fewer relevance vectors, i.e., less storage but lower accuracy and recalls. This illustrates the trade-off between storage gains and accuracy when the details of the obstacles' boundaries can be reduced via a lower value of $\gamma$ to achieve higher compression rate.

We compared the storage requirements for our SBKM and SKM representations and OctoMap. OctoMap's binary map required a compressed octree with $12432$ non-leaf nodes with $2$ bytes per node, leading to a storage requirement of $\sim25 kB$. Its fully probabilistic map required to store $47188$ leaf and non-leaf nodes with $5$ bytes per node, leading to a storage requirement of $\sim236kB$. As the space consumption depends on the computer architecture and how the relevance vector information is compressed, we provide only a rough estimate of storage requirements for our maps. For the SKM map, each support vector required $8$ bytes, including an integer for the support vector's location on the underlying grid and a float for its weight. As a result, $\sim20kB$ were needed to store the $2463$ resulting support vectors for $\gamma = 3.0$. As discussed in Sec. \ref{subsec:storage_req}, the SBKM map could be stored in two ways: 1) the relevance vectors' location, their label and their weight prior precision if Laplace approximation was allowed at test time; 2) the relevance vectors' location, their label, their weight mean and the largest eigenvalue $\lambda_{max}$ of the covariance matrix $\Sigma$ if Laplace approximation was not allowed at test time and our collision checking methods were used. The former stored an integer representing a relevance vector's location on the underlying grid and a float representing its weight prior's precision and its label (using the float sign). This required $8$ bytes on a $32$-bit architecture per relevance vector. Our SBKM map with $\Gamma = \sqrt{3.0}I$ contained $2141$ relevance vectors, leading to storage requirements of $\sim17kB$. The latter also needed $17kB$ to store the relevance vectors' location, their weight's mean and label. Besides, an extra float ($4$ bytes) is needed to store $\lambda_{max}$ leading to a similar total storage requirement of $17kB$. \NEW{These requirements were $32\%$ and $15\%$ better than those of OctoMap and our (non-Bayesian) SKM, respectively. To achieve a sparse Bayesian map representation, more computation is needed, leading to slower map update for SBKM compared to SKM and OctoMap. Since SBHM and SBKM are both Bayesian online mapping methods and share similar settings, we compared their map update time in Sec. \ref{subsec:compare_intel}. As $\gamma$ decreases, the number of relevance vectors of SBKM decreased while the number of support vectors of SKM increased. This is because the relevance vectors spread out in the environment while the support vectors were placed on both sides of the obstacle boundaries. Therefore, as $\gamma$ decreased, a relevance/support vector represented more space leading to fewer relevance vectors in the SBKM models and more support vectors, maintaining a sharp decision boundary, for the SKM models.}

We also compared the average collision checking time over one million random line segments $\bfp(t) = \bfp_0 + \bfv t$ and one million random second order polynomial curves $\bfp(t) = \bfp_0 + \bfv t + \bfa t^2$ for $t \in [0,t_f]$ using our complete method (Alg. \ref{alg:rvm_collision_checking_line} with Eq. \eqref{eq:rvm_tu_star} for line segments, Alg. \ref{alg:rvm_collision_checking_curve} with Eq. \eqref{eq:rvm_ru_star} for curves,  $K^+ + K^- = 10$ for score approximation, \NEW{and $e = -0.01$ for occupancy threshold}) and sampling-based methods with different sampling resolutions using the ground truth map. Fig.~\ref{fig:checking_line_comparison} and \ref{fig:checking_curve_comparison} show that the time for sampling-based collision checking increased as the time length $t_f$ increased or the sampling resolution decreased. Meanwhile, our method's time was stable at $\sim 3\mu s$ for checking line segments and at $\sim 11 \mu s$ for checking second-order polynomial curves suggesting our collision checking algorithms' suitability for real-time applications. 

%
%

\subsection{\NEW{Comparision with probabilistic map representations}}
\label{subsec:compare_intel}

%

\begin{table}[t]
\caption{Comparison between our sparse Bayesian kernel-based map (SBKM),  Sequential Bayesian Hilbert Map (SBHM)~\cite{senanayake2017bayesian} (with full and \NEW{diagonal covariance matrices}), and \NEW{OctoMap} on the Intel Research Lab dataset~\cite{Radish}. An RBK kernel with $\eta = 1, \Gamma = \sqrt{\gamma}I$ are used for SBHM and our SBKM maps. The metrics are the area under the receiver operating characteristic curve (AUC) and the negative log-likelihood loss (NLL).}
\label{table:compare_intel}
\centering
\begin{tabular}{ ccccccc } 
Methods  & SBHM & \NEW{SBHM} & SBKM & SBKM & \NEW{OM}\\
\hline
\hline
 $\gamma$ & $6.71$ & \NEW{$6.71$} & $6.71$ & $6.71$ & -\\
$\Sigma$  & full & \NEW{diag.} & full & $\lambda_{max}$ only & \NEW{-}\\
 AUC &  0.98&\NEW{0.98}& 0.96 & 0.95 & \NEW{0.95} \\
 NLL &  0.24&\NEW{0.24}&0.36 & 0.52 & \NEW{0.27}\\
 Feature dim. &  5600& \NEW{5600}& 3492 & 3492 & \NEW{-} \\
 Avg. time/scan & 11.8s & \NEW{0.03s} & 0.76s & 0.43s & \NEW{-}  \\
\end{tabular}
\end{table}

\begin{figure*}[t]
\centering

\begin{subfigure}[b]{0.32\textwidth}
        \centering
        \includegraphics[width=\textwidth]{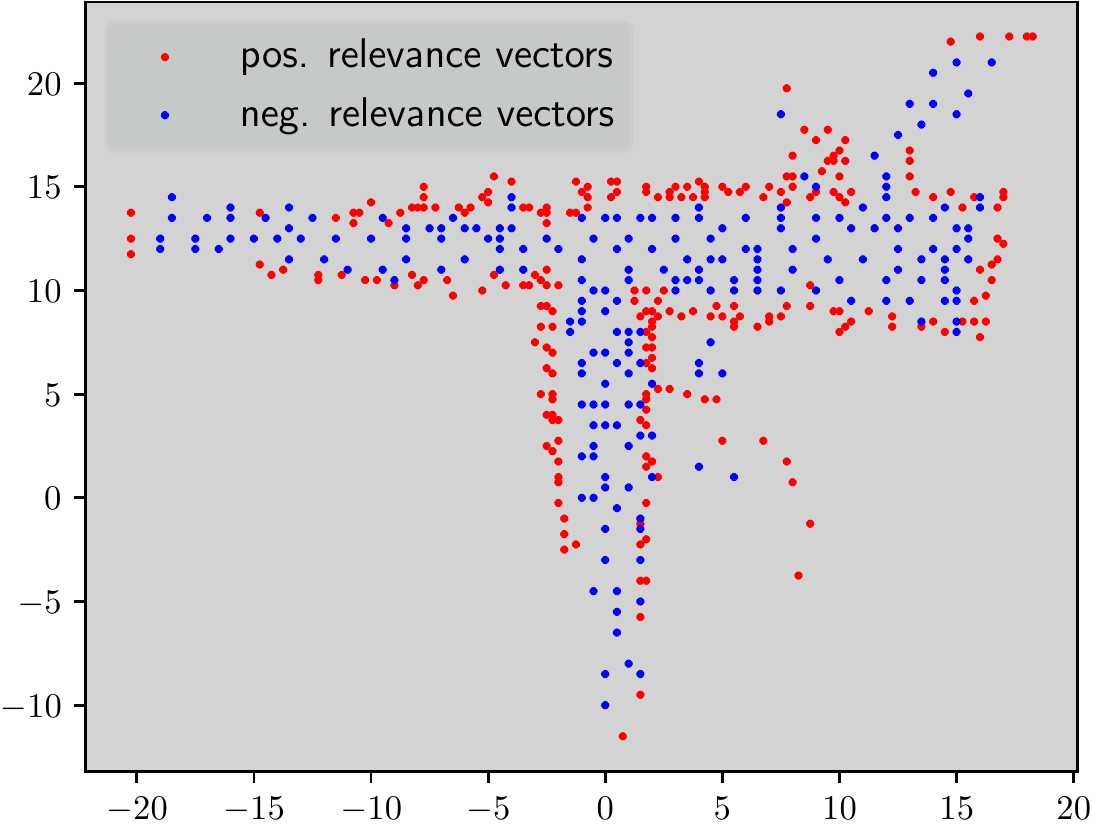}
                \caption{The final $453$ relevance vectors.}
        \label{fig:realexp_rvs}
\end{subfigure}%
\hfill
\begin{subfigure}[b]{0.357\textwidth}
        \centering
        \includegraphics[width=\textwidth]{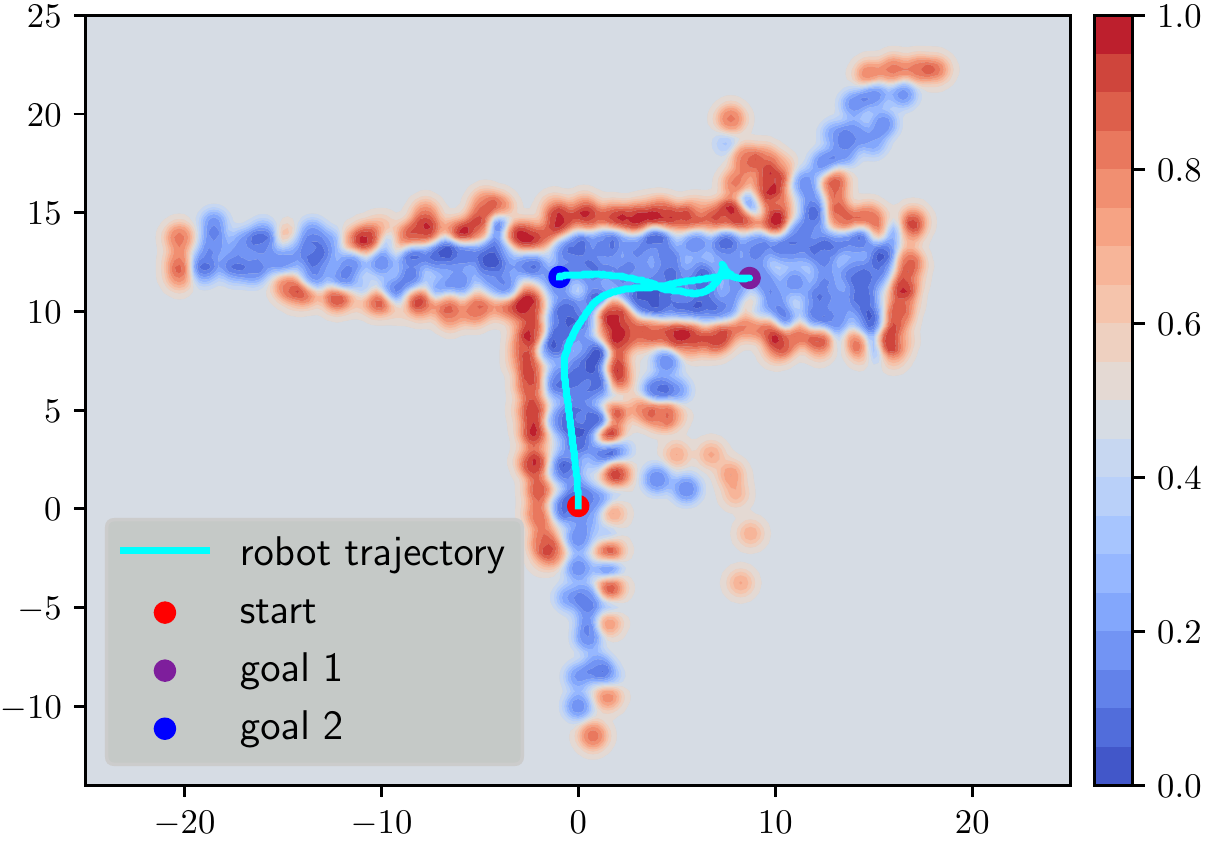}
                \caption{The final probabilistic map.}
        \label{fig:realexp_probmap}
\end{subfigure}%
\hfill
\begin{minipage}[b]{0.3\linewidth}
 \begin{subfigure}[b]{\textwidth}
        \centering
        \includegraphics[width=\textwidth]{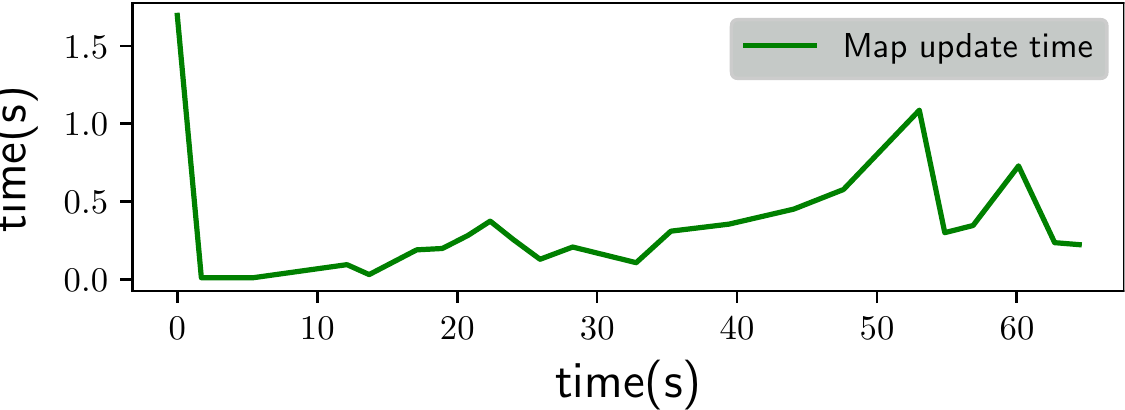}
                \caption{Map update time.}
        \label{fig:realexp_mapupdate}
\end{subfigure}%

    \begin{subfigure}[b]{\textwidth}
        \centering
        \includegraphics[width=\textwidth]{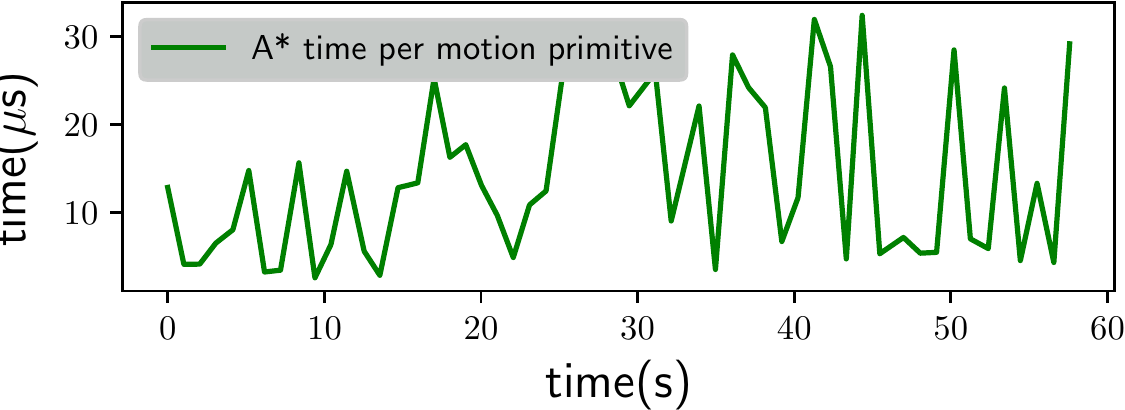}
                \caption{Planning time per motion primitive.}
        \label{fig:realexp_planning_time}
\end{subfigure}%
  \end{minipage} 
\caption{Real experiment with an autonomous Racecar robot navigating in an unknown hallway environment: (a) The final $453$ relevance vectors; (b) The final probabilistic map; (c) Map update time; (d) Planning time per motion primitive. Please refer to our website \url{https://thaipduong.github.io/sbkm} for the experiment video.}
\label{fig:realexp}
\end{figure*}

In this section, we compared our sparse Bayesian kernel-based map (SBKM) approach with \NEW{other probablistic occupancy maps: OctoMap~\cite{octomap} and} Sequential Bayesian Hilbert Map (SBHM)~\cite{senanayake2017bayesian}, a similar approach for Bayesian probabilistic mapping method from streaming local observation. We tried our best to match the parameters for a fair comparison, e.g. using the same kernel parameter $\Gamma = \sqrt{\gamma}\bfI$ with $\gamma = 6.71$ as provided by SBHM code \cite{senanayake2017bayesian}. The Intel Research Lab dataset~\cite{Radish} was used with both methods to build the map of the environment in an online manner. Our online training data (Sec.~\ref{subsec:online_probit_rvm}) were generated from a grid with resolution $0.2m$. Fig. \ref{fig:sbhm_intel} and  \ref{fig:rvmmap_intel} show similar final maps from the SBHM approach and our SBKM method, respectively. \NEW{Fig. \ref{fig:rvmmap_intel_variance} plots the our SBKM map's variance, distinguishing between known (low variance) and unknown (high variance) regions}. The dataset was split into a training set ($90\%$) and a test set ($10\%$). \NEW{Since we only considered probabilistic occupancy maps in this section,} the metrics for comparison were the area under the receiver operating characteristic curve (AUC) and the negative log-likelihood loss (NLL) of a point $\bfx$, defined as $NLL(y|\bfx, \bfxi) = -\log{p(y| \bfx, \bfxi)}$, 
where $y \in \{-1,1\}$ is the true label and $p(y| \bfx, \bfxi)$ is the predictive distribution in Eq. \eqref{eq:rvm_probit_score}. The AUC score and NLL loss were calculated over the test set. 

Table \ref{table:compare_intel} presents the metrics for both mapping methods. The SBHM map used a fixed grid of $5600$ hinged points with resolution $0.5m$ \NEW{for feature vector calculation}. Meanwhile, our map incrementally learned a sparse set of relevance vectors from the training dataset, not requiring a set of fixed hinged points which is hard to predetermine for unknown environments. Our final map's AUC score and NLL loss were slightly worse than those of SBHM with full covariance matrix while maintaining $\sim 35\%$ fewer points to represent the environment and having faster map updates with less than $1s$ per scan, on average, as shown in Table \ref{table:compare_intel} and Fig. \ref{fig:update_time_intel}. Our training algorithm incrementally built the set of relevance vectors and only updated the weights of the local vectors due to the use of $K$ nearest relevance vectors in Alg. \ref{alg:rvm_training_online}. Consequently, it did not have a fixed global set of points to optimize over as done by SBHM, leading to suboptimality in trade-off for sparseness. Note that both our map and the SBHM map estimated the mean $\bfmu$ and the full covariance matrix $\bfSigma$ of the weights' posterior for test time. If our collision checking algorithms are used for planning, only the largest eigenvalue $\lambda_{max}$ of $\bfSigma$ is needed and can be calculated efficiently using the sparsified inverse covariance matrix as shown in Sec. \ref{subsec:storage_req}. In this case, Table \ref{table:compare_intel} shows that our map update time was reduced by half to about $\sim 0.43s$ per scan (Table \ref{table:compare_intel} and Fig. \ref{fig:update_time_intel}) while offering similar AUC score to that of our SBKM map with full covariance matrix. The higher NLL loss was due to the upper bound used in Prop. \ref{prop:rvm_prob_bound} for point classification instead of the true occupancy probability. \NEW{Meanwhile, Octomap's AUC score was lower than ours and SBHM's as the default maximum ($0.97$) and minimum ($0.12$) values of the occupancy probability were used. The feature dimension and the average update time were not compared as OctoMap is not a kernel-based method as SBHM and our SBKM. A variant of SBHM that only uses diagonal covariance matrix updated the map $25$ times faster than our SBKM method with full covariance matrix. While SBKM can be improved using a diagonal covariance matrix, we leave this investigation for future work.}

\subsection{\NEW{Decision Boundary's Conservativeness}}
\label{subsec:boundary_conservativeness}

\begin{figure}[t]
\centering

\begin{subfigure}[b]{0.45\textwidth}
        \centering
        \includegraphics[width=\textwidth]{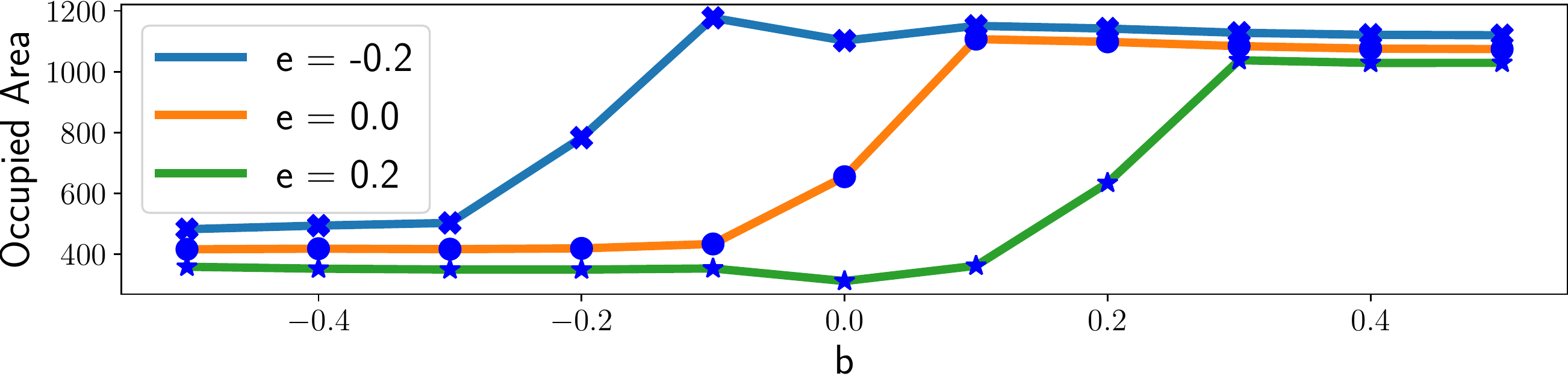}
        \caption{\NEW{The occupied area (with different values of $e$) versus the bias $b$.}}
        \label{fig:occupied_area_vs_b_e}
\end{subfigure}%

\begin{subfigure}[b]{0.45\textwidth}
        \centering
        \includegraphics[width=\textwidth]{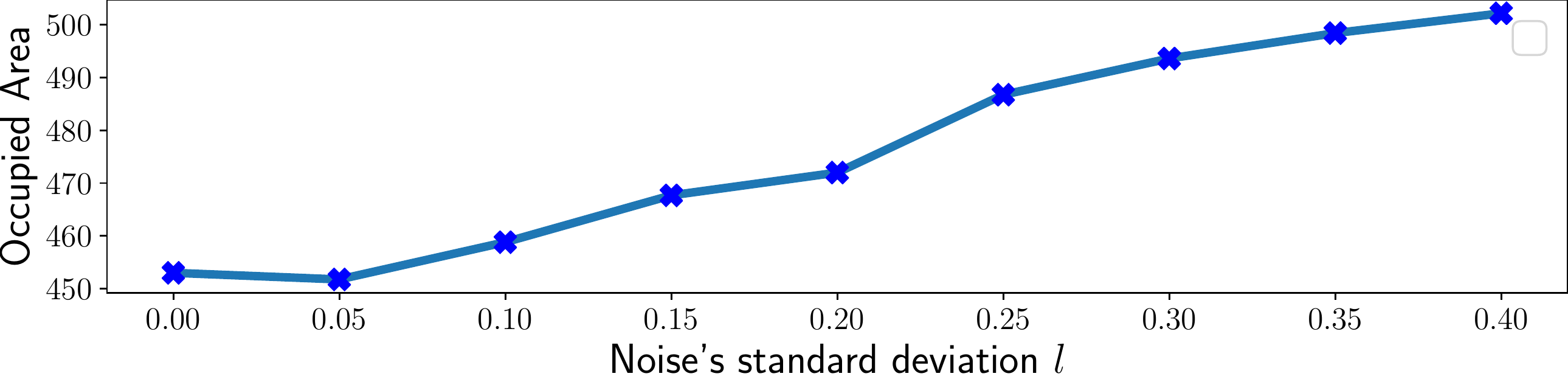}
        \caption{\NEW{The occupied area (with $b = -0.01, e = 0.0$) versus the noise's standard deviation.}}
        \label{fig:occupied_area_vs_noise}
\end{subfigure}%
\caption{\NEW{Occupied area versus the bias $b$ and the threshold $e$ (a) and versus noise level (b).}}
\end{figure}

\begin{figure*}[t]
\centering
\begin{subfigure}[b]{0.25\textwidth}
        \centering
        \includegraphics[width=\textwidth]{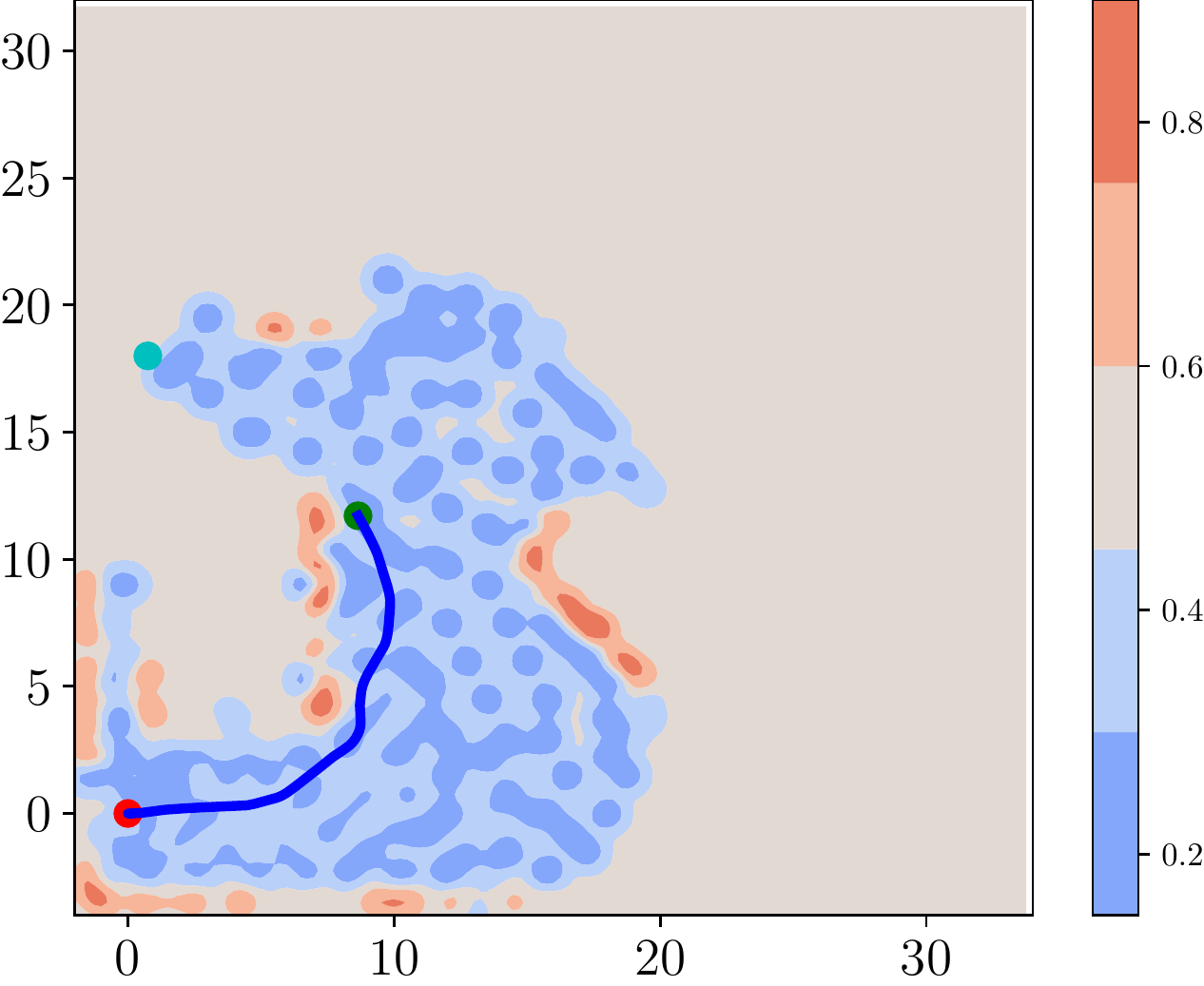}
        \caption{\NEW{$t=76s$}}
        \label{fig:active_map1}
\end{subfigure}%
\hfill
\begin{subfigure}[b]{0.25\textwidth}
        \centering
        \includegraphics[width=\textwidth]{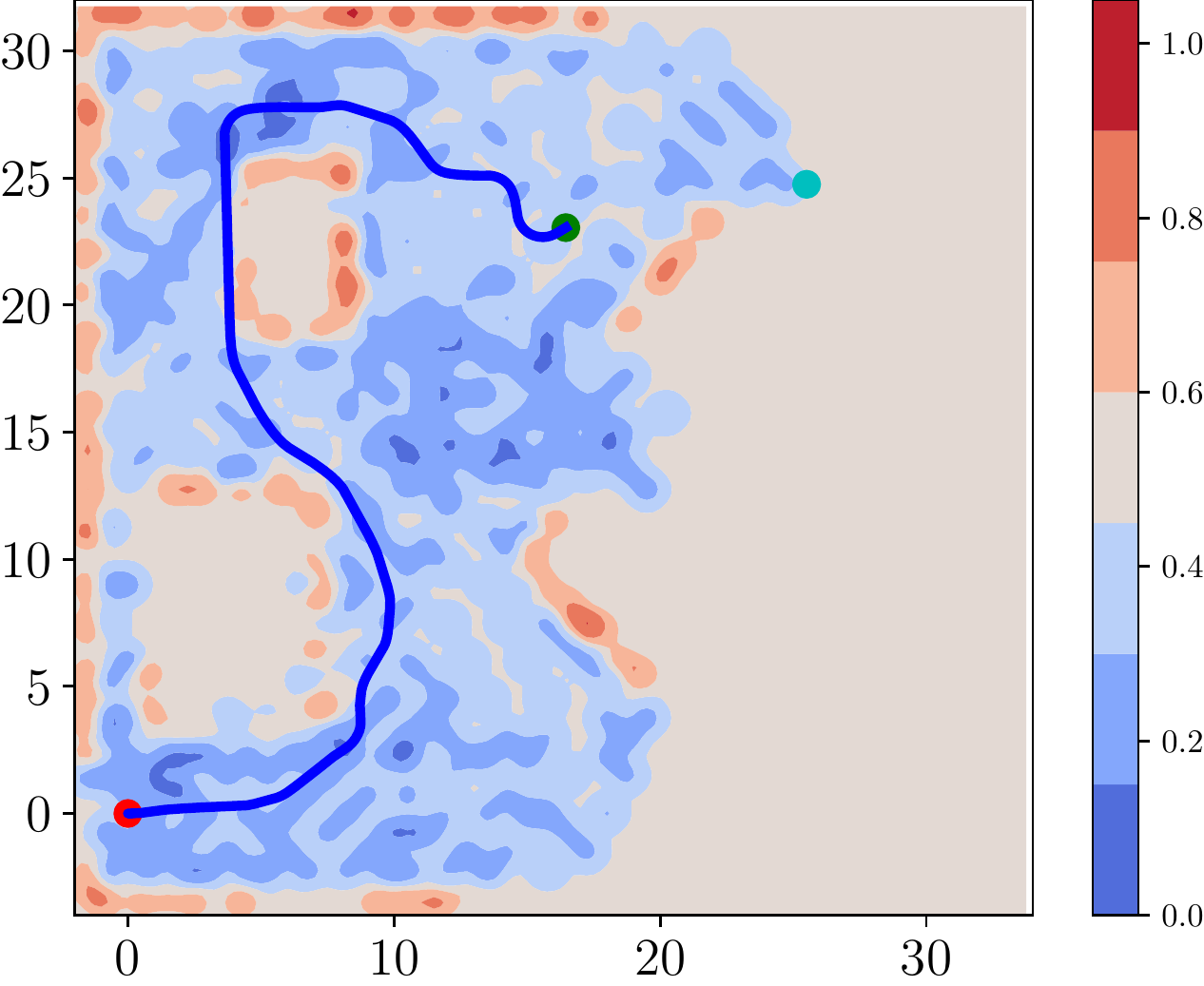}
        \caption{\NEW{$t=209s$}}
        \label{fig:active_map2}
\end{subfigure}%
\hfill
\begin{subfigure}[b]{0.25\textwidth}
        \centering
        \includegraphics[width=\textwidth]{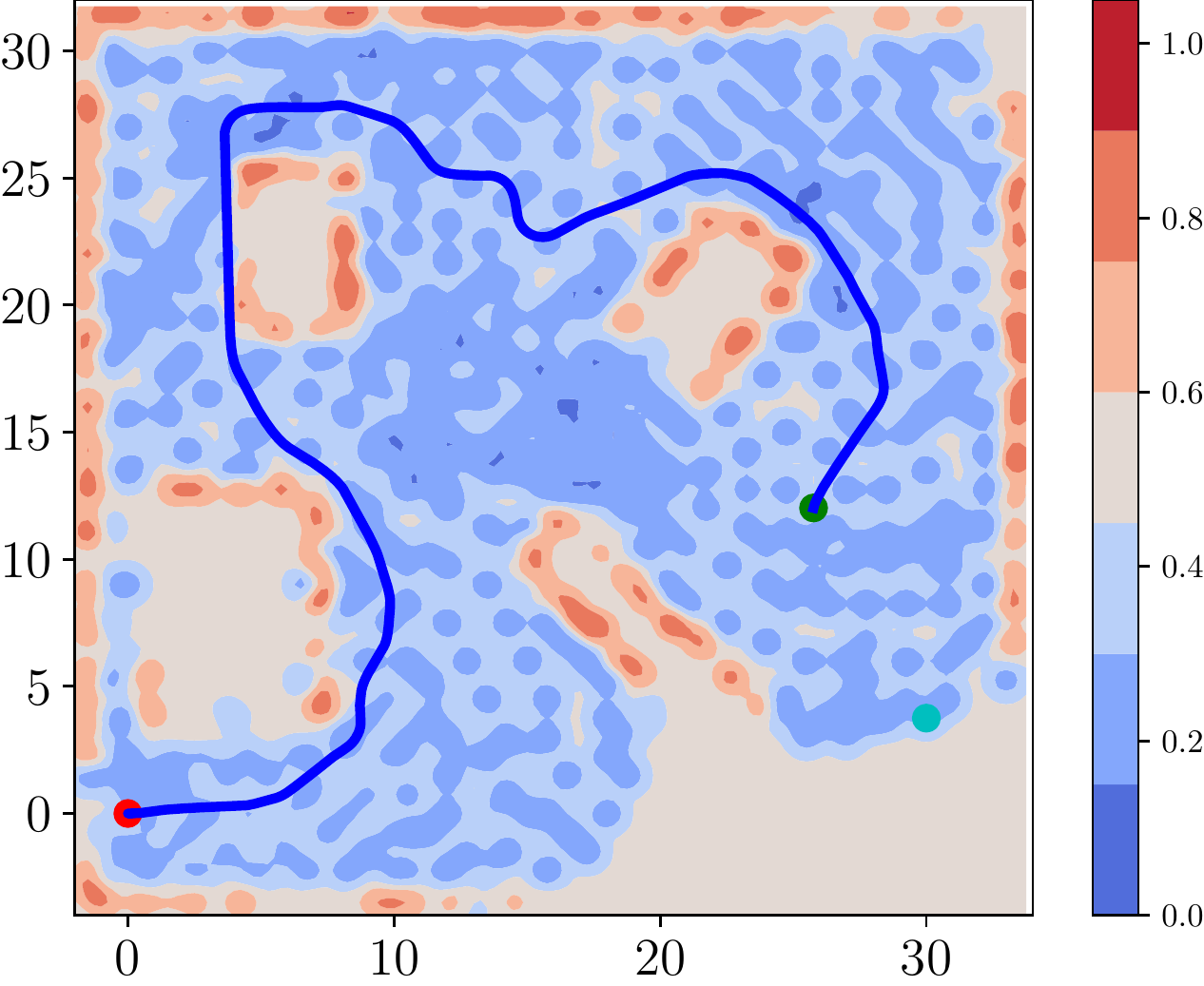}
        \caption{\NEW{$t=301s$}}
        \label{fig:active_map3}
\end{subfigure}%
\hfill
\begin{subfigure}[b]{0.25\textwidth}
        \centering
        \includegraphics[width=\textwidth]{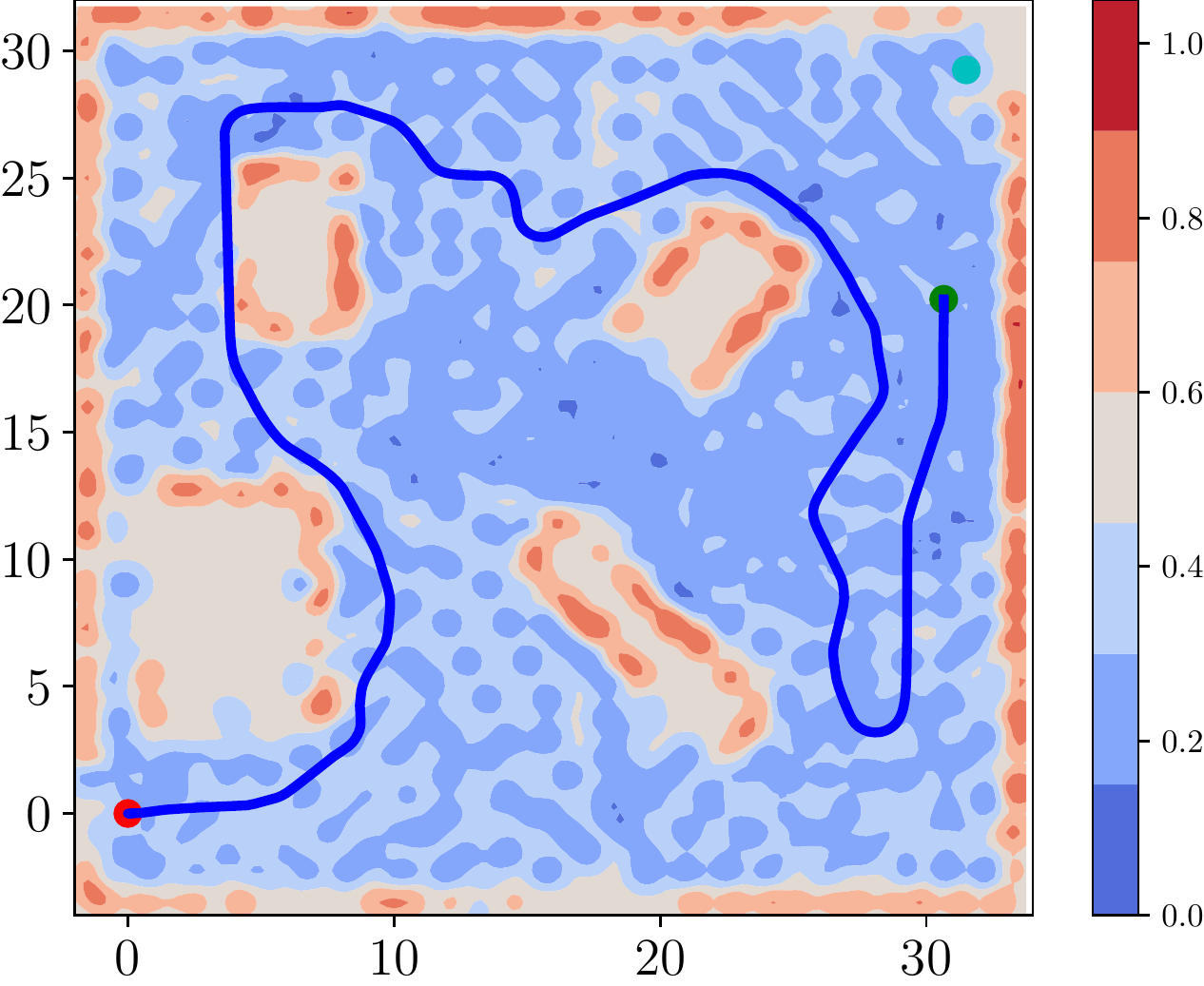}
        \caption{\NEW{$t=414s$}}
        \label{fig:active_map4}
\end{subfigure}%
\caption{\NEW{Active mapping at times $t=76s, 209s, 301s, 414s$. The red, green and cyan dots are the initial and current robot positions, and the goal actively picked to reduce the map entropy, respectively. The robot trajectory is in blue.}}
\label{fig:active_mapping}
\end{figure*}

\begin{figure}[t]
\centering
        \includegraphics[width=0.45\textwidth]{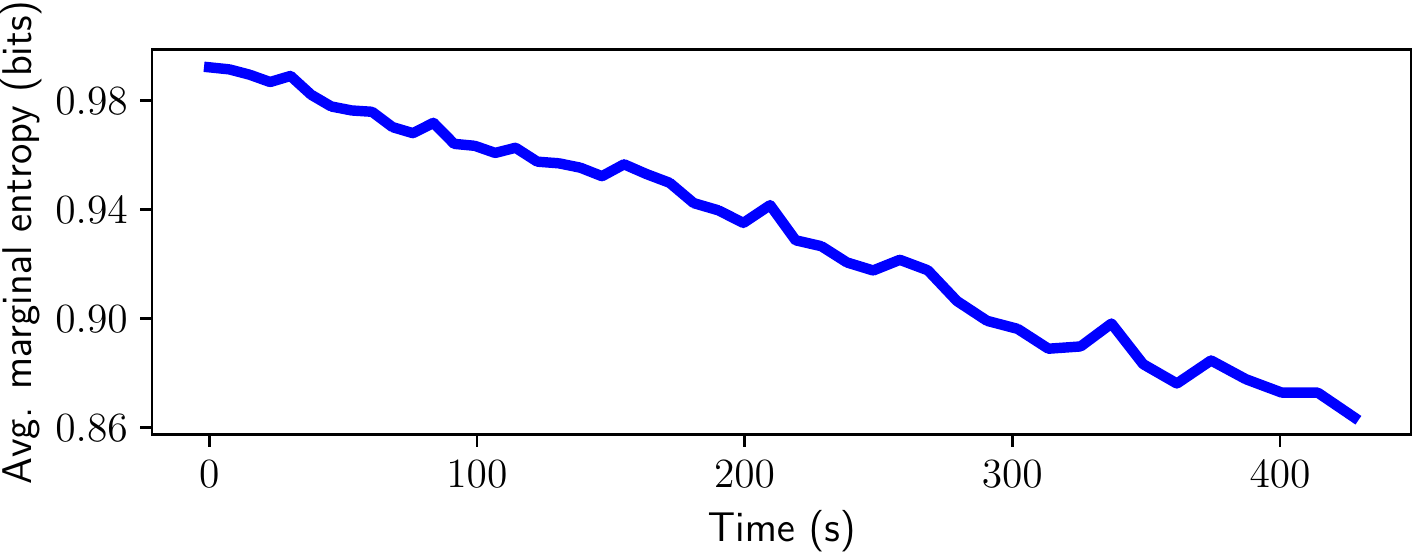}
\caption{\NEW{Average marginal entropy of a point in the map over time.}}
\label{fig:active_mapping_entropy}
\end{figure}

\NEW{
As the decision boundary affects the free/occupied area for robot navigation, we examined the decision boundary's conservativeness with respect to the bias $b$, the threshold $e$, and the measurement noise. Fig. \ref{fig:occupied_area_vs_b_e} plots the occupied area in our map, i.e., the area with occupancy probability greater than the threshold $e$, built from the Intel Lab dataset~\cite{Radish} for different values of the bias $b$ and the threshold $e$. As expected, the occupied area increases, i.e., smaller navigable area, if $e$ decreases and/or $b$ increases. Our SKM model \cite{duong2020autonomous} does not offer similar free/occupied area tuning as it does not have parameters $b$ and $e$. 

The decision boundary's conservativeness is also affected by the measurement noise since the robot should proceed carefully around the obstacle boundary if the depth measurements are noisy. To illustrate the robustness of our approach against measurement noise, we added Gaussian noise with zero mean and variance $l^2$ to the laser endpoints in the dataset and trained our model. The occupied area versus the noise's standard deviation $l$ is plotted in Fig. \ref{fig:occupied_area_vs_noise}. As the noise level $l$ increases, the occupied area increases making the robot more cautious around the obstacles in the environment.}

\subsection{Real Experiments}
\label{subsec:real_experiments}

Real experiments were carried out on an $1/10$th scale Racecar robot equipped with a Hokuyo UST-10LX Lidar and Nvidia TX2 computer. The robot body was modeled by a ball of radius $r = 0.25m$. The online training data (Sec.~\ref{subsec:online_probit_rvm}) were generated from a grid with resolution $0.25m$. We used an RBF kernel parameter $\Gamma = \sqrt{\gamma}\bfI$ with $\gamma = 3.0$ and an $R^*$-tree approximation of the score $F(\bfx)$ with $K^++K^- = 20$ nearest support vectors around the robot location $\bfp_k$ for map updating. For motion planning, second-order polynomial motion primitives were generated with time discretization of $\tau = 1$s as described in Sec.~\ref{subsec:auto_nav}. The motion cost was defined as $c(\bfs, \bfa) := (\|\bfa\|^2 + 2)\tau$ to encourage both smooth and fast motion~\cite{liu2017search}. Alg.~\ref{alg:rvm_collision_checking_curve} with Eq.~\eqref{eq:rvm_ru_star}, $\varepsilon = 0.1$, score approximation with $K^+ = K^- = 2$, \NEW{and threshold $e = -0.01$} was used for collision checking in Alg. \ref{alg:get_successors}. The trajectory generated by an $A^*$ motion planner was tracked using a closed-loop controller~\cite{arslan2016exact}. The robot navigated in an unknown hallway to two destinations consequently chosen by a human operator. Fig. \ref{fig:realexp_rvs} shows the learned  relevance vectors representing the environment. Fig. \ref{fig:realexp_probmap} shows the probabilistic map recovered from the relevance vectors together with the robot trajectory and the two chosen destinations.

The time taken by Alg.~\ref{alg:rvm_training_online} to update the relevance vectors from one lidar scan and the $A^*$ replanning time per motion primitive are shown in Fig.~\ref{fig:realexp_mapupdate} and \ref{fig:realexp_planning_time}. Map updates implemented in Python took $0.4$s on average. It took a longer time ($\sim 1s$) to update the map when the robot observed new large parts of the environment, e.g., at the beginning and toward the end of our experiment. To evaluate collision checking time, the $A^*$ replanning time was normalized by the number of motion primitives being checked to account for differences in planning to nearby and far goals. The planning time per motion primitive was $\sim 15\mu s$ on average and $\sim 30 \mu s$ at most, suggesting our collision checking algorithms' suitability for real-time applications.

\NEW{In both simulations and the real experiment, the free area contains multiple blobs of points with low occupancy probability, caused by the sparse set of relevance vectors and the fixed kernel parameters of SBKM. To improve on this, kernel parameter learning \cite{senanayake2018automorphing} can be used to to adaptively update the kernel parameters at different location based on how large the free space is. However, we leave this for future work.
}

\subsection{\NEW{Active Mapping}}
\label{subsec:active_mapping}

\NEW{Our SBKM representation enables uncertainty quantification which besides for collision checking can be used for active mapping. This ability is not offered by non-Bayesian mapping methods, such as SKM. In this section, we demonstrate active mapping of an unknown simulated environment using SBKM. Our approach estimates the map uncertainty in different regions and chooses the region with the highest uncertainty as the goal region. Specifically, we maintain a frontier, defined as a list of $L$ candidate poses $\calP_l$, $l = 1,2, \ldots, L$. For each pose $\calP_l$, we calculate the map uncertainty $H(\calS_l)$ of the field of view $\calS_l$ of the depth sensor. The map uncertainty of a region $\calS$ is measured as the average marginal entropy over the region:
\begin{equation}
\label{eq:marginal_entropy}
	H(\calS) = \frac{1}{|\calS|} \int_{\calS} h(\bfx) d\bfx,
\end{equation}
where $h(\bfx)$ is the marginal entropy of a point $\bfx$ in the region $\calS$, calculated using the predictive distribution in Def~\ref{def:threshold_classification} as
\begin{equation}
\begin{aligned}
h(\bfx) &=& -P(y = 1|\bfx, \bfxi) \log_2 P(y = 1|\bfx, \bfxi)& \\
&& - P(y = 0|\bfx, \bfxi)\log_2 P(y = 0|\bfx_i, \bfxi)&,
\end{aligned}
\end{equation}
$y \in \{-1,1\}$ is the predictive label of the point $\bfx$, and $|\calS|$ denotes the area of the region $\calS$. We choose the region $S_{l^*}$ with the largest average marginal entropy to explore:
\begin{equation}
l^* = \underset{l=1,2,\ldots, L}{\mathrm{argmax}} H(\calS_l).
\end{equation}
 
}


\NEW{In our active mapping experiment, the candidate poses $\calP_1, \calP_2, \ldots, \calP_L$ in the frontier were sampled from the laser endpoints up to the current time $t$ with $4$ different yaw angles: $0, \frac{\pi}{4}, \frac{\pi}{2}, \frac{3\pi}{4}$. Since the laser scans could not see through obstacles, we gained little information of the environment by placing the robot near the occupied regions. Therefore, only the endpoints with maximum lidar range, i.e. the laser ray did not hit an obstacle, and at least $2m$ away from the positive relevance vectors were considered. A hypothetical lidar field of view $\calS_l$ (similar to Fig. \ref{fig:laser_uninflated_ws} without the obstacles) simulating a Hokuyo UST-10LX lidar, was placed at each candidate pose~$\calP_l$. We sampled $N = 100$ points $\bfx_i, i = 1, \ldots, N,$ from $\calS_l$ and calculated their marginal entropy $H(\bfx_i)$. The average marginal entropy of the region $\calS_l$ (Eq. \eqref{eq:marginal_entropy}) was approximated as 
\begin{equation}
\label{eq:approx_marginal_entropy}
H(\calS_l) \approx \frac{1}{N}\sum_{i = 1}^{N} h(\bfx_i).
\end{equation}
The robot picked the goal region $\calS_{l^*}$ with the highest map uncertainty from the set $\{\calS_1, \calS_2, \ldots, \calS_L\}$ every $0.5s$ and planned a trajectory to reach the goal using our collision checking methods and the same $A^*$ planner used in Sec. \ref{subsec:real_experiments}. Fig. \ref{fig:active_mapping} shows the SBKM map, the robot trajectory and the candidate pose associated with the chosen goal region at different times as the robot successfully explored and actively built the map of the environment. The average marginal entropy of the map (Fig.~\ref{fig:active_mapping_entropy}), estimated using Eq.~\eqref{eq:approx_marginal_entropy} with $N = 20736$ points sampled on a regular grid of resolution $0.25m$, shows our active mapping approach reduced the map uncertainty over time. 
}

\section{Conclusion}
\label{sec:conclusion}

This paper proposes a sparse Bayesian kernel-based mapping method for efficient online generation of large occupancy maps, supporting autonomous robot navigation in unknown environments. Our map representation, as a sparse set of relevance vectors learned from streaming range observations of the environment, is efficient to store. It supports efficient and complete collision checking for general curves modeling potential robot trajectories. Our experiments demonstrate the potential of this model at generating compressed, yet accurate, probabilistic environment models. Our results offer a promising venue for quantifying safety and uncertainty and enabling real-time long-term autonomous navigation in unpredictable environments. 
Future work will \NEW{further} explore active exploration and map uncertainty reduction, \NEW{kernel parameter learning} as well as simultaneous localization and mapping using the proposed map representations.


\appendices

\section{Proof of Proposition \ref{prop:rvm_prob_bound}}
\label{appendix:rvm_prob_bound}
\begin{proof}
A point $\bfx$ is considered free if:
\begin{equation}
\bfPhi_{\bfx}^\top\bfmu + b < e \sqrt{1 + \bfPhi_{\bfx}^\top \bfSigma \bfPhi_{\bfx}},
\end{equation}
where $\bfPhi_{\bfx}$ is the feature vector $\bfPhi_{\bfx} = [k_1(\bfx), k_2(\bfx), \ldots, k_M(\bfx)]^\top$. We use the following lower bound and upper bound on $\bfPhi_{\bfx}^\top \bfSigma \bfPhi_{\bfx}$: $0 \leq \bfPhi_{\bfx}^\top \bfSigma \bfPhi_{\bfx} \leq \lambda_{\max} \sum_{m = 1} ^{M}(k_m(\bfx))^2$ where $\lambda_{\max} \geq 0$ is the largest eigenvalue of the covariance matrix $\bfSigma$. Since $k_m(\bfx) > 0$ for all $m$, we have:
\begin{align}
1\leq \sqrt{1 + \bfPhi_{\bfx}^\top \bfSigma \bfPhi_{\bfx}} \leq  1 +\sqrt{\lambda_{\max}} \sum_{m = 1} ^{M}(k_m(\bfx)).
\end{align}
Therefore, the point $\bfx$ is still free if
\begin{equation}
\bfPhi_{\bfx}^\top\bfmu + b  \leq e (1 +\mathbbm{1}_{\{\textcolor{blue}{e\leq 0}\}}\sqrt{\lambda_{\max}} \sum_{m = 1} ^{M}(k_m(\bfx))),
\end{equation}
or $\sum_{m = 1} ^{M}(\bfmu_{m} - e\mathbbm{1}_{\{\textcolor{blue}{e\leq 0}\}}\sqrt{\lambda_{\max}})k_m(\bfx) + b - e \leq 0$.\qedhere
\end{proof}

\section{Proof of Proposition \ref{prop:rvm_prob_bound_amgm}}
\label{appendix:rvm_prob_bound_amgm}
\begin{proof}
A point $\bfx$ is free if Eq. \eqref{eq:rvm_prob_bound_rewritten} holds. Let $\bfx^+_*$ be the closest positive relevance vector to $\bfx$ and $\bfx_j^-$ be any negative relevance vector. We have:
\begin{eqnarray}
\sum_{i = 1} ^{M^+}\nu_i^+k(\bfx, \bfx_i^+) - \sum_{j = 1} ^{M^-}\nu_j^-k(\bfx, \bfx_j^-) + b - e \leq \nonumber \\
\leq (\sum_{i = 1} ^{M^+}\nu_i^+)k(\bfx, \bfx^+_*) - \nu_j^-k(\bfx, \bfx_j^-) + b - e \nonumber
\end{eqnarray}
Under Assumptions \ref{assumption:rbf} and \ref{assumption:threshold_b}, both terms $\nu_j^-k_j(\bfx)$ and $e-b$ are non-negative. By the arithmetic mean- geometric mean inequality, we have:
\small
\begin{eqnarray}
\nu_j^-k(\bfx, \bfx_j^-)+ e-b &=& n_2 \frac{\nu_j^-k(\bfx, \bfx_j^-)}{n_2} + n_1 \frac{e-b}{n_1} \nonumber \\
&\geq & \scaleMathLine[0.6]{(n_1+n_2)\left(\frac{\nu_j^-k(\bfx, \bfx_j^-)}{n_2}\right)^{\frac{n_2}{n_1 + n_2}}\left(\frac{e-b}{n_1}\right)^{\frac{n_1}{n_1 + n_2}}} \nonumber \\
&=& \rho(e-b,\nu_j^- k(\bfx,\bfx_j^-)). \nonumber
\end{eqnarray}
\normalsize
Therefore, a point $\bfx$ is free if
\begin{equation}
\label{eq:g3_appdx}
(\sum_{i = 1} ^{M^+}\nu_i^+)k(\bfx, \bfx^+_*) - \rho(e-b,\nu_j^- k(\bfx,\bfx_j^-)) \leq 0. \qedhere
\end{equation}
\end{proof}
%

\normalsize
\iftrue
\section{Proof of Proposition \ref{prop:rvm_prob_bound_lineseg}}
\label{appendix:rvm_prob_bound_lineseg}
\begin{proof}
By plugging $k(\bfx,\bfx_*^+) = \eta e^{-\Vert \Gamma(\bfx - \bfx^+_* )\Vert^2}$, and $k(\bfx,\bfx_j^-) = \eta e^{-\gamma\Vert \Gamma(\bfx - \bfx^-_j) \Vert^2}$ into Eq. \eqref{eq:g3_appdx}, a point $\bfx$ is free if
\begin{equation}
\label{eq:g3_appdx_rbf}
e^{-\Vert {\bf\Gamma}(\bfx - \bfx^+_*) \Vert^2 + \frac{n_2}{n_1+n_2}\Vert {\bf\Gamma}(\bfx - \bfx^-_j) \Vert^2} \leq \frac{\rho(e-b,\nu_j^-)}{\eta^{\frac{n_1}{n_1+n_2}}\sum_{i = 1} ^{M^+}\nu_i^+}
\end{equation}

Substituting the test point $\bfx$ by $\bfp(t) = \bfp_0 + t\bfv$ in Eq. \eqref{eq:g3_appdx_rbf}, the point $\bfp(t)$ is free if:
\small
\begin{eqnarray}
V(t, \bfx^+_*,\bfx^-_j)  &=& -(n_1+n_2)\Vert {\bf\Gamma}(\bfp_0 + t\bfv - \bfx^+_*) \Vert^2  \nonumber \\
&& + n_2\Vert {\bf\Gamma}(\bfp_0 + t\bfv - \bfx^-_j) \Vert^2 - (n_1+n_2)\beta \leq 0, \nonumber
\end{eqnarray}
\normalsize
where $\beta = \log\frac{\rho(e-b,\nu_j^-)}{\eta^{\frac{n_1}{n_1+n_2}}\sum_{i = 1} ^{M^+}\nu_i^+}$. By expanding the quadratic norms in $V(t, \bfx^+_*,\bfx^-_j)$, the point $\bfp(t)$ is free if:
\begin{eqnarray}
V(t, \bfx^+_*,\bfx^-_j) &=& at^2 + b(\bfx_*^+, \bfx_j^-)t+ c(\bfx_*^+, \bfx_j^-)\leq 0 \label{eq:quad_poly_appx} \\
\text{where } a &=& -n_1\Vert{\bf\Gamma}\bfv\Vert^2, \nonumber \\
b(\bfx_*^+, \bfx_j^-) &=& -2\bfv^\top {\bf\Gamma}^\top {\bf\Gamma}(n_1\bfp_0 - (n_1 + n_2)\bfx^+_* + n_2\bfx^-_j), \nonumber \\
c(\bfx_*^+, \bfx_j^-) &=& - (n_1 + n_2)\Vert{\bf\Gamma}(\bfp_0- \bfx^+_*)\Vert^2  \nonumber \\
&&+ n_2\Vert{\bf\Gamma}(\bfp_0- \bfx^-_j)\Vert^2 - (n_1 + n_2)\beta .\nonumber
\end{eqnarray}
Note that $V(t, \bfx^+_*,\bfx^-_j)$ is a quadratic polynomial in $t$ and the point $\bfp(t)$ is free if $V(t, \bfx^+_*,\bfx^-_j) \leq 0$.
\begin{enumerate}
	\item If it has less than $2$ roots, Eq. \eqref{eq:quad_poly_appx} is satisfied for all $t$.
	\item If it has $2$ roots $t_1 < t_2$, then $V(t, \bfx^+_*,\bfx^-_j) \leq 0$ for $t \geq t_2$ or $t \leq t_1$. There are three cases: 
\begin{enumerate}
	\item $t_1 < t_2 \leq 0$: $V(t, \bfx^+_*,\bfx^-_j) \leq 0$ for all $t \geq 0$ or the entire ray $s(t)$ is free;
	\item $0 \leq t_1 < t_2$: $V(t, \bfx^+_*,\bfx^-_j) \leq 0$ for $t \in [0, t_1]$ or the ray $s(t)$ is free for $t \in [0, t_1]$.
	\item $t_1 \leq 0 \leq t_2$: $V(0, \bfx^+_*,\bfx^-_j) \geq 0$ or the ray $s(t)$ is colliding.
\end{enumerate}
\end{enumerate}
\begin{eqnarray}
\text{Let }&&\tau(\bfp_0, \bfx^+_*, \bfx_j^-) \nonumber \\
&=&\begin{cases} 
      +\infty, & \text{if } V(t, \bfx^+_*,\bfx^-_j)  \text{ has less than 2 roots} \\
      +\infty, & \text{if } V(t, \bfx^+_*,\bfx^-_j) \text{ has 2 roots $t_1 < t_2 \leq 0$} \\
      t_1 & \text{if } V(t, \bfx^+_*,\bfx^-_j)  \text{ has 2 roots $0 \leq t_1 < t_2$} \\
      0 & \text{if } V(t, \bfx^+_*,\bfx^-_j)  \text{ has 2 roots $t_1 \leq 0 \leq t_2$} 
\end{cases}. \nonumber
\end{eqnarray}
Note that $\bfx^+_*$ varies with $t$ but belongs to a finite set, we can calculate $\tau(\bfp_0, \bfx^+_i, \bfx_j^-)$ for all positive relevance vectors $\bfx^+_i$ and take the minimum value. Therefore, $\bfp(t)$ is free as long as:
\begin{equation}
\label{eq:rvm_tu_appx}
t \leq t_u = \min_{i = 1, \ldots, M^+} {\tau(\bfp_0, \bfx^+, \bfx_j^-)}
\end{equation}
Note that Eq. \eqref{eq:rvm_tu_appx} holds for any negative relevance vector $\bfx_j^-$. Therefore, the point $\bfp(t)$ is free as long as $t~\leq~t_u^*~=~\max_{j = 1, \ldots, M^-}\min_{i =1, \ldots, M^+} {\tau(\bfp_0, \bfx^+, \bfx_j^-)}.$ \qedhere
\end{proof}
\fi

\section{Proof of Proposition \ref{prop:rvm_prob_bound_ball}}
\label{appendix:rvm_prob_bound_ball}
\begin{proof}
\NEW{Consider an arbitrary ray $\bfp'(t) = \bfp_0 + t\bfv', 0\leq t < \infty$. If we scale the velocity $\bfv'$ by a positive constant $\lambda$, i.e. $\bfv = \lambda \bfv'$, the ray $\bfp(t) = \bfp_0 + t\bfv, 0\leq t < \infty$ represents the same ray as $\bfp'(t)$. We scale the vector $v'$ by $\lambda = \frac{1}{\Vert \bf\Gamma v' \Vert}$ so that the velocity vector $\bfv$ satisfies $\Vert \bf\Gamma v \Vert = 1$. Using the Cauchy-Schwarz inequality in Eq. \eqref{eq:quad_poly_appx} in Appendix \ref{appendix:rvm_prob_bound_lineseg}}, we have:
\begin{eqnarray}
&&-2t\bfv^\top {\bf\Gamma}^\top {\bf\Gamma} (n_1\bfp_0 - (n_1 + n_2)\bfx^+_* + n_2\bfx^-_j) \nonumber \\
&\leq& 2t\Vert {\bf\Gamma}(n_1\bfp_0 - (n_1 + n_2)\bfx^+_* + n_2\bfx^-_j)\Vert \nonumber
\end{eqnarray}
Therefore, the point $\bfp(t)$ is free if $\bar{V}(t, \bfx^+_*,\bfx^-_j) \leq 0$. \NEW{Following the same reasoning as Prop. \ref{prop:rvm_prob_bound_lineseg}, 
the point $\mathbf{p}(t)$ is free for $0 < t < r_u$ or $0 < t < r_u^*$. In other words, the interior of the ellipsoids $\mathcal{E}(\bfp_0, r_u) \subseteq \mathcal{E}(\bfp_0, r_u^*)$ is free.} \qedhere
\end{proof}

\ifCLASSOPTIONcaptionsoff
  \newpage
\fi



%
\bibliographystyle{cls/IEEEtran}
\bibliography{bib/thai_ref.bib} 

\begin{thebibliography}{10}
\providecommand{\url}[1]{#1}
\csname url@samestyle\endcsname
\providecommand{\newblock}{\relax}
\providecommand{\bibinfo}[2]{#2}
\providecommand{\BIBentrySTDinterwordspacing}{\spaceskip=0pt\relax}
\providecommand{\BIBentryALTinterwordstretchfactor}{4}
\providecommand{\BIBentryALTinterwordspacing}{\spaceskip=\fontdimen2\font plus
\BIBentryALTinterwordstretchfactor\fontdimen3\font minus
  \fontdimen4\font\relax}
\providecommand{\BIBforeignlanguage}[2]{{%
\expandafter\ifx\csname l@#1\endcsname\relax
\typeout{** WARNING: IEEEtran.bst: No hyphenation pattern has been}%
\typeout{** loaded for the language `#1'. Using the pattern for}%
\typeout{** the default language instead.}%
\else
\language=\csname l@#1\endcsname
\fi
#2}}
\providecommand{\BIBdecl}{\relax}
\BIBdecl

\bibitem{human_friendly_nav_guzzi_icra13}
J.~Guzzi, A.~Giusti, L.~M. Gambardella, G.~Theraulaz, and G.~A.~D. Caro,
  ``{Human-friendly robot navigation in dynamic environments},'' in \emph{IEEE
  Int. Conf. on Robotics and Automation (ICRA)}, 2013, pp. 423--430.

\bibitem{safe_auto_nav_pavone_rss18}
L.~Janson, T.~Hu, and M.~Pavone, ``{Safe motion planning in unknown
  environments: optimality benchmarks and tractable policies},'' in
  \emph{Robotics: Science and Systems (RSS)}, Pittsburgh, Pennsylvania, June
  2018.

\bibitem{rgbd-slam}
P.~Henry, M.~Krainin, E.~Herbst, X.~Ren, and D.~Fox, ``{RGB-D mapping: using
  Kinect-style depth cameras for dense 3D modeling of indoor environments},''
  \emph{The International Journal of Robotics Research (IJRR)}, vol.~31, no.~5,
  pp. 647--663, 2012.

\bibitem{chisel}
M.~Klingensmith, I.~Dryanovski, S.~Srinivasa, and J.~Xiao, ``Chisel: real time
  large scale 3d reconstruction onboard a mobile device,'' in \emph{Robotics:
  Science and Systems (RSS)}, 2015.

\bibitem{elastic_fusion}
T.~Whelan, R.~Salas-Moreno, B.~Glocker, A.~Davison, and S.~Leutenegger,
  ``{ElasticFusion: real-time dense SLAM and light source estimation},''
  \emph{The International Journal of Robotics Research (IJRR)}, vol.~35,
  no.~14, pp. 1697--1716, 2016.

\bibitem{DenseSurfelMapping}
K.~{Wang}, F.~{Gao}, and S.~{Shen}, ``Real-time scalable dense surfel
  mapping,'' in \emph{IEEE Int. Conf. on Robotics and Automation (ICRA)}, 2019,
  pp. 6919--6925.

\bibitem{behley2018rss_surfel_slam}
J.~Behley and C.~Stachniss, ``{Efficient surfel-based SLAM using 3D laser range
  data in urban environments},'' in \emph{Robotics: Science and Systems (RSS)},
  2018.

\bibitem{Kaess_infiniteplanes}
M.~{Kaess}, ``Simultaneous localization and mapping with infinite planes,'' in
  \emph{IEEE Int. Conf. on Robotics and Automation (ICRA)}, 2015, pp.
  4605--4611.

\bibitem{Bowman_SemanticSLAM_ICRA17}
S.~Bowman, N.~Atanasov, K.~Daniilidis, and G.~Pappas, ``Probabilistic data
  association for semantic slam,'' in \emph{IEEE Int. Conf. on Robotics and
  Automation (ICRA)}, 2017.

\bibitem{quadric_slam}
L.~{Nicholson}, M.~{Milford}, and N.~{S\"underhauf}, ``{QuadricSLAM: dual
  quadrics from object detections as landmarks in object-oriented SLAM},''
  \emph{IEEE Robotics and Automation Letters}, vol.~4, no.~1, 2019.

\bibitem{cubeslam}
S.~{Yang} and S.~{Scherer}, ``{CubeSLAM: monocular 3-D object SLAM},''
  \emph{IEEE Transactions on Robotics}, vol.~35, no.~4, pp. 925--938, 2019.

\bibitem{Shan_OrcVIO_IROS20}
M.~Shan, Q.~Feng, and N.~Atanasov, ``{OrcVIO: object residual constrained
  visual-inertial odometry},'' in \emph{IEEE/RSJ Int. Conf. on Intelligent
  Robots and Systems (IROS)}, 2020.

\bibitem{teixeira2016real}
L.~Teixeira and M.~Chli, ``Real-time mesh-based scene estimation for aerial
  inspection,'' in \emph{IEEE/RSJ Int. Conf. on Intelligent Robots and Systems
  (IROS)}, 2016.

\bibitem{piazza2018real}
E.~Piazza, A.~Romanoni, and M.~Matteucci, ``Real-time cpu-based large-scale
  three-dimensional mesh reconstruction,'' \emph{IEEE Robotics and Automation
  Letters}, vol.~3, no.~3, pp. 1584--1591, 2018.

\bibitem{kimera}
A.~Rosinol, M.~Abate, Y.~Chang, and L.~Carlone, ``{Kimera: an open-source
  library for real-time metric-semantic localization and mapping},'' in
  \emph{IEEE Int. Conf. on Robotics and Automation (ICRA)}, 2020.

\bibitem{occgrid}
A.~Elfes, ``Using occupancy grids for mobile robot perception and navigation,''
  \emph{Computer}, vol.~22, no.~6, pp. 46--57, 1989.

\bibitem{InfiniTAM}
O.~K{\"{a}}hler, V.~A. Prisacariu, and D.~W. Murray, ``Real-time large-scale
  dense 3d reconstruction with loop closure,'' in \emph{European Conference on
  Computer Vision (ECCV)}, 2016, pp. 500--516.

\bibitem{VoxelMapVisualSLAM}
M.~Muglikar, Z.~Zhang, and D.~Scaramuzza, ``{Voxel map for visual SLAM},'' in
  \emph{IEEE Int. Conf. on Robotics and Automation (ICRA)}, 2020.

\bibitem{octomap}
A.~Hornung, K.~M. Wurm, M.~Bennewitz, C.~Stachniss, and W.~Burgard,
  ``{OctoMap}: an efficient probabilistic {3D} mapping framework based on
  octrees,'' \emph{Autonomous Robots}, vol.~34, no.~3, pp. 189--206, 2013.

\bibitem{octree_fusion}
M.~Zeng, F.~Zhao, J.~Zheng, and X.~Liu, ``Octree-based fusion for realtime 3d
  reconstruction,'' \emph{Graphical Models}, vol.~75, no.~3, pp. 126--136,
  2013.

\bibitem{supereight}
E.~Vespa, N.~Nikolov, M.~Grimm, L.~Nardi, P.~H.~J. Kelly, and S.~Leutenegger,
  ``Efficient octree-based volumetric slam supporting signed-distance and
  occupancy mapping,'' \emph{IEEE Robotics and Automation Letters}, vol.~3,
  no.~2, pp. 1144--1151, 2018.

\bibitem{curless1996volumetric}
B.~Curless and M.~Levoy, ``A volumetric method for building complex models from
  range images,'' in \emph{Conference on Computer Graphics and Interactive
  Techniques (SIGGRAPH)}, 1996, pp. 303--312.

\bibitem{kazhdan2006poisson}
M.~Kazhdan, M.~Bolitho, and H.~Hoppe, ``Poisson surface reconstruction,'' in
  \emph{Eurographics Symposium on Geometry Processing}, 2006.

\bibitem{kinfu}
S.~Izadi, D.~Kim, O.~Hilliges, D.~Molyneaux, R.~Newcombe, P.~Kohli, J.~Shotton,
  S.~Hodges, D.~Freeman, A.~Davison, and A.~Fitzgibbon, ``{KinectFusion:
  real-time 3D reconstruction and interaction using a moving depth camera},''
  in \emph{ACM Sym. on User Interface Software and Technology (UIST)}, 2011,
  pp. 559--568.

\bibitem{voxblox}
H.~Oleynikova, Z.~Taylor, M.~Fehr, R.~Siegwart, and J.~Nieto, ``{Voxblox:
  incremental 3D Euclidean signed distance fields for on-board MAV planning},''
  in \emph{IEEE/RSJ Int. Conf. on Intelligent Robots and Systems (IROS)}, 2017.

\bibitem{fiesta}
L.~Han, F.~Gao, B.~Zhou, and S.~Shen, ``Fiesta: a fast incremental euclidean
  distance fields for online quadrotor motion planning,'' in \emph{IEEE/RSJ
  Int. Conf. on Intelligent Robots and Systems (IROS)}, 2019.

\bibitem{niessner2013real}
M.~Nie{\ss}ner, M.~Zollh{\"o}fer, S.~Izadi, and M.~Stamminger, ``Real-time 3d
  reconstruction at scale using voxel hashing,'' \emph{ACM Transactions on
  Graphics}, vol.~32, no.~6, 2013.

\bibitem{duong2020autonomous}
T.~Duong, N.~Das, M.~Yip, and N.~Atanasov, ``Autonomous navigation in unknown
  environments using sparse kernel-based occupancy mapping,'' in \emph{IEEE
  Int. Conf. on Robotics and Automation (ICRA)}, 2020.

\bibitem{thrun2005probabilistic}
S.~Thrun, W.~Burgard, and D.~Fox, \emph{Probabilistic robotics}.\hskip 1em plus
  0.5em minus 0.4em\relax MIT press, 2005.

\bibitem{gmapping}
G.~{Grisetti}, C.~{Stachniss}, and W.~{Burgard}, ``Improved techniques for grid
  mapping with rao-blackwellized particle filters,'' \emph{IEEE Transactions on
  Robotics}, vol.~23, no.~1, pp. 34--46, 2007.

\bibitem{OCallaghan2012gaussian}
S.~T. O'Callaghan and F.~T. Ramos, ``Gaussian process occupancy maps,''
  \emph{The International Journal of Robotics Research (IJRR)}, vol.~31, no.~1,
  pp. 42--62, 2012.

\bibitem{wang2016fast}
J.~Wang and B.~Englot, ``Fast, accurate gaussian process occupancy maps via
  test-data octrees and nested bayesian fusion,'' in \emph{IEEE Int. Conf. on
  Robotics and Automation (ICRA)}, 2016, pp. 1003--1010.

\bibitem{jadidi2017warped}
M.~G. Jadidi, J.~V. Miro, and G.~Dissanayake, ``Warped gaussian processes
  occupancy mapping with uncertain inputs,'' \emph{IEEE Robotics and Automation
  Letters}, vol.~2, no.~2, pp. 680--687, 2017.

\bibitem{ramos2016hilbertmap}
F.~Ramos and L.~Ott, ``Hilbert maps: scalable continuous occupancy mapping with
  stochastic gradient descent,'' \emph{The International Journal of Robotics
  Research (IJRR)}, vol.~35, no.~14, pp. 1717--1730, 2016.

\bibitem{senanayake2017bayesian}
R.~Senanayake and F.~Ramos, ``Bayesian hilbert maps for dynamic continuous
  occupancy mapping,'' in \emph{Conference on Robot Learning (CoRL)}, 2017, pp.
  458--471.

\bibitem{tipping2000relevance}
M.~E. Tipping, ``The relevance vector machine,'' in \emph{Advances in neural
  information processing systems}, 2000, pp. 652--658.

\bibitem{tipping2001sparse}
M.~E. Tipping, ``Sparse bayesian learning and the relevance vector machine,''
  \emph{Journal of Machine Learning Research}, vol.~1, no. Jun, pp. 211--244,
  2001.

\bibitem{tipping2003fast}
M.~E. Tipping, A.~C. Faul \emph{et~al.}, ``Fast marginal likelihood
  maximisation for sparse bayesian models.'' in \emph{Int. Conf. on Artificial
  Intelligence and Statistics (AISTATS)}, 2003.

\bibitem{lopez2017aggressive}
B.~T. Lopez and J.~P. How, ``Aggressive 3-d collision avoidance for high-speed
  navigation,'' in \emph{IEEE Int. Conf. on Robotics and Automation (ICRA)},
  2017, pp. 5759--5765.

\bibitem{chen2017improving}
J.~Chen and S.~Shen, ``Improving octree-based occupancy maps using environment
  sparsity with application to aerial robot navigation,'' in \emph{IEEE Int.
  Conf. on Robotics and Automation (ICRA)}, 2017, pp. 3656--3663.

\bibitem{fridovich2017atommap}
D.~Fridovich-Keil, E.~Nelson, and A.~Zakhor, ``Atommap: A probabilistic
  amorphous 3d map representation for robotics and surface reconstruction,'' in
  \emph{IEEE Int. Conf. on Robotics and Automation (ICRA)}, 2017, pp.
  3110--3117.

\bibitem{bialkowski2016efficient}
J.~Bialkowski, M.~Otte, S.~Karaman, and E.~Frazzoli, ``Efficient collision
  checking in sampling-based motion planning via safety certificates,''
  \emph{The International Journal of Robotics Research (IJRR)}, vol.~35, no.~7,
  pp. 767--796, 2016.

\bibitem{luo2014empirical}
J.~Luo and K.~Hauser, ``An empirical study of optimal motion planning,'' in
  \emph{IEEE/RSJ Int. Conf. on Intelligent Robots and Systems (IROS)}.\hskip
  1em plus 0.5em minus 0.4em\relax IEEE, 2014, pp. 1761--1768.

\bibitem{hauser2015lazy}
K.~Hauser, ``Lazy collision checking in asymptotically-optimal motion
  planning,'' in \emph{IEEE Int. Conf. on Robotics and Automation (ICRA)},
  2015.

\bibitem{Tsardoulias2016planninggridmap}
E.~G. Tsardoulias, A.~Iliakopoulou, A.~Kargakos, and L.~Petrou, ``A review of
  global path planning methods for occupancy grid maps regardless of obstacle
  density,'' \emph{Journal of Intelligent {\&} Robotic Systems}, vol.~84,
  no.~1, pp. 829--858, 2016.

\bibitem{das2017fastron}
N.~Das, N.~Gupta, and M.~Yip, ``{Fastron: an online learning-based model and
  active learning strategy for proxy collision detection},'' in
  \emph{Conference on Robot Learning (CoRL)}, 2017, pp. 496--504.

\bibitem{pan2015efficient}
J.~Pan and D.~Manocha, ``Efficient configuration space construction and
  optimization for motion planning,'' \emph{Engineering}, vol.~1, no.~1, pp.
  046--057, 2015.

\bibitem{huh2016learningGMM}
J.~Huh and D.~D. {Lee}, ``{Learning high-dimensional mixture models for fast
  collision detection in rapidly-exploring random trees},'' in \emph{IEEE Int.
  Conf. on Robotics and Automation (ICRA)}, 2016, pp. 63--69.

\bibitem{das2020learning}
N.~Das and M.~Yip, ``Learning-based proxy collision detection for robot motion
  planning applications,'' \emph{IEEE Transactions on Robotics}, 2020.

\bibitem{pan2012fcl}
J.~Pan, S.~Chitta, and D.~Manocha, ``Fcl: A general purpose library for
  collision and proximity queries,'' in \emph{IEEE Int. Conf. on Robotics and
  Automation (ICRA)}.\hskip 1em plus 0.5em minus 0.4em\relax IEEE, 2012, pp.
  3859--3866.

\bibitem{liu2017search}
S.~Liu, N.~Atanasov, K.~Mohta, and V.~Kumar, ``Search-based motion planning for
  quadrotors using linear quadratic minimum time control,'' in \emph{IEEE/RSJ
  Int. Conf. on Intelligent Robots and Systems (IROS)}, 2017, pp. 2872--2879.

\bibitem{de2000stabilization}
A.~De~Luca, G.~Oriolo, and M.~Vendittelli, ``Stabilization of the unicycle via
  dynamic feedback linearization,'' \emph{IFAC Proceedings Volumes}, vol.~33,
  no.~27, pp. 687--692, 2000.

\bibitem{franch2009control}
J.~Franch and J.~Rodriguez-Fortun, ``Control and trajectory generation of an
  ackerman vehicle by dynamic linearization,'' in \emph{European Control
  Conference (ECC)}, 2009, pp. 4937--4942.

\bibitem{mackay1992evidence}
D.~J. MacKay, ``The evidence framework applied to classification networks,''
  \emph{Neural Computation}, vol.~4, no.~5, pp. 720--736, 1992.

\bibitem{nabney2004efficient}
I.~T. Nabney, ``Efficient training of rbf networks for classification,''
  \emph{International Journal of Neural Systems}, vol.~14, no.~3, pp. 1--8,
  2004.

\bibitem{neal2012bayesian}
R.~M. Neal, \emph{Bayesian learning for neural networks}.\hskip 1em plus 0.5em
  minus 0.4em\relax Springer Science \& Business Media, 2012, vol. 118.

\bibitem{freespaceassumption}
S.~{Koenig} and Y.~{Smirnov}, ``Sensor-based planning with the freespace
  assumption,'' in \emph{IEEE Int. Conf. on Robotics and Automation (ICRA)},
  1997.

\bibitem{lehoucq1998arpack}
R.~B. Lehoucq, D.~C. Sorensen, and C.~Yang, \emph{ARPACK users' guide: solution
  of large-scale eigenvalue problems with implicitly restarted Arnoldi
  methods}.\hskip 1em plus 0.5em minus 0.4em\relax SIAM, 1998.

\bibitem{Russell_AI_Modern_Approach}
S.~Russell and P.~Norvig, \emph{Artificial intelligence: a modern
  approach}.\hskip 1em plus 0.5em minus 0.4em\relax Prentice Hall Press, 2009.

\bibitem{karaman2010incremental}
S.~Karaman and E.~Frazzoli, ``Incremental sampling-based algorithms for optimal
  motion planning,'' \emph{Robotics: Science and Systems (RSS)}, vol. 104,
  no.~2, 2010.

\bibitem{Radish}
\BIBentryALTinterwordspacing
A.~Howard and N.~Roy, ``The robotics data set repository (radish),'' 2003.
  [Online]. Available: \url{http://radish.sourceforge.net/}
\BIBentrySTDinterwordspacing

\bibitem{arslan2016exact}
O.~Arslan and D.~E. Koditschek, ``Exact robot navigation using power
  diagrams,'' in \emph{IEEE Int. Conf. on Robotics and Automation (ICRA)},
  2016.

\bibitem{senanayake2018automorphing}
R.~Senanayake, A.~Tompkins, and F.~Ramos, ``Automorphing kernels for
  nonstationarity in mapping unstructured environments.'' in \emph{CoRL}, 2018,
  pp. 443--455.

\end{thebibliography}

%


\begin{IEEEbiography}[{\includegraphics[width=1in,height=1.25in,clip,keepaspectratio]{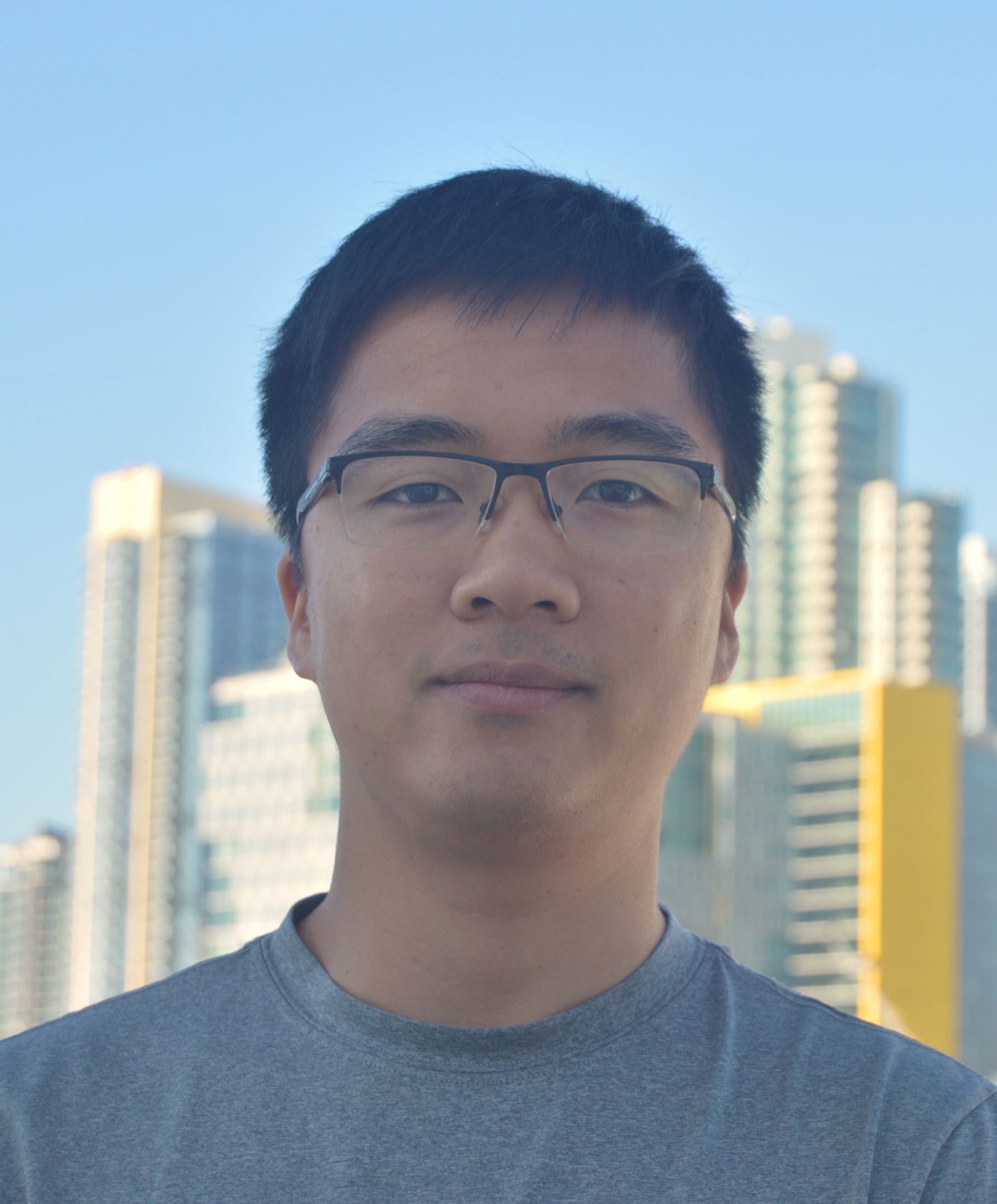}}]{Thai Duong}
is a PhD student in Electrical and Computer Engineering at the University of California, San Diego. He received a B.S. degree in Electronics and Telecommunications from Hanoi University of Science and Technology, Hanoi, Vietnam in 2011 and an M.S. degree in Electrical and Computer Engineering from Oregon State University, Corvallis, OR,  in 2013. His research interests include machine learning with applications to robotics, mapping and active exploration using mobile robots, robot dynamics learning, and decision making under uncertainty.
\end{IEEEbiography}
\begin{IEEEbiography}[{\includegraphics[width=1in,height=1.25in,clip,keepaspectratio]{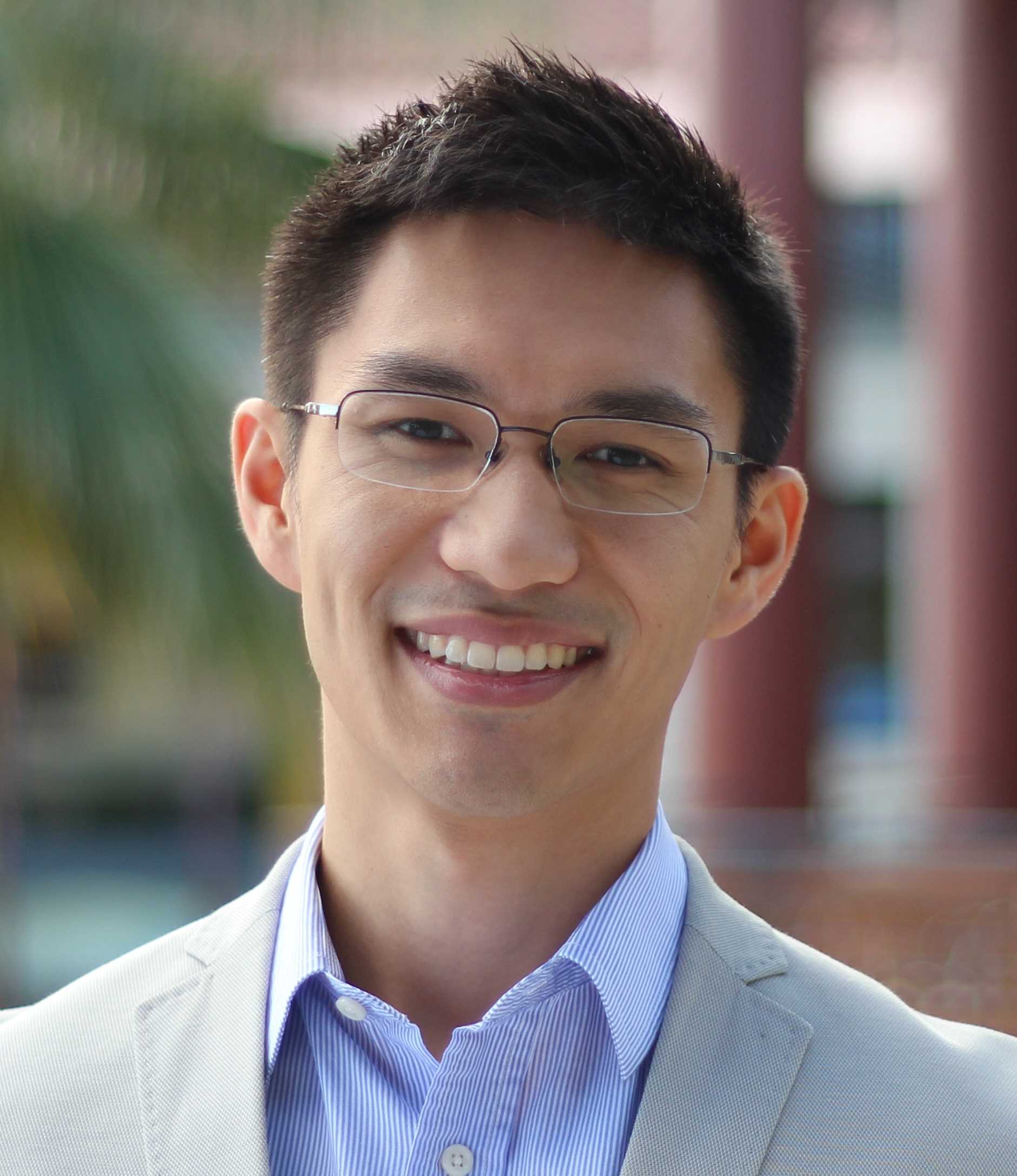}}]{Michael Yip}
is an Assistant Professor of Electrical and Computer Engineering at UC San Diego, IEEE RAS Distinguished Lecturer, Hellman Fellow, and Director of the Advanced Robotics and Controls Laboratory (ARCLab). His group currently focuses on solving problems in data-efficient and computationally efficient robot control and motion planning through the use of various forms of learning representations, including deep learning and reinforcement learning strategies. His lab applies these ideas to surgical robotics and the automation of surgical procedures. Previously, Dr. Yip's research has investigated different facets of model-free control, planning, haptics, soft robotics, and computer vision strategies, all towards achieving automated surgery. Dr. Yip's work has been recognized through several best paper awards at ICRA, including the inaugural best paper award for IEEE's Robotics and Automation Letters. Dr. Yip has previously been a Research Associate with Disney Research Los Angeles in 2014, a Visiting Professor at Stanford University in 2019, and a Visiting Professor with Amazon Robotics' Machine Learning and Computer Vision group in Seattle, WA in 2018. He received a B.Sc. in Mechatronics Engineering from the University of Waterloo, an M.S. in Electrical Engineering from the University of British Columbia, and a Ph.D. in Bioengineering from Stanford University.
\end{IEEEbiography}
\begin{IEEEbiography}[{\includegraphics[width=1in,height=1.25in,clip,keepaspectratio]{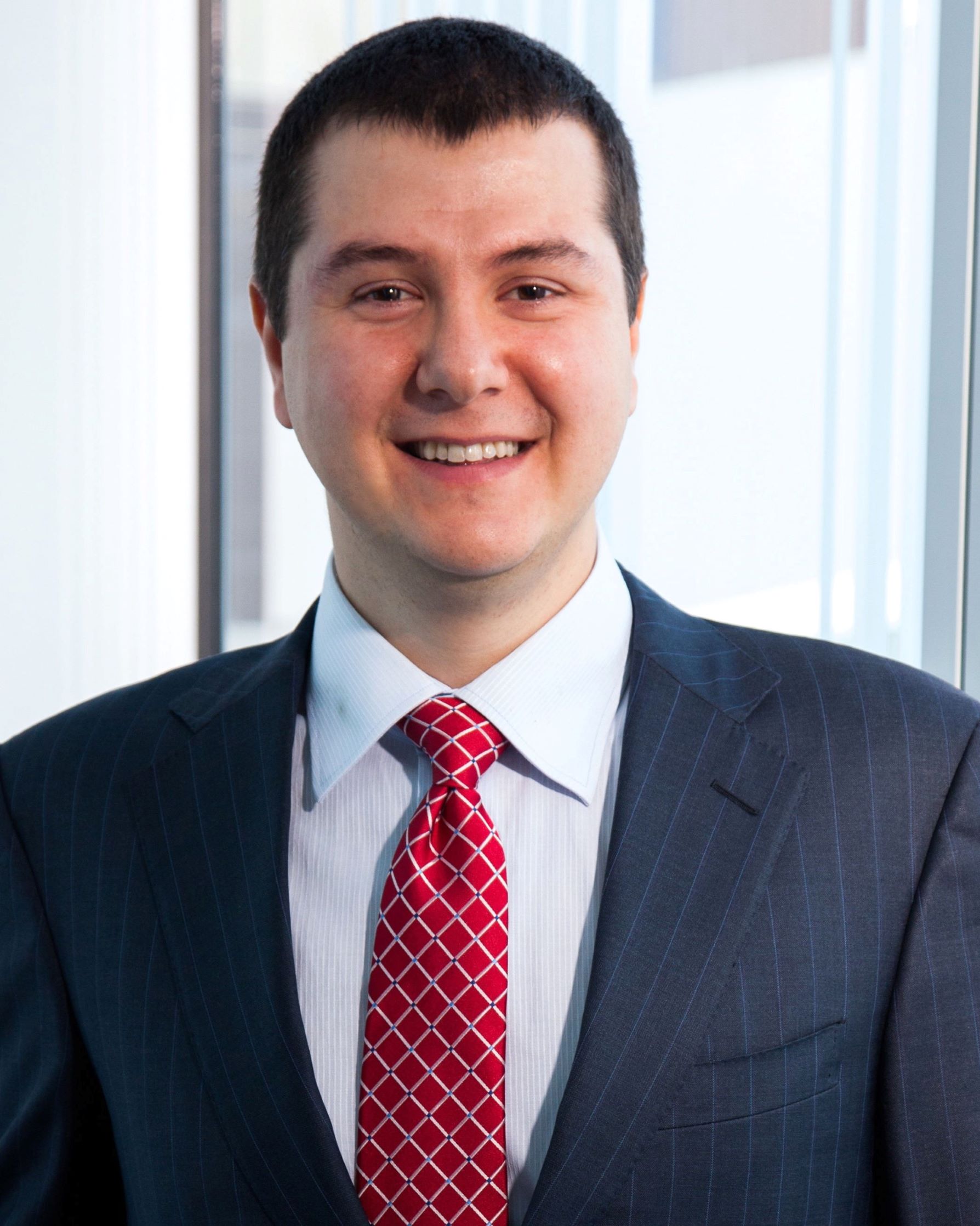}}]{Nikolay Atanasov}
(S'07-M'16) is an Assistant Professor of Electrical and Computer Engineering at the University of California San Diego. He obtained a B.S. degree in Electrical Engineering from Trinity College, Hartford, CT, in 2008 and M.S. and Ph.D. degrees in Electrical and Systems Engineering from the University of Pennsylvania, Philadelphia, PA, in 2012 and 2015, respectively. His research focuses on robotics, control theory, and machine learning, applied to active sensing using ground and aerial robots. He works on probabilistic environment models that unify geometry and semantics and on optimal control and reinforcement learning approaches for minimizing uncertainty in these models. Dr. Atanasov's work has been recognized by the Joseph and Rosaline Wolf award for the best Ph.D. dissertation in Electrical and Systems Engineering at the University of Pennsylvania in 2015 and the best conference paper award at the International Conference on Robotics and Automation in 2017.
\end{IEEEbiography}




\end{document}